%% file: T2VT.tex
\documentclass{article}

     \usepackage[nonatbib,preprint]{neurips_2020}

\usepackage[utf8]{inputenc} 
\usepackage[T1]{fontenc}    
\usepackage{hyperref}       
\usepackage{url}            
\usepackage{booktabs}       
\usepackage{amsfonts}       
\usepackage{nicefrac}       
\usepackage{microtype}      
\usepackage[colorinlistoftodos, textwidth=20mm]{todonotes}

\usepackage{nicefrac}       
\usepackage{amsmath,amssymb}
\usepackage{mathtools}
\usepackage{pifont}
\usepackage[makeroom]{cancel}
\usepackage{placeins}
\usepackage{algorithm}
\usepackage{algorithmic}
\usepackage{wrapfig}
\usepackage{subfigure}
\usepackage{xcolor}
\definecolor{blue1}{HTML}{00A6FB}
\definecolor{green1}{HTML}{97DB4F}
\definecolor{bronze}{HTML}{CD7F32}
\definecolor{neonFuchsia}{HTML}{FE59C2}

\definecolor{mikadoYellow}{HTML}{FFC40C}
\definecolor{lapisLazuli}{HTML}{26619C}
\definecolor{harvardCrimson}{HTML}{C90016}
\definecolor{harlequin}{HTML}{3FFF00}

\usepackage{tikz}
\usetikzlibrary{shapes.geometric}
\usetikzlibrary{shapes.arrows}
\usetikzlibrary{positioning}
\usetikzlibrary{patterns}
\tikzset{%
	dblock/.style = {rectangle, 
		text centered, dashed, draw,
		semithick, rounded corners,
		font=\sffamily\footnotesize,
	},
	block/.style    = { rectangle, draw, 
		font=\sffamily\footnotesize,
		text width=8em, text centered,
		rounded corners, minimum height=2em 
	},
	tblock/.style = {rectangle,
		font=\sffamily\footnotesize,
	},
	line/.style= {draw},
	dline/.style = {dashed},
	arrow/.style     = { draw, -> },
	darrow/.style = {dashed, ->},
}

\usepackage{pgfplots}
\usepgfplotslibrary{fillbetween}
\pgfplotsset{compat=newest, 
	tick label style={font=\scriptsize},
	label style={font=\scriptsize},
	legend style={font=\scriptsize}}

\newenvironment{customlegend}[1][]{%
	\begingroup
	\csname pgfplots@init@cleared@structures\endcsname
	\pgfplotsset{#1}%
}{%
	\csname pgfplots@createlegend\endcsname
	\endgroup
}%

\def\addlegendimage{\csname pgfplots@addlegendimage\endcsname}

\pgfplotscreateplotcyclelist{custom}{%
	green1!60!black,dashed,ultra thick\\
	orange,dotted,ultra thick\\
	blue1, ultra thick\\
	violet!90!white,dash dot,ultra thick\\
	mikadoYellow,densely dashed, ultra thick\\
	lapisLazuli,densely dotted, ultra thick\\
	harvardCrimson,densely dash dot, ultra thick\\
	harlequin, ultra thick\\
	neonFuchsia,densely dotted, ultra thick\\
	bronze,densely dash dot, ultra thick\\
	black, ultra thick\\
}

\usepackage{amsthm}
\theoremstyle{remark}

\theoremstyle{definition}
\newtheorem{definition}{Definition}[section]

\usepackage{thmtools, thm-restate}

\usepackage{xspace}
\DeclareRobustCommand{\eg}{e.g.,\@\xspace}
\DeclareRobustCommand{\ie}{i.e.,\@\xspace}
\DeclareRobustCommand{\wrt}{w.r.t.\@\xspace}
\DeclareRobustCommand{\iid}{i.i.d.\@\xspace}

\DeclareSymbolFont{matha}{OML}{txmi}{m}{it}
\DeclareMathSymbol{\varv}{\mathord}{matha}{118}

\title{Time-Variant Variational Transfer for Value Functions}

\author{%
  Giuseppe Canonaco\thanks{equal contribution} \\
  Politecnico di Milano, Milan, Italy \\
  \texttt{giuseppe.canonaco@polimi.it} \\
  \And
   Andrea Soprani\footnotemark[1] \\
   Politecnico di Milano, Milan, Italy \\
  \texttt{andrea.soprani96@gmail.com} \\
  \AND
   Manuel Roveri \\
   Politecnico di Milano, Milan, Italy \\
  \texttt{manuel.roveri@polimi.it} \\
  \And
   Marcello Restelli \\
   Politecnico di Milano, Milan, Italy \\
  \texttt{marcello.restelli@polimi.it} \\
}

\begin{document}

\maketitle

\begin{abstract}
In most of the transfer learning approaches to reinforcement learning (RL) the distribution over the tasks is assumed to be stationary. Therefore, the target and source tasks are \iid samples of the same distribution. In the context of this work, we consider the problem of transferring value functions through a variational method when the distribution that generates the tasks is time-variant, proposing a solution that leverages this temporal structure inherent in the task generating process. Furthermore, by means of a finite-sample analysis, the previously mentioned solution is theoretically compared to its time-invariant version. Finally, we will provide an experimental evaluation of the proposed technique with three distinct temporal dynamics in three different RL environments.
\end{abstract}

\section{Introduction}\label{sec:intro}
Reinforcement learning (RL) techniques~\cite{sutton2011reinforcement} are becoming increasingly effective in dealing with complex problems~\cite{vinyals2019grandmaster,silver2018general,openai2019dota} at the cost of requiring a huge amount of experience to achieve these impressive results. Therefore, a desirable feature for RL algorithms is sample efficiency, which could be achieved, among all other alternatives, through transfer learning (TL)~\cite{JMLR:v10:taylor09a,lazaric2012transfer}. TL allows an RL algorithm to reuse knowledge coming from a set of already solved tasks in order to speed up the learning phase of new ones. Depending on what kind of knowledge representation is being transferred, we have different TL algorithms in the related literature. Therefore, in order to perform the transfer, we may have algorithms leveraging policies or options~\cite{fernandez2006probabilistic,konidaris2007building}, samples~\cite{taylor2008transferring, lazaric2008transfer, tirinzoni2018importance, tirinzoni2019transfer}, features~\cite{barreto2017successor, lehnert2018transfer}, value-functions~\cite{taylor2007transfer, tirinzoni2018transfer} or parameters~\cite{killian2017robust, al2017continuous, nagabandi2018learning, du2019task}.

In the classical TL setting, the source and target tasks usually come from the same distribution, hence it would be sensible to use the Bayesian framework to iteratively refine the prior knowledge coming from the source tasks as more evidence from the target is collected. Following this rationale, in~\cite{wilson2007multi}, under the assumption that the tasks share similarities in their Markov Decision Process (MDP)~\cite{puterman2014markov} representation, a hierarchical Bayesian solution is proposed, whose main drawback lies in the need to solve an auxiliary MDP in order to perform actions on the task currently faced. Another methodology, along this line of research, has been developed in~\cite{lazaric2010bayesian}, which still leverages hierarchical Bayesian models, but this time assuming the tasks share commonalities through their value functions. Furthermore, in~\cite{doshi2016hidden}, a Bayesian framework able to adapt optimal policies to variations of the task dynamics is developed. They use a latent variable, which, together with the state-action couple, entirely describes the system dynamics. The uncertainty over the latent variable is modeled independently of the uncertainty over the state. This limitation is overcome in the extension to their framework proposed in~\cite{killian2017robust}. In~\cite{perez2020generalized} another extension to~\cite{doshi2016hidden} is proposed, which accounts for multiple variation factors that potentially also come from the reward function. A more general and efficient approach is instead developed in~\cite{tirinzoni2018transfer}, which iteratively refines the distribution over optimal value functions by means of a variational procedure as more experience from the target task is collected. 

In real-world applications, the system to be controlled by an agent is very likely to evolve with time. Therefore, in the task generating process of a family of similar tasks, there may be an underlying temporal dynamic to consider. Temporal dynamics are usually not considered in the related TL literature. For this reason, in this paper, we will extend the work developed in~\cite{tirinzoni2018transfer} in order to take into account a time-variant distribution inherent to the task generating process. We will, then, provide a theoretical comparison between our solution and the time-invariant approach of~\cite{tirinzoni2018transfer}. Finally, we will provide an experimental comparison of the two approaches in three different RL environments with three distinct temporal dynamics.

\section{Preliminaries}
\label{sec:prelim}
In this section, we describe the setting introduced in~\cite{tirinzoni2018transfer} by adding a time-variant distribution over the tasks. We will start with basic RL concepts and some notation in Section~\ref{subsec:background}, and we will conclude with the variational approach to transfer in Section~\ref{subsec:variational_transfer}.
\subsection{Reinforcement Learning Background}\label{subsec:background}
Let us consider a time-variant distribution $\mathcal{D}_t$ over tasks. We model each task $\mathcal{M}_{t}$ coming from $\mathcal{D}_t$ as a discounted Markov Decision Process (MDP)~\cite{puterman2014markov}, which is defined as a tuple $\mathcal{M}_{t}=\{\mathcal{S}, \mathcal{A}, \mathcal{P}_{t}, \mathcal{R}_{t}, p_0, \gamma\}$. $\mathcal{S}$ and $\mathcal{A}$ represent the state space and the action space, respectively. $\mathcal{P}_{t}$ is the Markovian transition function, where $\mathcal{P}_{t}(s'|s,a)$ is the transition density from state $s$ to state $s'$ given that the action $a$ is executed on the environment. $\mathcal{R}_{t}:\mathcal{S}\times \mathcal{A}\rightarrow \mathbb{R}$ is the reward function, assumed to be uniformly bounded by a constant $R_{max}>0$. $p_0$ and $\gamma \in [0,1)$ are the initial state distribution and the discount factor, respectively. Therefore, for each task $t$ our goal is to find a deterministic policy, $\pi_{t}:\mathcal{S} \rightarrow \mathcal{A}$, maximizing the long-term return over a possibly infinite horizon. In other words, this means being able to get $\pi_{t}^* \in \arg \max_{\pi}J_t(\pi)$, where $J_t(\pi)=\mathbb{E}_{\mathcal{M}_{t}, \pi}[\sum_{h=0}^{\infty}\gamma^h \mathcal{R}_{t}(s_h,a_h)].$ The optimal policy $\pi_{t}^*$ is a greedy policy \wrt the optimal value function, \ie $\pi_{t}^*(s)=\arg\max_a \mathcal{Q}_{t}^*(s,a)$ for all $s$, where $\mathcal{Q}_{t}^*(s,a)$ is defined as the expected return obtained by taking action $a$ in state $s$ and then following the optimal policy afterward. From now on, for the sake of readability, we will drop the $t$ subscript whenever this does not imply ambiguity. 

In this context, we focus on a set of parametrized value functions, $\mathcal{Q} = \{\mathcal{Q}_{\theta}: \mathcal{S}\times\mathcal{A}\rightarrow\mathbb{R} | \theta \in \mathbb{R}^p\}$, also called $\mathcal{Q}$-functions. We assume that each $\mathcal{Q}_{\theta} \in \mathcal{Q}$ is uniformly bounded by $\frac{R_{max}}{1-\gamma}$. An optimal $\mathcal{Q}$-function is also the fixed point of the optimal Bellman operator \cite{puterman2014markov}, which is defined as follows: $T\mathcal{Q}_{\theta}(s, a) = \mathcal{R}(s, a) + \gamma\mathbb{E}_{s'\sim\mathcal{P}}[\max_{a'}\mathcal{Q}_{\theta}(s', a')]$. Therefore, a measure of optimality for a value function during learning is its Bellman error, defined as $B_{\theta}=T\mathcal{Q}_{\theta}-\mathcal{Q}_{\theta}$. Of course, if $B_{\theta}(s,a)=0~ \forall (s,a)~ \in~ \mathcal{S}\times \mathcal{A}$, then $Q_{\theta}$ is optimal, which implies that minimizing the squared Bellman error, $||B_{\theta}||^2_{\nu}$, is a good objective for learning ($\nu$ is the distribution over $\mathcal{S}\times \mathcal{A}$, assumed to exist). In practice, the Bellman error is not used, since it requires two independent samples of the next state $s'$ for each couple $(s,a)$~\cite{maillard2010finite,sutton2011reinforcement}. For this reason, usually, the Bellman error is replaced by the Temporal Difference (TD) error $b(\theta)$, which corresponds to an approximation of the former using one sample $\left<s_h,a_h,r_h,s_{h+1}\right>$, so $b_h(\theta) = r_h +\gamma\max_{a'}\mathcal{Q}_{\theta}(s_{h+1}, a')-\mathcal{Q}_{\theta}(s_h, a_h)$. Therefore, given a set $D=\left<s_h,a_h,r_h,s_{h+1}\right>_{h=1}^N$ the squared TD error is $||B_{\theta}||^2_D = \frac{1}{N}\sum_{h=1}^{N}b_h(\theta)^2$.
\begin{wrapfigure}{L}{0.6\textwidth}
	\begin{minipage}{0.6\textwidth}
		\vspace{-23pt}
		\begin{algorithm}[H]
			\caption{Variational Transfer}
			\label{alg:variational_transfer}
			\begin{algorithmic}[1]
				\STATE {\bfseries Input:} Target task $\mathcal{M}_{t}$, source weights $\Theta_s$
				\STATE Estimate prior $p(\theta)$ from $\Theta_s$
				\STATE Initialize parameters: $\xi \leftarrow \arg\min_{\xi \in \Xi}D_{KL}(q_{\xi}||p)$
				\STATE Initialize dataset $D=\emptyset$
				\WHILE{$True$}
				\STATE Sample initial state $s_0 \sim p_0$
				\WHILE{$s_h$ is not terminal}
				\STATE Sample weights $\theta\sim q_{\xi}(\theta)$
				\STATE Take action $a_h=\arg\max_{a}\mathcal{Q}_{\theta}(s_h,a)$
				\STATE $s_{h+1}\sim \mathcal{P}_{t}(\cdot|s_h,a_h)$, $r_{h+1}=\mathcal{R}_{i}(s_h,a_h)$
				\STATE $D\leftarrow D \cup \left< s_h, a_h, r_{h+1}, s_{h+1}\right>$
				\STATE Estimate $\nabla_{\xi}\mathcal{L}(\xi)$ using $D'\subseteq D$
				\STATE Update $\xi$ with $\nabla_{\xi}\mathcal{L}(\xi)$ using any optimizer
				\ENDWHILE
				\ENDWHILE
			\end{algorithmic}
		\end{algorithm}
	\end{minipage}
\end{wrapfigure}

\subsection{Variational Transfer of Value Functions}\label{subsec:variational_transfer}
In the context described above, an optimal solution to an RL problem is a greedy policy \wrt an optimal value function that is parameterized by a vector of weights $\theta$. Therefore, we can safely consider a distribution over optimal weights $p(\theta)$ instead of the distribution $\mathcal{D}$ over tasks since the latter induces a distribution over optimal $\mathcal{Q}$-functions~\cite{tirinzoni2018transfer}.
Now, given a prior on the weights $p(\theta)$ and a dataset $D=\left<s_h,a_h,r_h,s_{h+1}\right>_{h=1}^N$, the optimal Gibbs posterior that minimizes an oracle upper bound on the expected loss is known to be~\cite{catoni2007pac}:
\begin{align}
	q(\theta) = \frac{e^{-\Psi ||B_{\theta}||_D^2}p(\theta)}{\int e^{-\Psi ||B_{\theta'}||_D^2}p(\theta')d\theta'}, \label{eq:gibbs_posterior}
\end{align}
where $\Psi>0$, which will be set to $\psi^{-1}N$, for some constant $\psi>0$ as in~\cite{tirinzoni2018transfer}. It is worth noting that $q$ becomes a Bayesian posterior every time $e^{-\Psi ||B_{\theta}||_D^2}$ can be interpreted as the likelihood of $D$. Since the integral at the denominator of Equation~\eqref{eq:gibbs_posterior} is intractable, a variational approximation through a parametrized family of posteriors $q_{\xi}$, such that $\xi \in \Xi$, is proposed. In this way, it is sufficient to find $\xi^*$ such that $q_{\xi^*}$ minimizes the Kullback-Leibler (KL) divergence \wrt the Gibbs posterior $q$, which is equivalent to minimizing the (negative) evidence lower bound (ELBO)~\cite{blei2017variational}:
\begin{align}\label{eq:ELBO}
	\min_{\xi \in \Xi} \mathcal{L}(\xi) = \min_{\xi \in \Xi}\left\{ \mathbb{E}_{\theta \sim q_{\xi}}\left[||B_\theta||_D^2\right] + \frac{\psi}{N}D_{KL}(q_{\xi}(\theta)||p(\theta))\right\}.
\end{align}

Therefore, the idea behind the variational transfer of value functions (Algorithm~\ref{alg:variational_transfer}) is to alternate a sampling from the posterior on the optimal value function with the optimization of the posterior via $\nabla_{\xi}\mathcal{L}(\xi)$, assuming to have already solved a finite number of source tasks $\mathcal{M}_{1} \ldots \mathcal{M}_{n}$, which, in turn, implies having the set of their approximate solutions $\Theta_s=\{\theta_1, \ldots, \theta_n\}$. The weight resampling can be interpreted as a guess on the task that we need to solve based on the current belief. After sampling, the algorithm acts on the RL problem as if such guess were correct and then will adjust the belief based on the new experience through the optimization of the variational parameters $\xi$. Notice that, as long as $\nabla_{\xi}\mathcal{L}(\xi)$ can be efficiently computed, any approximator for the $\mathcal{Q}$-functions and any prior/posterior distributions can be used. To this end, since the $\max$ operator in the temporal difference error of Equation~\eqref{eq:ELBO} is not differentiable, the \textit{mellowmax} is used instead, which is differentiable and was proven to converge to the same fixed point of the optimal Bellman operator in~\cite{tirinzoni2018transfer}. From now on, we will denote the mellow Bellman error with $\tilde{B}_\theta$.


\section{Time-Variant Kernel Density Estimation for Variational Transfer}
\label{sec:TVKDE-VT}
In the context of this work, we will model the evolution of time over a discrete grid of asymptotically dense time instants. Let $\{\theta_{ij}\}_{j=1}^{M_i}$ be a set of independent solutions to the $i^{th}$ task, observed at time $t_i=\frac{i}{n}$, $1\le i \le n$, with $\theta_{ij} \in \mathbb{R}^p$ and $\theta_{ij} \sim P(\cdot,t_i)$. Notice that, at time $t_i$, we allow to tackle $M_i$ times the task coming from the distribution $P(\cdot,t_i)$, for the sake of generality. Furthemore, let $M_i$ be a discrete random variable for each $i$. Finally, let us introduce a Time-Variant Kernel Density Estimator of this form:
\begin{equation}\label{eq:TVKDE}
	\hat{p}(\theta, t)=
	\frac{1}{a_0(-\rho) \bar{N} \lambda |H|^{\frac{1}{2}}}\sum_{i= 1}^{n}K_T\left(\frac{t-t_i}{\lambda}\right)
	\sum_{j=1}^{M_i}K_S(H^{-\frac{1}{2}}(\theta-\theta_{ij})),
\end{equation}
which is based on~\cite{hall2006real} and will be used as a prior in order to model a time-variant distribution on the solved tasks. The factor $a_0(-\rho)=\int_{-\rho}^{1}K_T(t)dt$ is used to perform the boundary correction, recovering consistency at the boundaries \cite{jones1993simple}, therefore also in $t=1$. $K_T$ is the temporal kernel, whereas $K_S$ is the multivariate non-negative spatial kernel. Furthermore, $H$ is the spatial kernel bandwidth matrix, $\lambda~\in [0,1]$ is the temporal kernel bandwidth, and $\bar{N}=\sum_{i=1}^{n}M_i$.

Now under the following assumptions (also stated in \cite{hall2006real}):
\begin{restatable}[Task independence]{assumption}{taskindependence}\label{ass:task_independence}
	For $1 \le i \ne i' \le n$,$1 \le j \le M_i$, and $1 \le j' \le M_{i'}$, $\theta_{ij}$ and $\theta_{i'j'}$ are independent;	
\end{restatable}

\begin{restatable}[Differentiable density function]{assumption}{differentiabledesnity}\label{ass:differentiable_density}
	$p(\theta, t): \mathbb{R}^p \times (0,1] \rightarrow \mathbb{R}$ is twice differentiable for every $t$, $\theta$;	
\end{restatable}

\begin{restatable}[Bounded derivatives]{assumption}{boundedderivatives}\label{ass:lipschitz_continuity}
	$p(\theta, t): \mathbb{R}^p \times (0,1] \rightarrow \mathbb{R}$ has two bounded derivatives;	
\end{restatable}

\begin{restatable}[On the spatial kernel]{assumption}{spatialkernel}\label{ass:spatial_kernel} Let $\alpha=(\alpha_1, \ldots, \alpha_p)$ be a multi-index, with $\alpha_i \geq 0$ for $i=1, \ldots, p$, $\theta^{\alpha}= \prod_{i=1}^{p}\theta_i^{\alpha_i}$ for each $\theta \in \mathbb{R}^p$, and $N_0$ is an index set where all $p$ components of each member are either $0$ or even integers. 
	\begin{align*}
	\int_{\mathbb{R}^p}K_S(\theta)d\theta=1, &\lim_{||\theta||\rightarrow \infty}||\theta||^pK_S(\theta)=0, \int_{\mathbb{R}^p}\theta^{\alpha}K_S(\theta)d\theta=\mu_{\alpha}\le \infty, \alpha \in N_0,	\\
	&\int_{\mathbb{R}^p}\theta^{\alpha}K_S(\theta)d\theta=0, \alpha \notin N_0;
	\end{align*}
\end{restatable}

\begin{restatable}[On the temporal kernel]{assumption}{temporalkernel}\label{ass:temporal_kernel}
	\[
	\int_{-c}^{c}K_T(t)dt=1, \int_{-c}^{c}tK_T(t)dt=0, \int_{-c}^{c}t^2K_T(t)dt=\sigma_T \le \infty;
	\]	
\end{restatable}

we can write

\begin{restatable}[Uniform consistency of the density estimator]{theorem}{convergence}\label{thm:convergence}
	Assume \ref{ass:task_independence}~-~\ref{ass:temporal_kernel}. Moreover, assume that $K_S$ is spherically symmetric, with a bounded, H$\ddot{o}$lder-continuous derivative, that $K_T$ is a compactly supported kernel on  a subset of $\mathbb{R}$, that all the $M_i$s are independent and identically distributed random variables with mean $m>0$ and all moments finite, independent of the $\theta_{ij}$s. Take $H$ and $\lambda$ such that $|H|^{\frac{1}{2}}(n) \rightarrow 0$, $\lambda(n) \rightarrow 0$ and $n^{1-\epsilon}|H|^{\frac{1}{2}}\lambda \rightarrow \infty$ for some $\epsilon>0$ as $n\rightarrow \infty$, then
	\begin{align*}
	\hat{p}(\theta,t)= p(\theta, t) + O\big[(\bar{N}|H|^{\frac{1}{2}}\lambda)^{-\frac{1}{2}}(\log n)^{\frac{1}{2}} + tr(H) +\lambda \big]
	\end{align*}
	uniformly in $(\theta,t) \in \mathcal{K}\times\mathcal{I}$, with probability 1, where $\mathcal{K}$ is a compact subset of $\mathbb{R}^p$ and $\mathcal{I}$ is a compact subset of $(0,1]$.
\end{restatable}
A proof of the above theorem is shown in Appendix~\ref{appx:thm3.6} and leverages the same approach as in~\cite{hall2006real} being a weaker version, in terms of convergence rate, of their Theorem 1. This weakening was necessary to obtain an upper bound in closed-form expression of the KL-divergence between the prior and the posterior in Equation~\eqref{eq:ELBO}. Indeed, if we choose $q_{\xi}(\theta)=\frac{1}{K}\sum_{k=1}^{K}\mathcal{N}(\theta|\mu_{k}, \Sigma_k)$, with variational parameter $\xi=(\mu_{1},\ldots,\mu_{K},\Sigma_{1},\ldots,\Sigma_K)$, and we choose $K_S$ as a Gaussian kernel, then for a fixed time instant $t$ our prior is a mixture of Gaussians with non-uniform weights. Therefore, through the upper bound on the KL-divergence shown in Appendix~\ref{appx:KL-UB}, we have that the ELBO upper bounds the KL-divergence between the approximate and the exact posterior. Since the covariance matrices of the posterior must be positive definite, we will learn the factor $L$ of their Cholesky decomposition as in~\cite{tirinzoni2018transfer}.

Before going on with the finite sample analysis of the following section, we would like to discuss the previous assumptions and their potential limiting effects on applications. For what concerns Assumptions \ref{ass:spatial_kernel} and \ref{ass:temporal_kernel} they do not pose any limit, since, as we know from kernel density estimation theory, the kernel type is not so relevant for a good estimate of the density. Assumptions \ref{ass:differentiable_density} and \ref{ass:lipschitz_continuity}, instead, are necessary in order to have some degree of regularity allowing learning of the time-variant distribution (without those assumption the kernel density estimator would not be consistent). The range of time-variant distributions where our approach will be theoretically effective is reduced due to Assumptions \ref{ass:differentiable_density} and \ref{ass:lipschitz_continuity}, but remains still relevant from an application point of view.

\section{Finite Sample Analysis}\label{sec:finitesampleanalysis}
In order to provide a finite sample analysis of Algorithm~\ref{alg:variational_transfer} based on the prior of Section~\ref{sec:TVKDE-VT}, we will extend Theorem 2 of~\cite{tirinzoni2018transfer} in our context, enabling also a theoretical comparison between the two respective versions of Algorithm~\ref{alg:variational_transfer}. Therefore, considering the family of linearly parametrized value functions, $Q_{\theta}(s,a)=\theta^T\phi(s,a)$, having bounded weights $||\theta||_2\le \theta_{max}$ and uniformly bounded features $||\phi(s,a)||_2\le \phi_{max}$, and assuming that only finite data are available, we can bound the expected mellow Bellman error under the variational distribution minimizing Equation~\eqref{eq:ELBO} for any fixed target task $\mathcal{M}_t$ through the following theorem.
\begin{restatable}[Bound on the expected mellow Bellman error]{theorem}{convergence}\label{thm:mellowBellmanError}
	Let $\hat{\xi}$ be the variational parameter minimizing Equation~\eqref{eq:ELBO} on a dataset $D$ of N i.i.d. samples distributed according to $\mathcal{M}_t$ and $\nu$. Moreover, let $\theta^* = \arg\inf_{\theta}||\tilde{B}_\theta||_{\nu}^2$ and define $v(\theta^*)=\mathbb{E}_{\mathcal{N}(\theta^*,\frac{1}{N}I)}[\varv(\theta)]$, with $\varv(\theta)=\mathbb{E}_{\nu}[\mathbb{V}ar_{\mathcal{P}_t}[\tilde{b}(\theta)]]$, where $\tilde{b}(\theta) = r + \gamma mellow\textit{-}max_{a'}Q_{\theta}(s',a')-Q_{\theta}(s,a)$. Then, there exist constants $c_1, c_2, c_3$ such that with probability at least $1-\delta$ over the choice of $D$:
	\begin{align*}
		\mathbb{E}_{q_{\hat{\xi}}}\left[\left|\left|\tilde{B}_\theta\right|\right|^2_{\nu}\right] \le 2 \left|\left|\tilde{B}_{\theta^*}\right|\right|_\nu^2 + v(\theta^*) + c_1\sqrt{\frac{\log\frac{2}{\delta}}{N}} + \frac{c_2 + \psi p \log N +\psi \varphi(\Theta_s)}{N} + \frac{c_3}{N^2},
	\end{align*}
	where
	\begin{align}
		\varphi(\Theta_s) = \frac{1}{\sigma^2}\sum_{j:\theta_j \in \Theta_s}\frac{c_{j}^{\hat{p}} e^{-\beta||\theta^*-\theta_j||}}{\sum_{j':\theta_{j'} \in \Theta_s}c_{j'}^{\hat{p}} e^{-\beta||\theta^*-\theta_{j'}||}}||\theta^*-\theta_j||,
	\end{align}
	assuming the matrix $H$ of Equation~\eqref{eq:TVKDE} to be an isotropic covariance matrix with variance $\sigma^2$, $\beta=\frac{1}{2 \sigma^2}$ and $c_j^{\hat{p}}$ the weight assigned to the $j^{th}$ prior component. Furthermore, we are assuming $M_i=1$ for each $i$ in our estimator.
\end{restatable}
The above theorem shows the difference between the plain mixture version of Algorithm~\ref{alg:variational_transfer}~\cite{tirinzoni2018transfer} and our solution. Indeed, looking at $\varphi(\Theta_s)$, we are able to shed some light on the different theoretical properties of the two versions. More specifically, in the plain mixture version, factor $c_j^{\hat{p}}$ does not appear, which implies uniform importance of the source solutions $\Theta_s$ \wrt the target task. On the other hand, in our version of the algorithm, we are able to give different importance to each source solution through $c_j^{\hat{p}}$. In our time-variant scenario, this importance will be greater on more recent solutions than older ones, potentially allowing a reduction of the term $\varphi(\Theta_s)$ in contrast to the time-invariant version. A proof for the above theorem is provided in Appendix~\ref{appx:FS-analysis}.

\section{Related Works}\label{sec:relatedworks}
Our work is built upon~\cite{tirinzoni2018transfer}, but differs from it because we leverage a time-variant structure underlying the task generating process, which is not taken into account in~\cite{tirinzoni2018transfer}. A theoretical comparison between the two solutions is available in Section~\ref{sec:finitesampleanalysis} through Theorem~\ref{thm:mellowBellmanError}, whereas the experimental comparison is in Section~\ref{sec:experiments}. Furthermore, our work relates both to~\cite{wilson2007multi}, which deals with finite MDPs, and to~\cite{lazaric2010bayesian}, which leverages the commonalities in the value function structure, but, in contrast to our work, they do not account for a time-variant distribution. The work done in~\cite{doshi2016hidden,killian2017robust,perez2020generalized} leverage latent embeddings in order to model variations between tasks, which eventually are solved through a model-based RL algorithm, while we propose a model-free approach.

Another related work is~\cite{hall2015online}, in which the authors develop a theoretical low-regret algorithm accounting for potential underlying dynamics. However, they use the online learning framework, whereas we are working in a transfer learning setting. Furthermore, in~\cite{du2019task}, videos are used to learn a prior (mainly to model the physical dynamics) which is incorporated into a model-based RL algorithm. In~\cite{Yang2020Single}, a single-episode policy-transfer methodology was developed leveraging variational inference, but for contexts in which the differences in dynamics can be identified in the early steps of an episode. In the context of supervised learning, our work relates also to~\cite{minku2014make}, which proposes a transfer learning mechanism in the context of a possibly non-stationary environment through a weighting approach, and~\cite{minku2011ddd, du2019multi}, which, instead, do transfer in non-stationary environments through ensembles. Finally, in~\cite{khodak2019adaptive} the authors are able to consider optimal initializations varying through time, but they develop this approach for a meta-learning framework, while our work considers a transfer learning setting.

\section{Experiments}\label{sec:experiments}
In this section, we compare our time-variant solution for transfer learning with the associated non-time-variant solution of~\cite{tirinzoni2018transfer} in three different domains with three different temporal dynamics. A detailed description of the used parameters together with the analytical expression of the employed dynamics are provided in Appendix~\ref{appx:experiments}.

\subsection{Temporal Dynamics}\label{subsec:time-dynamics}
The distribution over the tasks is usually a given distribution over one or more parameters defining the task itself. Therefore, in order to obtain time-variance in such distribution, we will change its mean over time according to a certain dynamic. These dynamics are linear, polynomial, and sinusoidal. In the context of these experiments, we will use a time-variant Gaussian distribution, clipping its realizations within the domain of the parameters defining the task (for further details see Appendix~\ref{appx:experiments}).

\subsection{Two-Rooms Environment}\label{subsec:2-rooms}
In this setting, we have an agent navigating two rooms separated by a wall. The agent starts from the bottom-left corner and must reach the opposite one. The only way to reach this goal is to pass through the door whose position is unknown to the agent. The actions available to the agent are \textit{up}, \textit{down}, \textit{left}, and \textit{right}, which allow the agent to move in the respective directions by one position, unless he/she hits a wall (in this last case the position remains unchanged). Furthermore, the final position of the agent after a movement action is altered by a Gaussian noise $\mathcal{N}(0,0.2)$. The state space is modeled through a $10\times10$ continuous grid. Finally, the reward function is $0$ everywhere except in the goal state, where it is $1$. The discount factor $\gamma=0.99$. For this setting, we used linearly parametrized $\mathcal{Q}$-functions with $121$ evenly-spaced radial basis features.

We considered source tasks taken at ten different time instants to learn the target, corresponding to the eleventh instant of time. We sampled five tasks from the time-variant distribution for each $i=1,\ldots,11$. The parameter that defines the task is the door location, hence the time-variant distribution is over that parameter, as we mentioned above. We solve all the source tasks by directly minimizing the TD error, then we exploit the learned solutions to perform the transfer over the target. We compare our time-variant variational transfer algorithm leveraging a $c$-components posterior ($c$-T2VT) with the mixture of Gaussian variational transfer using still $c$-components ($c$-MGVT)~\cite{tirinzoni2018transfer}. More specifically, our time-variant prior will consider the source task solutions as equally spaced samples in the time interval $[0,1]$, moreover, in order to perform transfer to the eleventh task, we will use the distribution provided by our estimator for $t=1$. Finally,  the temporal kernel will be Epanechnikov in the context of all the experiments.

The average return over the last $50$ learning episodes as a function of the number of training iterations is shown in Figure~\ref{fig:2-rooms}, for the time dynamics mentioned in Section~\ref{subsec:time-dynamics}. Each learning curve is computed using $50$ independent runs, each of which resamples both the source and target tasks, with $95\%$ confidence intervals. For polynomial and linear dynamics, we can see an advantage of our technique in the early learning iterations. The sinusoidal dynamic is designed to disadvantage our technique \wrt $c$-MGVT, indeed, it makes the target task appear twice in the sources. This fact inevitably favors $c$-MGVT, which will give a higher weight to those source tasks being sampled from the same distribution of the target. Observe that $c$-MGVT gives uniform weights to all the source tasks, hence increasing the replicas importance within the sources, whereas $c$-T2VT gives increasing weights the more recent the source solution.

\input{figure_2-rooms.tex}

\subsection{Three-Rooms Environment}\label{subsec:3-rooms}
This scenario is an extension of the previous one, hence the environmental settings remain the same, the agent has just an additional wall to traverse in order to reach his/her goal. Of course, the position of the door for this additional wall is still unknown to the agent. To increase the complexity of the dynamics, we let the two doors move in opposite directions starting at the two far ends of the room, each door with the same dynamic. In Figure~\ref{fig:3-rooms}, we compare $c$-T2VT with $c$-MGVT using still $95\%$ confidence intervals. As for the polynomial dynamics, we observe a better performance of $c$-T2VT \wrt $c$-MGVT, whereas, for the sinusoidal dynamics, we have essentially the same behavior as in the two rooms environment. Finally, in the linear dynamics, we observe that the difference in performance between the two algorithms is not statistically significant.

\input{figure_3-rooms.tex}

\subsection{Mountain Car}\label{subsec:mountain-car}
In this section, we consider a classic control environment known as Mountain Car~\cite{sutton2011reinforcement}. In Mountain Car the agent is an underpowered car whose goal is to escape from a valley. Due to the limitation to its engine, the car has to alternately drive up along the two slopes of the valley in order to gain sufficient momentum to overcome gravity. In Figure~\ref{fig:mountain-car}, we have a comparison between $c$-T2VT and $c$-MGVT on the three proposed dynamics. We observe a statistically significant improvement in the polynomial dynamics across the whole learning process for $c$-T2VT, which also extends to the sinusoidal dynamic case. We would like to highlight the differences between the sinusoidal dynamic in Mountain Car \wrt the previous two environments. Here our algorithm is able to perform better due to a bias-variance trade-off in its favor. More specifically, the value functions vary more rapidly in Mountain Car than in the room environments \wrt a change in the task-defining parameters. Therefore, our prior estimator has less variance, since it considers only the latest sources, at the cost of a bias increase, because it discards the first task, which has the same parametrization as the target (due to the periodicity of the $\sin$ function). $c$-MGVT considers all the source tasks with the same uniform weight, hence it is able to consider the tasks that have an equivalent parametrization to the target, but are farther behind in the history of the sources. This fact decreases the bias at the cost of accepting a greater variance in the prior estimation. In Mountain Car the trade-off proposed by our algorithm is more advantageous than the one proposed by $c$-MGVT due to the more rapidly changing behavior of the value functions. As for the linear dynamics, we do not observe a statistically significant difference in performance between the two algorithms. 

\input{figure_mountain-car.tex} 

\subsection{Choosing $\lambda$ through Maximum-Likelihood}\label{subsec:sensitivity}
Up to now, we have kept $\lambda$ and $H$ at given constant values in order to allow a more faithful comparison between $c$-T2VT and $c$-MGVT (the matrix $H$ was the same in the two algorithms whereas $\lambda$ was set to $0.3333$ leveraging the intuition that the more recent tasks were more important than the older ones). Of course, from the theory of Kernel Density Estimation, we know that appropriately setting these parameters is crucial to get a good estimate of the density. Therefore, an automatic data-driven approach would be desirable. In the context of this work, we propose a maximum likelihood scheme (assuming $M_i=1 ~\forall ~i$):
\begin{align}
	&\arg \max_{\lambda} L_{\lambda}=\prod_{h=1}^n\frac{\hat{p}_{-h}(\theta_{h},t_h)}{\hat{p}_{-h}(t_h)}, \text{where} \label{eq:likelihood-optimization}\\
	&\hat{p}_{-h}(\theta_{h},t_h)=\frac{1}{a_0(-\rho) (\bar{N}-1) \lambda |H|^{\frac{1}{2}}}\sum_{i \ne h}K_T\left(\frac{t_h-t_i}{\lambda}\right)
	K_S(H^{-\frac{1}{2}}(\theta_{h}-\theta_{i}))\nonumber\\
	&\hat{p}_{-h}(t_h)=\int \hat{p}_{-h}(\theta_{h},t_h) d\theta_{h}=\frac{1}{a_0(-\rho) (\bar{N}-1) \lambda}\sum_{i \ne h} K_T\left(\frac{t_h-t_i}{\lambda}\right). \nonumber	
\end{align}
In Appendix~\ref{subsec:lambda-sensitivity-results}, we report the performance achievable with this approach together with a sensitivity analysis \wrt the parameter $\lambda$. Furthermore, still in Appendix~\ref{subsec:lambda-sensitivity-results}, we include some implementation details related to the optimization of the likelihood function in Equation~\eqref{eq:likelihood-optimization}. Note that, in accordance with what has been done in~\cite{tirinzoni2018transfer}, the spatial bandwidth is set to $10^{-5}I$ which would prevent us from successfully optimizing Equation~\eqref{eq:likelihood-optimization} due to numerical issues, hence we set it to $I$ in order to select the best lambda.

\section{Discussion and Conclusions}
In this paper, we presented a time-variant approach for transferring value functions through a variational scheme. In order to deal with a time-variant distribution of the tasks, we have devised a suitable estimator for the prior to be used in the variational scheme providing its uniform consistency over a compact subset of $\mathbb{R}^p\times(0,1]$. We have, then, provided a finite sample analysis on the performance of the variational transfer algorithm based on our estimator, enabling a theoretical comparison with the time-invariant version of~\cite{tirinzoni2018transfer}. Finally, we have experimentally proved our algorithm abilities to deal with time-variant distributions.

Notice that discriminating the source tasks according to time is an additional step that brings transfer learning approaches and learning in non-stationary environments a bit closer together~\cite{minku2019transfer}. It is also important to highlight the fact that, instead of considering time, we could switch to any other variable (\eg the parameter defining the task itself, which, in the context of Mountain Car, is the base speed of the agent) as long as it is available together with each source solution and we can properly remap it into $(0,1]$. This could allow us to leverage completely different structures in order to perform transfer to the target task. Moreover, in order to further improve the capabilities of the algorithm to deal with time-variant distributions, it would be relevant to leverage Gaussian Processes with a non-stationary covariance function as future work~\cite{remes2017non}. Finally, we would also like to highlight the possibility of using this time-variant transfer paradigm also in lifelong learning scenarios~\cite{chen2018lifelong} as a potential future direction.

\section*{Broader Impact}
Leveraging the time-variant structure in transfer learning tasks strengthens the bond of transfer learning itself with different other fields such as learning in non-stationary environments and life-long learning. This could potentially allow the exploitation of transfer learning into the previously mentioned fields, giving life to new advancements. Furthermore, being able to take into account the previously mentioned structure gives access to one more tool to deal with real-world applications, which are often time-variant, with a potential reduction for the time-to-market of RL based systems. This tool could be leveraged in order to make an RL system more able to adapt through time. Therefore, it could give a greater societal trust toward this system or, on the other hand, could potentially increase mistrust due to a higher decision-making power given to the system itself. In both cases, more awareness on the part of the general public towards artificial intelligence is needed in order to both avoid overconfidence and blind mistrust of these systems.

\begin{ack}
Use unnumbered first level headings for the acknowledgments. All acknowledgments
go at the end of the paper before the list of references. Moreover, you are required to declare 
funding (financial activities supporting the submitted work) and competing interests (related financial activities outside the submitted work). 
More information about this disclosure can be found at: \url{https://neurips.cc/Conferences/2020/PaperInformation/FundingDisclosure}.

Do {\bf not} include this section in the anonymized submission, only in the final paper. You can use the \texttt{ack} environment provided in the style file to autmoatically hide this section in the anonymized submission.
\end{ack}

\bibliography{T2VT}
\bibliographystyle{plain}
\clearpage
\appendix
\section{Proof of Theorem 3.6}\label{appx:thm3.6}
\begin{definition}\label{def:1}For a spatial kernel $K_S$: $\mu_l(K_S) = \int \theta^lK_S(\theta)d\theta$
\end{definition}
\begin{definition}For a temporal kernel $K_T$: $a_l(-\rho) = \int_{-\rho}^{1}t^l K_T(t)dt$
\end{definition}
\begin{restatable}[Estimator consistency on the right boundary]{lemma}{estimatorconsistency}\label{lma:estimator_consistency}
	Let $t \in B_r = \left \{\tau: 1-\lambda \le \tau \le 1\right \}$ then under assumptions of Theorem \ref{thm:convergence}:
	\[
	\mathbb{E}[\hat{p}(\theta, t)|\mathcal{M}]=p(\theta,t) + O(\lambda) + O(tr(H)), 
	\]
	where $\mathcal{M}$ represents all the discrete random variables $M_i$ for $i=1 \ldots n$.	
\end{restatable}
\begin{proof}
	\begin{align}
		&\mathbb{E}[\hat{p}(\theta, t)|\mathcal{M}] = \nonumber\\
		&= \frac{1}{\bar{N} \lambda |H|^{\frac{1}{2}} a_0(-\rho)} \sum_{i=1}^{n}\int K_T\left(\frac{t-\tau}{\lambda}\right) \sum_{j=1}^{M_i} \int_{-\infty}^{+\infty} K_S\left(H^{-\frac{1}{2}}(\theta-x)\right)p(x,\tau)dxd\tau \\
		& = \frac{1}{\bar{N} \lambda \cancel{|H|^{\frac{1}{2}}} a_0(-\rho)} \sum_{i=1}^{n}\int K_T\left(\frac{t-\tau}{\lambda}\right) \sum_{j=1}^{M_i} \int_{+\infty}^{-\infty} - K_S(y)p(\theta - H^{\frac{1}{2}}y,\tau)\cancel{|H|^{\frac{1}{2}}}dyd\tau \label{eq:3}\\
		& = \frac{1}{\bar{N} \lambda a_0(-\rho)}\sum_{i=1}^{n}\int K_T\left(\frac{t-\tau}{\lambda}\right) \sum_{j=1}^{M_i} \int_{-\infty}^{+\infty}K_S(y) \bigg(p(\theta, \tau) -(H^{\frac{1}{2}}y)^T\nabla^S p(\theta,\tau) + \nonumber\\ 
		&\qquad\qquad \frac{1}{2}(H^{\frac{1}{2}}y)^T \mathcal{H}^Sp(\theta, \tau) (H^{\frac{1}{2}}y) + o(tr(H))\bigg) dyd\tau \label{eq:4}\\
		& = \frac{1}{\bar{N} \lambda a_0(-\rho)}\sum_{i=1}^{n}\int K_T\left(\frac{t-\tau}{\lambda}\right) M_i \bigg( \int_{-\infty}^{+\infty}K_S(y)p(\theta, \tau)dy \nonumber \\
		&\qquad\qquad \cancel{- \int_{-\infty}^{+\infty} K_S(y)(H^{\frac{1}{2}}y)^T\nabla^S p(\theta,\tau)dy} + \nonumber \\
		&\qquad\qquad \int_{-\infty}^{+\infty} \frac{1}{2}K_S(y)(H^{\frac{1}{2}}y)^T \mathcal{H}^Sp(\theta, \tau) (H^{\frac{1}{2}}y)dy + o(tr(H)) \bigg)d\tau \label{eq:5}\\
		& = \frac{1}{\bar{N} \lambda a_0(-\rho)}\sum_{i=1}^{n}\int K_T\left(\frac{t-\tau}{\lambda}\right) M_i \bigg(p(\theta, \tau) + \nonumber \\
		&\qquad\qquad \frac{1}{2}\mu_2(K_S)tr(H\mathcal{H}^sp(\theta, \tau)) +o(tr(H)) \bigg)d\tau \label{eq:6}\\
		& = \frac{1}{\lambda a_0(-\rho)} \int_{\frac{t-1}{\lambda}}^{\frac{t}{\lambda}}K_T\left(\frac{t-\tau}{\lambda}\right) \bigg( p(\theta, \tau) + O(tr(H))\bigg)d\tau \label{eq:7}\\
		& = \frac{1}{\lambda a_0(-\rho)}\left(\int_{-\rho}^{1} K_T\left(\frac{t-\tau}{\lambda}\right)p(\theta,\tau)d\tau + O(tr(H))\int_{-\rho}^{1}K_T\left(\frac{t-\tau}{\lambda}\right)d\tau \right) \label{eq:8}\\
		& =  \frac{\cancel{\lambda}}{\cancel{\lambda} a_0(-\rho)}\left(-\int_{1}^{-\rho}K_T(v)p(\theta,t-\lambda v)dv - O(tr(H))\int_{1}^{-\rho}K_T(v)dv \right) \label{eq:9} \\
		& = \frac{1}{a_0(-\rho)} \bigg(\int_{-\rho}^{1}K_T(v)\bigg(p(\theta,t) -\lambda vp'(\theta,t) + \nonumber \\
		&\qquad\qquad \frac{1}{2}\lambda^2 v^2 p''(\theta,t) + o(\lambda^2)\bigg)dv + O(tr(H))\bigg) \label{eq:10}\\
		& = p(\theta,t) - \lambda p'(\theta,t)\frac{a_1(-\rho)}{a_0(-\rho)} + O(\lambda^2) + O(tr(H)) \\
		& = p(\theta,t) + O(\lambda) + O(tr(H)),
	\end{align}
	where in (\ref{eq:3}) we performed a change of variable, $y=H^{-\frac{1}{2}}(\theta-x)$, in (\ref{eq:4}) we used the following Taylor expansion: 
	\[
	p(\theta - H^{\frac{1}{2}}y, \tau) = p(\theta,\tau) - (H^{\frac{1}{2}}y)^T \nabla^S p(\theta,\tau) + \frac{1}{2}(H^{\frac{1}{2}}y)^T \mathcal{H}^S p(\theta,\tau)(H^{\frac{1}{2}}y) + o(tr(H)),
	\]
	in (\ref{eq:5}) we used Assumption \ref{ass:spatial_kernel},in (\ref{eq:6}) we used Definition \ref{def:1}, in (\ref{eq:7}) we used $\frac{t-\tau}{\lambda} \in [\frac{t-1}{\lambda},\frac{t}{\lambda})$, in \ref{eq:8} we set $t=1-\rho\lambda$, which implies $\frac{t-\tau}{\lambda} \in [-\rho,\frac{1}{\lambda} -\rho)$, then we used the support of $K_T$ (assumed to be $[-1,1]$ without loss of generality) since $\lambda \rightarrow 0$. Finally, in \ref{eq:9} we used a change of variable, $\frac{t-\tau}{\lambda}=v$, and in \ref{eq:10} we used the following Taylor expansion:
	\[
	p(\theta,t-\lambda v)= p(\theta,t) -\lambda vp'(\theta,t) + \frac{1}{2}\lambda^2v^2 p''(\theta,v)+ o(\lambda^2).
	\] 
\end{proof}
Notice that we reported the consistency proof only on the right boundary because is the one we use in the context of our algorithm. The above procedure can be easily adjusted to prove consistency of the estimator on the left boundary getting the same convergence rate. Moreover, analogously, we can obtain consistency away from the two boundaries with a convergence rate squared \wrt $\lambda$.

\begin{definition}\label{def:3}For a spatial kernel $K_S$: $R(K_S) = \int K_S^2(\theta)d\theta$
\end{definition}
\begin{definition}\label{def:4}For a temporal kernel $K_T$: $b_{K_T}(-\rho) = \int_{-\rho}^{1} K_T^2(t)dt$
\end{definition}
\begin{restatable}[Variance of the estimator on the right boundary]{lemma}{estimatorvariance}\label{lma:estimator_variance}
	Let $t \in B_r = \left \{\tau: 1-\lambda \le \tau \le 1\right \}$ then under assumptions of Theorem \ref{thm:convergence}:
	\[
	\mathbb{V}ar[\hat{p}(\theta, t)|\mathcal{M}] \le \frac{C_1}{\bar{N}|H|^{\frac{1}{2}}\lambda}, 
	\]
	where $\mathcal{M}$ represents all the discrete random variables $M_i$ for $i=1 \ldots n$.	
\end{restatable}

\begin{proof}
	\begin{align}
		& \mathbb{V}ar[\hat{p}(\theta, t)|\mathcal{M}] = \frac{1}{\bar{N}a_0^2(-\rho)}\mathbb{V}ar\left[ \frac{1}{|H|^{\frac{1}{2}}\lambda}K_T\left(\frac{t-t_i}{\lambda}\right)K_S\left(H^{-\frac{1}{2}}(\theta-x_{ij})\right) \right] \\
		& = \frac{1}{\bar{N}a_0^2(-\rho)}\bigg( \mathbb{E}\left[ \frac{1}{|H|\lambda^2}K_T^2\left(\frac{t-t_i}{\lambda}\right)K_S^2\left(H^{-\frac{1}{2}}(\theta-x_{ij})\right) \right] - \nonumber\\
		&\qquad\qquad \mathbb{E}^2\left[ \frac{1}{|H|^{\frac{1}{2}}\lambda}K_T\left(\frac{t-t_i}{\lambda}\right)K_S\left(H^{-\frac{1}{2}}(\theta-x_{ij})\right) \right] \bigg) \\
		& = \frac{1}{\bar{N}a_0^2(-\rho)}\bigg( \int \frac{1}{|H|\lambda^2} K_T^2\left( \frac{t-\tau}{\lambda} \right) \int_{+\infty}^{-\infty}-|H|^{\frac{1}{2}}K_S^2(y)p(\theta-H^{\frac{1}{2}}y,\tau)dyd\tau - \nonumber \\
		&\qquad\qquad \left( \int \frac{1}{|H|^{\frac{1}{2}}\lambda} K_T\left( \frac{t-\tau}{\lambda} \right) \int_{+\infty}^{-\infty}-|H|^{\frac{1}{2}}K_S(y)p(\theta-H^{\frac{1}{2}}y,\tau)dyd\tau \right)^2 \bigg) \label{eq:15}\\
		& = \frac{1}{\bar{N}a_0^2(-\rho)}\bigg( \int \frac{1}{|H|^{\frac{1}{2}}\lambda^2} K_T^2\left( \frac{t-\tau}{\lambda} \right) \int_{-\infty}^{+\infty}K_S^2(y)\left( p(\theta, \tau) + o(1) \right)dyd\tau - \nonumber \\
		&\qquad\qquad \left( \int \frac{1}{\lambda} K_T\left( \frac{t-\tau}{\lambda} \right) \int_{-\infty}^{+\infty}K_S(y)\left( p(\theta, \tau) + o(1) \right)dyd\tau \right)^2 \bigg) \label{eq:16}\\
		& = \frac{1}{\bar{N}a_0^2(-\rho)}\bigg( \int \frac{1}{|H|^{\frac{1}{2}}\lambda^2} K_T^2\left( \frac{t-\tau}{\lambda} \right) \left( p(\theta, \tau) + o(1) \right)R(K_S)d\tau - \nonumber \\
		&\qquad\qquad \left( \int \frac{1}{\lambda} K_T\left( \frac{t-\tau}{\lambda} \right)\left( p(\theta, \tau) + o(1) \right)d\tau \right)^2 \bigg) \label{eq:17}\\
		& = \frac{1}{\bar{N}a_0^2(-\rho)}\bigg( \int_{-\rho}^{1} \frac{1}{|H|^{\frac{1}{2}}\lambda^2} K_T^2\left( \frac{t-\tau}{\lambda} \right) \left( p(\theta, \tau) + o(1) \right)R(K_S)d\tau - \nonumber \\
		&\qquad\qquad \left( \int_{-\rho}^{1} \frac{1}{\lambda} K_T\left( \frac{t-\tau}{\lambda} \right)\left( p(\theta, \tau) + o(1) \right)d\tau \right)^2 \bigg) \label{eq:18}\\
		& = \frac{1}{\bar{N}a_0^2(-\rho)}\bigg( \int_{1}^{-\rho} -\frac{1}{|H|^{\frac{1}{2}}\lambda} K_T^2(v) \left( p(\theta, t-\lambda v) + o(1) \right)R(K_S)dv - \nonumber \\
		&\qquad\qquad \left( \int_{1}^{-\rho} - K_T(v)\left( p(\theta, t-\lambda v) + o(1) \right)dv \right)^2 \bigg) \label{eq:19}\\
		& = \frac{1}{\bar{N}a_0^2(-\rho)}\bigg( \int_{-\rho}^{1} \frac{1}{|H|^{\frac{1}{2}}\lambda} K_T^2(v) \left( p(\theta, t) + o(1) \right)R(K_S)dv - \nonumber \\
		&\qquad\qquad \left( \int_{-\rho}^{1} K_T(v)\left( p(\theta, t) + o(1) \right)dv \right)^2 \bigg) \label{eq:20}\\
		& = \frac{1}{\bar{N}a_0^2(-\rho)} \left( \frac{p(\theta,t)+o(1)}{|H|^{\frac{1}{2}}\lambda}R(K_S)b_{K_T}(-\rho)-\big(a_0(-\rho)(p(\theta,t)+o(1))\big)^2 \right) \label{eq:21}\\
		& = \frac{p(\theta, t)R(K_S)b_{K_T}(-\rho)}{\bar{N}|H|^{\frac{1}{2}}\lambda a_0^2(-\rho)} + O\left(\frac{1}{\bar{N}|H|^{\frac{1}{2}}\lambda}\right) \\
		& = O\left(\frac{1}{\bar{N}|H|^{\frac{1}{2}}\lambda}\right) \quad\rightarrow\quad \exists~ C_1:\mathbb{V}ar[\hat{p}(\theta, t)|\mathcal{M}] \le \frac{C_1}{\bar{N}|H|^{\frac{1}{2}}\lambda}, \label{eq:23}
	\end{align}
	where in (\ref{eq:15}) we performed a change of variable, $y=H^{-\frac{1}{2}}(\theta-x)$, in (\ref{eq:16}) we used the following Taylor expansion: 
		\[
		p(\theta - H^{\frac{1}{2}}y, \tau) = p(\theta,\tau) + o(1),
		\]
	in \ref{eq:17} we used Definition \ref{def:3}, in \ref{eq:18} we considered the fact that $t \in B_r$ as we have done in \ref{eq:7} and \ref{eq:8} of the proof of \ref{lma:estimator_consistency}, in \ref{eq:19} we performed a change of variable, $\frac{t-\tau}{\lambda}=v$, in \ref{eq:20} we used the following Taylor expansion:
		\[
		p(\theta,t-\lambda v)= p(\theta,t) + o(1),
		\]
	whereas in \ref{eq:21} we have used Definition \ref{def:4}. Finally, in \ref{eq:23} we have used the fact that $p(\theta,t)$ has bounded derivatives and is a pdf, therefore it has finite supremum.
\end{proof}

\begin{restatable}[Bound on the absolute values]{lemma}{absolutevalue}\label{lma:absolute_value}
	Let $t \in B_r = \left \{\tau: 1-\lambda \le \tau \le 1\right \}$ then under assumptions of Theorem \ref{thm:convergence}:
	$\hat{p}(\theta, t)-\mathbb{E}[\hat{p}(\theta, t)|\mathcal{M}]$ is the sum of $\bar{N}$ independent random variables, denoted as $v_i$, with zero mean and absolute values bounded by $\frac{C_2}{\bar{N}|H|^{\frac{1}{2}}\lambda}$. $\mathcal{M}$ represents all the discrete random variables $M_i$ for $i=1 \ldots n$.
\end{restatable}

\begin{proof}
	\begin{align}
		&|v_i|=\bigg| \frac{1}{\bar{N} \lambda|H|^{\frac{1}{2}} a_0(-\rho)} K_T\left(\frac{t-t_i}{\lambda}\right) K_S(H^{-\frac{1}{2}}(\theta-x_{ij}))- \nonumber \\ 
		&\qquad\qquad \frac{p(\theta,t) + O(\lambda) + O(tr(H))}{\bar{N}} \bigg| \label{eq:24}\\
		& \le \left| \frac{M_TM_S}{\bar{N} \lambda|H|^{\frac{1}{2}} a_0(-\rho)} - \frac{p(\theta,t) + O(\lambda) + O(tr(H))}{\bar{N}} \right| \label{eq:25}\\
		& = O\left(\frac{1}{\bar{N} \lambda|H|^{\frac{1}{2}}}\right) ~ \rightarrow \exists C_2: ~ 	\left| v_i \right| \le \frac{C_2}{\bar{N}|H|^{\frac{1}{2}}\lambda},
	\end{align}
	where in \ref{eq:24} we used \textit{lemma} \ref{lma:estimator_consistency} and in \ref{eq:25} we used the fact that $K_T$ has a compact support on $\mathbb{R}$ and $K_S$ has a supremum.
\end{proof}
Now the proof of \textit{Theorem} \ref{thm:convergence} can follow. 
\begin{proof}
	Let $\xi = C \left(\frac{\log n}{\bar{N}|H|^{\frac{1}{2}}\lambda}\right)^\frac{1}{2}$ and $C_3 = \frac{1}{\max \left(C_1, \frac{C_2}{3}\right)}$, using Bernstein's inequality we can write:
	\begin{align}
		&\mathbb{P}\left( \left| \hat{p}(\theta, t)-\mathbb{E}[\hat{p}(\theta, t) |\mathcal{M}]\right| > \xi |\mathcal{M} \right) \le 2 \exp \left( - \frac{\frac{1}{2}\xi^2}{\frac{C_1}{\bar{N}|H|^{\frac{1}{2}}\lambda} + \frac{1}{3}\frac{C_2\xi}{\bar{N}|H|^{\frac{1}{2}}\lambda}} \right) \\
		& = 2\exp \left( - \frac{\frac{1}{2}C^2 \log n}{C_1 + \frac{1}{3}C_2\xi} \right) \le 2\exp \left( - \frac{C_3C^2 \log n}{1 + \xi} \right), \forall (\theta, t).
	\end{align}
	Therefore, if $C_4>0$ is given, and we choose $C^2>\frac{3C_4}{C_3}$, then we can write:
	\begin{align}
		&\sup_{(\theta, t) ~\in~ \mathbb{R}^p\times \mathcal{I}} \mathbb{P}\left( \left| \hat{p}(\theta, t)-\mathbb{E}[\hat{p}(\theta, t)|\mathcal{M}] \right| > C \left(\frac{\log n}{\bar{N}|H|^{\frac{1}{2}}\lambda}\right)^\frac{1}{2} \right) \le 2 \exp\left( -\frac{3C_4 \log n}{1+\xi} \right) \\
		& = 2 n^{ -\frac{3C_4}{1+\xi} }.
	\end{align}
	Now restricting to finite subsets $\mathcal{K}_n \subset \mathcal{K} \subset \mathbb{R}^p$ and $\mathcal{I}_n\subset \mathcal{I}$ where $\mathcal{K}_n \times \mathcal{I}_n$ has at most $\lfloor n^{\frac{2C_4}{1+\xi}} \rfloor$ elements, we have:
	
	\begin{align}
		\mathbb{P}\left( \sup_{(\theta, t) ~\in~ \mathcal{K}_n\times \mathcal{I}_n} \left| \hat{p}(\theta, t)-\mathbb{E}[\hat{p}(\theta, t)|\mathcal{M}] \right| > C \left(\frac{\log n}{\bar{N}|H|^{\frac{1}{2}}\lambda}\right)^\frac{1}{2} \right) \le 2n^{\frac{-C_4}{1+\xi}}.\label{eq:34}
	\end{align}
	From the H$\ddot{o}lder$-continuity of the estimator (since the two kernels have bounded first derivative):
	\begin{align}
		&\sup_{(\theta, t) ~\in~ \mathcal{K}\times \mathcal{I}}\left\{ \left|\hat{p}(\theta, t)-\mathbb{E}[\hat{p}(\theta, t)|\mathcal{M}] \right|\right\} - \sup_{(\theta, t) ~\in~ \mathcal{K}_n\times \mathcal{I}_n} \left\{\left|\hat{p}(\theta, t)-\mathbb{E}[\hat{p}(\theta, t)|\mathcal{M}] \right|\right\} = \nonumber\\
		&\left|\sup_{(\theta, t) ~\in~ \mathcal{K}\times \mathcal{I}}\left\{ \left|\hat{p}(\theta, t)-\mathbb{E}[\hat{p}(\theta, t)|\mathcal{M}] \right|\right\} - \sup_{(\theta, t) ~\in~ \mathcal{K}_n\times \mathcal{I}_n} \left\{\left|\hat{p}(\theta, t)-\mathbb{E}[\hat{p}(\theta, t)|\mathcal{M}] \right|\right\}\right| \le \nonumber\\
		&\quad \left|\sup_{(\theta, t) ~\in~ \mathcal{K}\times \mathcal{I}}\left\{ \hat{p}(\theta, t)-\mathbb{E}[\hat{p}(\theta, t)|\mathcal{M}] \right\} - \sup_{(\theta, t) ~\in~ \mathcal{K}_n\times \mathcal{I}_n}\left\{ \hat{p}(\theta, t)-\mathbb{E}[\hat{p}(\theta, t)|\mathcal{M}]\right\} \right| \le \nonumber\\
		&\quad D||v^*-v_n^*||^{\alpha}
		  ,
	\end{align}
	where 
	\begin{align*}
		&v^*=\arg \sup_{(\theta, t) ~\in~ \mathcal{K}\times \mathcal{I}} \left\{\hat{p}(\theta, t)-\mathbb{E}[\hat{p}(\theta, t)|\mathcal{M}] \right\} \\
		&v_n^*=\arg \sup_{(\theta, t) ~\in~ \mathcal{K}_n\times \mathcal{I}_n} \left\{\hat{p}(\theta, t)-\mathbb{E}[\hat{p}(\theta, t)|\mathcal{M}] \right\},
	\end{align*}
	therefore:
	
	\begin{align}
		&\mathbb{P}\left( \sup_{(\theta, t) ~\in~ \mathcal{K}\times \mathcal{I}} \left|\hat{p}(\theta, t)-\mathbb{E}[\hat{p}(\theta, t)|\mathcal{M}] \right| > C \left(\frac{\log n}{\bar{N}|H|^{\frac{1}{2}}\lambda}\right)^\frac{1}{2} + D||v^*-v_n^*||^{\alpha} \right) \le \nonumber \\
		&\qquad\qquad \mathbb{P}\left( \sup_{(\theta, t) ~\in~ \mathcal{K}_n\times \mathcal{I}_n} \left|\hat{p}(\theta, t)-\mathbb{E}[\hat{p}(\theta, t)|\mathcal{M}] \right| > C \left(\frac{\log n}{\bar{N}|H|^{\frac{1}{2}}\lambda}\right)^\frac{1}{2} \right)	
	\end{align}
	
	now, for sufficiently large $C_4$, $||v^*-v_n^*||\le \frac{\sqrt{p+1}}{2} \sqrt[p+1]{\frac{(K_{max}-K_{min})^p(I_{max}-I_{min})}{\left\lfloor n^{\frac{2C_4}{1+\xi}}\right\rfloor}}$ and $D\left(\frac{\sqrt{p+1}}{2} \sqrt[p+1]{\frac{(K_{max}-K_{min})^p(I_{max}-I_{min})}{\left\lfloor n^{\frac{2C_4}{1+\xi}}\right\rfloor}}\right)^{\alpha}$ is negligible \wrt $\xi$ as $n$ tends to infinity, where $K_{max}$ e $K_{min}$ are the endpoints for each dimension of $\mathcal{K}$ (we assume them to be the same in each dimension for the sake of simplicity). Analogously for $I_{max}$ and $I_{min}$ (notice that $\mathcal{I}$ is monodimensional).
	
	Therefore:
	\begin{align}
	&\mathbb{P}\bigg( \sup_{(\theta, t) ~\in~ \mathcal{K}\times \mathcal{I}} \left|\hat{p}(\theta, t)-\mathbb{E}[\hat{p}(\theta, t)|\mathcal{M}] \right| > C \left(\frac{\log n}{\bar{N}|H|^{\frac{1}{2}}\lambda}\right)^\frac{1}{2} + \nonumber \\ 
	&\qquad\qquad D\left(\frac{p+1}{4}\right)^{\frac{\alpha}{2}}\left(\frac{(K_{max}-K_{min})^p(I_{max}-I_{min})}{\left\lfloor n^{\frac{2C_4}{1+\xi}}\right\rfloor}\right)^{\frac{\alpha}{p+1}} \bigg) \le \nonumber \\
	&\qquad\qquad \mathbb{P}\left( \sup_{(\theta, t) ~\in~ \mathcal{K}_n\times \mathcal{I}_n} \left|\hat{p}(\theta, t)-\mathbb{E}[\hat{p}(\theta, t)|\mathcal{M}] \right| > C \left(\frac{\log n}{\bar{N}|H|^{\frac{1}{2}}\lambda}\right)^\frac{1}{2} \right) \label{eq:37}.
	\end{align}
	From (\ref{eq:34}) and (\ref{eq:37}), we can write:
	\begin{align}
		&\mathbb{P}\bigg( \sup_{(\theta, t) ~\in~ \mathcal{K}\times \mathcal{I}} \left|\hat{p}(\theta, t)-\mathbb{E}[\hat{p}(\theta, t)|\mathcal{M}] \right| < C \left(\frac{\log n}{\bar{N}|H|^{\frac{1}{2}}\lambda}\right)^\frac{1}{2} + \nonumber \\
		&\qquad\qquad D\left(\frac{p+1}{4}\right)^{\frac{\alpha}{2}}\left(\frac{(K_{max}-K_{min})^p(I_{max}-I_{min})}{\left\lfloor n^{\frac{2C_4}{1+\xi}}\right\rfloor}\right)^{\frac{\alpha}{p+1}} \bigg) \ge 1-2n^{\frac{-C_4}{1+\xi}}
	\end{align}
	
	Therefore, as $n \rightarrow \infty$ with probability $1$:
	\begin{align}
		&\left| \hat{p}(\theta, t)-p(\theta, t) - O(\lambda) - O(tr(|H|))  \right|= O \bigg[C\left(\frac{\log n}{\bar{N}|H|^{\frac{1}{2}}\lambda}\right)^\frac{1}{2}  + \nonumber \\ &\qquad\qquad D\left(\frac{p+1}{4}\right)^{\frac{\alpha}{2}}\left(\frac{(K_{max}-K_{min})^p(I_{max}-I_{min})}{\left\lfloor n^{\frac{2C_4}{1+\xi}}\right\rfloor}\right)^{\frac{\alpha}{p+1}} \bigg], \forall~ (\theta, t)~ \in~ \mathcal{K} \times \mathcal{I} 
	\end{align}
	Finally, we get:
	\begin{align}
		\hat{p}(\theta, t) = p(\theta, t) + O\left[\left(\frac{\log n}{\bar{N}|H|^{\frac{1}{2}}\lambda}\right)^\frac{1}{2} + \lambda + tr(H) \right], \forall~ (\theta, t)~ \in~ \mathcal{K} \times \mathcal{I} 
	\end{align}
	
\end{proof}
\section{Upper bound on the KL-Divergence between the prior and the posterior}\label{appx:KL-UB}
In this section, we report the steps needed to get an upper bound on the KL-Divergence between the posterior $q$ our prior $\hat{p}$. Let us define $S=\frac{1}{a_0(-\rho) \bar{N} \lambda }\sum_{i=1}^{n}\sum_{j=1}^{M_i}K_T(\frac{t-t_i}{\lambda})$, hence:
\begin{align}
	&D_{KL}(q||\hat{p}(\cdot, t)) = \int q(\theta) log\frac{q(\theta)}{\hat{p}(\theta,t)}d\theta= \int q(\theta) log\frac{q(\theta)}{\frac{S}{S}\hat{p}(\theta,t)}d\theta \\
	& = \int q(\theta) log\frac{q(\theta)}{\frac{1}{S a_0(-\rho) \bar{N} |H|^{\frac{1}{2}} \lambda }\sum_{i= 1}^{n}K_T(\frac{t-t_i}{\lambda})
		\sum_{j=1}^{M_i}K_S(H^{-\frac{1}{2}}(\theta-\theta_{ij}))}d\theta + \nonumber\\
	&\qquad\qquad \int q(\theta) \log \frac{1}{S} d\theta \label{eq:42}
\end{align}
Now the first term in Equation (\ref{eq:42}) is the KL-Divergence between two Mixture of Gaussians, which can be upper bounded using the same procedure as in \cite{hershey2007approximating}, and the second term is a constant in the ELBO optimization. Therefore:
\begin{align}
	D_{KL}(q||\hat{p}(\cdot, t)) \le D_{KL}(\chi^{(2)}||\chi^{(1)}) + \log\frac{1}{S} + \sum_{i,j}\chi^{(2)}_{j,i}D_{KL}(f_i^{q}||f_j^{\hat{p}}),\label{eq:KL-Bound}
\end{align}
where we are rewriting $q=\sum_{i}c_i^qf_i^{q}$ and $\hat{p}=\sum_{j}c_j^{\hat{p}}f_j^{\hat{p}}$ with $c_x^y$ being a generic weight and $f_x^y=\mathcal{N}(\mu_x^y, \Sigma_x^y)$ being a generic component, $(x,y) \in \{(i,q), (j,\hat{p})\}$. Furthermore, we have:
\begin{align}
	\chi_{i,j}^{(1)}=\frac{c_j^{\hat{p}}\chi_{j,i}^{(2)}}{\sum_{i'}\chi_{j,i'}^{(2)}},\qquad \chi_{j,i}^{(2)} = \frac{c_i^{(q)}\chi_{i,j}^{(1)}e^{-D_{KL}(f_i^{q}||f_j^{\hat{p}})}}{\sum_{j'}\chi_{i,j'}^{(1)}e^{-D_{KL}(f_i^{q}||f_{j'}^{\hat{p}})}}.
\end{align}
Finally, notice that $c_i^q=\frac{1}{C}$ for each $i$, where $C$ is the number of components for the posterior, whereas $c_j^{\hat{p}}=\frac{1}{S a_0(-\rho) \bar{N} \lambda }K_T(\frac{t-t_i}{\lambda})$, with a little abuse of notation over the index $i$ and $j$.

\section{Proof of Theorem 4.1}\label{appx:FS-analysis}
The prof of Theorem \ref{thm:mellowBellmanError} is straightforward, we just need to follow the same procedure of \cite{tirinzoni2018transfer} plugging in the bound on the KL-Divergence of Equation \ref{eq:KL-Bound}. In the following we report the proof for completeness.

\begin{proof}
	We start from Lemma 2 of \cite{tirinzoni2018transfer} with variational parameter $\hat{\xi}=(\hat{\mu}_1,\ldots,\hat{\mu}_C,\hat{\Sigma}_1,\ldots,\hat{\Sigma}_C)$, whereas, for the right-hand side, we set $\mu_i=\theta^*$ and $\Sigma_i=cI$ for each $i=1,\ldots,C$, for some $c>0$:
	\begin{align}
		&\mathbb{E}_{q_{\hat{\xi}}}\left[\left|\left|\tilde{B}_\theta\right|\right|_\nu^2\right] \le \inf_{\xi \in \Xi} \left\{ \mathbb{E}_{q_{\xi}}\left[\left|\left|\tilde{B}_\theta\right|\right|_\nu^2\right] + \mathbb{E}_{q_{\hat{\xi}}}\left[\varv(\theta)\right] + 2\frac{\psi}{N}D_{KL}(q_{\xi}||\hat{p}) \right\} + 8\frac{R_{max}^2}{(1-\gamma)^2}\sqrt{\frac{\log\frac{2}{\delta}}{2N}} \nonumber \\
		&\qquad\qquad \le  \mathbb{E}_{\mathcal{N}(\theta^*,cI)}\left[\left|\left|\tilde{B}_\theta\right|\right|_\nu^2\right] + \mathbb{E}_{\mathcal{N}(\theta^*,cI)}\left[\varv(\theta)\right] + 2\frac{\psi}{N}D_{KL}(\mathcal{N}(\theta^*,cI)||\hat{p}) + \nonumber\\ &\qquad\qquad 8\frac{R_{max}^2}{(1-\gamma)^2}\sqrt{\frac{\log\frac{2}{\delta}}{2N}}. \label{eq:48}
	\end{align}
	From Appendix \ref{appx:KL-UB} we have:
	\begin{align}
		&D_{KL}(\mathcal{N}(\theta^*,cI)||\hat{p}) \le \nonumber\\ 
		&\qquad\qquad D_{KL}(\chi^{(2)}||\chi^{(1)}) + \log\frac{1}{S} + \sum_{j}\chi_j^{(2)}D_{KL}(\mathcal{N}(\theta^*,cI) || \mathcal{N}(\theta_j,\sigma^2I)),\label{eq:49}
	\end{align}
	where
	\begin{align}
		\chi_j^{(1)} = c_j^{\hat{p}}, \qquad \chi_j^{(2)} = \frac{c_j^{\hat{p}}e^{-D_{KL}(\mathcal{N}(\theta^*,cI)||\mathcal{N}(\theta_j,\sigma^2I))}}{\sum_{j'}c_{j'}^{\hat{p}}e^{-D_{KL}(\mathcal{N}(\theta^*,cI)||\mathcal{N}(\theta_{j'},\sigma^2I))}}
	\end{align}
	obtained noticing that we can remove the index $i$ because we have reduced the posterior to one component. $\chi_j^{(2)}$ can be rewritten:
	\begin{align}
		 \chi_j^{(2)} = \frac{c_j^{\hat{p}}e^{-\frac{1}{2\sigma^2}||\theta^*-\theta_j||}}{\sum_{j'}c_{j'}^{\hat{p}}e^{-\frac{1}{2\sigma^2}||\theta^*-\theta_{j'}||}}
	\end{align}
	if we plug in the closed form expression of the KL-Divergence (\ref{eq:52}) into its definition.
	\begin{align}
		D_{KL}(\mathcal{N}(\theta^*,cI)||\mathcal{N}(\theta_j,\sigma^2I)) = \frac{1}{2}\left(p\log\frac{\sigma^2}{c} + p\frac{c}{\sigma^2} + \frac{||\theta^*-\theta_j||}{\sigma^2}-p\right)\label{eq:52}.
	\end{align}
	Now we proceed upper bounding the first and then the third term of \ref{eq:49}:
	\begin{align}
		&D_{KL}(\chi^{(2)}||\chi^{(1)}) = \sum_{j} \chi^{(2)}_j \log \frac{\chi^{(2)}_j}{\chi^{(1)}_j} \\
		&= \sum_{j} \chi^{(2)}_j \log \chi^{(2)}_j - \sum_{j} \chi^{(2)}_j \log \chi^{(1)}_j \label{eq:54}\\
		& \le \sum_{j} \chi^{(2)}_j \log \frac{1}{c_j^{\hat{p}}} \label{eq:55}
	\end{align}
	where we got \ref{eq:55} just noticing in \ref{eq:54} that the first term is negative. Considering the third term, we have:
	\begin{align}
		\sum_{j}\chi_j^{(2)}D_{KL}(\mathcal{N}(\theta^*,cI) || \mathcal{N}(\theta_j,\sigma^2I)) &= \frac{1}{2}\sum_{j}\chi_j^{(2)} \left(p\log\frac{\sigma^2}{c} + p\frac{c}{\sigma^2} + \frac{||\theta^*-\theta_j||}{\sigma^2}-p\right) \nonumber\\
		&\le \frac{1}{2}p\log\frac{\sigma^2}{c} + \frac{1}{2}p\frac{c}{\sigma^2} + \sum_{j}\chi_j^{(2)}\frac{||\theta^*-\theta_j||}{2\sigma^2}.
	\end{align}
	Therefore:
	\begin{align}
		&D_{KL}(\mathcal{N}(\theta^*,cI)||\hat{p}) \le\nonumber \\
		&\qquad\qquad \sum_{j} \chi^{(2)}_j \log \frac{1}{c_j^{\hat{p}}} + \log\frac{1}{S} + \frac{1}{2}p\log\frac{\sigma^2}{c} + \frac{1}{2}p\frac{c}{\sigma^2} + \sum_{j}\chi_j^{(2)}\frac{||\theta^*-\theta_j||}{2\sigma^2}.
	\end{align}
	Now leveraging the above equation, the following upper bound obtained in the proof of Theorem 3 in \cite{tirinzoni2018transfer}:
	\begin{align}
		\mathbb{E}_{\mathcal{N}(\theta^*,cI)}\left[\left|\left|\tilde{B}_\theta\right|\right|_\nu^2\right] \le 2\left|\left|\tilde{B}_{\theta^*}\right|\right|_\nu^2 + \frac{1}{2}\gamma^2\kappa^2c^2\phi_{max}^4+c(\theta_{max}\phi_{max}(1+\gamma))^2,
	\end{align}
	and setting c=$\frac{1}{N}$ (since the bound hold for any constant parameter $c>0$), $c_1=\frac{8R_{max}^2}{\sqrt{2}(1-\gamma)^2}$, $c_2=\theta_{max}^2\phi_{max}^2(1-\gamma)^2+\psi p \log\sigma^2 + 2\psi\sum_j \chi_j^{(2)}\log\frac{1}{c_j^{\hat{p}}}+2\psi\log\frac{1}{S}$, $c_3=\frac{1}{2}\gamma^2\kappa^2\phi_{max}^4 + \frac{\psi p}{\sigma^2}$ and $\varphi(\Theta_s)=\frac{1}{\sigma^2}\sum_{j}\chi_j^{(2)}||\theta^*-\theta_j||$, we can rewrite Equation (\ref{eq:48}) in the following way:
	\begin{align}
		\mathbb{E}_{q_{\hat{\xi}}}\left[\left|\left|\tilde{B}_\theta\right|\right|_\nu^2\right] \le 2\left|\left|\tilde{B}_{\theta^*}\right|\right|_\nu^2 +v(\theta^*) + c_1\sqrt{\frac{\log\frac{2}{\delta}}{N}} + \frac{c_2 +\psi p \log N +\psi\varphi(\Theta_s)}{N} + \frac{c_3}{N}
	\end{align}
\end{proof}

\section{Experimental Details}\label{appx:experiments}
In this section, we provide some additional experimental details together with further results. 
\subsection{Parametrization}
ADAM \cite{kingma2014adam} is used in every experiment as optimizer. The source tasks are solved by a direct minimization of the TD error as described in section 3.4 of \cite{tirinzoni2018transfer}, using a \textit{batch size} of $50$ for the rooms environments and of $32$ for Mountain Car, a \textit{buffer size} of $50000$, the projection parameter of the mellow-max TD error gradient set to $0.5$, the learning rate $\alpha=10^{-3}$. The exploration is $\epsilon$-greedy with $\epsilon$ linearly decaying from $1$ to $0.01$ for Mountain Car and to $0.02$ for the rooms environments. Both decays happen within $50\%$ of the maximum number of learning iterations.

In the \textbf{rooms} environments, for what concern the two transfer algorithms, $c$-T2VT, and $c$-MGVT, we have the following parametrization: \textit{batch size} of $50$, \textit{buffer size} of $50000$, projection parameter of the mellow-max TD error gradient set to $0.5$ (see section 3.4 of \cite{tirinzoni2018transfer}), the parameter of Equation (\ref{eq:ELBO}) $\psi=10^{-6}$, $10$ weights to estimate the expected TD error, the learning rates are set to $\alpha_\mu=10^{-3}$ and $\alpha_L=0.1$ for the mean and the Cholesky factor $L$ of the posterior (moreover, the minimum eigenvalue reachable by $L$ is set to $\sigma^2_{min}=10^{-4}$). Finally, for the prior, we use a diagonal isotropic matrix $H=10^{-5}I$ and $\lambda=0.3333$ in the context of $c$-T2VT, furthermore, we have  $\Sigma=10^{-5}I$ for the prior in the context of $c$-MGVT.

In the \textbf{Mountain Car} environment, $c$-T2VT and $c$-MGVT are parametrized in the following way: \textit{batch size} of $500$, \textit{buffer size} of $10000$, projection parameter of the mellow-max TD error gradient  set to $0.5$, the parameter of Equation (\ref{eq:ELBO}) $\psi=10^{-4}$, $10$ weights to estimate the expected TD error, the learning rates are set to $\alpha_\mu=10^{-3}$ and $\alpha_L=10^{-4}$ for the mean and the Cholesky factor $L$ of the posterior (moreover, the minimum eigenvalue reachable by $L$ is set to $\sigma^2_{min}=10^{-4}$). Finally, for the prior, we use a diagonal isotropic matrix $H=10^{-5}I$ and $\lambda=0.3333$ in the context of $c$-T2VT, furthermore, we have $\Sigma=10^{-5}I$ for the prior in the context of $c$-MGVT.

\subsection{Temporal Dynamics}
In this section, we provide the analytical form of the different dynamics employed in our experiments. Notice that these dynamics need to be plugged into the mean of our Gaussian distribution from where we sample the parametrization defining the task (for the rooms environment we will sample the positions of the doors, whereas, for the Mountain Car environment, we will sample the base speed).
\begin{itemize}
	\item \textbf{Linear}: $2t-1,~t~\in~[0,1]$;
	\item \textbf{Polynomial}: $ at^4 + bt^3 + ct^2 + dt + e, ~t~\in~[0,1]$ and $a = -15.625$, $b = 39.5833$, $c = -31.875$, $d = 9.91667$ and $e = -1$;
	\item \textbf{Sinusoidal}: $\sin(2\pi t),~t~\in~[0,1]$.
\end{itemize}
In Figure \ref{fig:dynamics}, we report the graphical representation of the above analytical functions.

Now, given the range for a parameter $[k_{min},k_{max}]$, a given dynamic will span over this interval in the following way: $d(t)\frac{(k_{max}-k_{min})}{2}+\frac{(k_{max}+k_{min})}{2}$. Finally, notice that, $[k_{min},k_{max}]=[0.001, 0.0015]$ for Mountain Car, whereas $[k_{min},k_{max}]= [0.7 + padding, 9.3 - padding]$ for the parameters of the rooms environments. The $padding$ variable is $0$ for the 2-rooms, whereas is $2$ for the 3-rooms environments. This $padding$ variable was necessary in the 3-rooms environments in order for the TD gradient algorithm to be able to solve the source tasks in every configuration of the two doors.

\input{figure_temporal-dynamics.tex}

\subsection{$\lambda$-sensitivity results}\label{subsec:lambda-sensitivity-results}
In Figures \ref{fig:2-rooms-lambda-1c} and \ref{fig:2-rooms-lambda-3c}, we report a sensitivity analysis of our algorithm \wrt $\lambda$ in the 2-rooms environment. This analysis is carried out computing the performance of the learning algorithm \wrt different values of the previously mentioned parameter (whereas $H=10^{-5}I$ for every $\lambda$). These results are also compared with the performance of the algorithm when $\lambda$ is chosen according to the likelihood optimization described in Section \ref{subsec:sensitivity}. In Figures \ref{fig:3-rooms-lambda-1c} and \ref{fig:3-rooms-lambda-3c}, we report the above-described analysis in the context of the 3-rooms environment. In general, the performance of the likelihood approach is satisfying, for both $1$-T2VT and $3$-T2VT, even though in some cases it is not optimal. For what concern, the polynomial dynamic this may be due to its plateau (see Figure \ref{fig:dynamics}) which bias the choice for $\lambda$ toward bigger values since the likelihood is evaluated in a cross-validation manner. For the same reason, in the $\sin$ dynamic case, the likelihood-based approach tends to select an average $\lambda$. Finally, the linear case in the 2-rooms is almost optimal, whereas, in the 3-rooms, the performance decreases. This is due to the fact that, in the 3-rooms environment, we have 2 parameters governing the dynamics (the two doors positions) making the choice of $\lambda$ harder to make in this setting.

\textbf{Implementation Details}: since the $\lambda ~\in~ [0,1]$, we performed a grid search in order to optimize Equation \ref{eq:likelihood-optimization} .

\input{figure_lambda-sensitivity-2-rooms-1c.tex}

\input{figure_lambda-sensitivity-2-rooms-3c.tex}

\input{figure_lambda-sensitivity-3-rooms-1c.tex}

\input{figure_lambda-sensitivity-3-rooms-3c.tex}
\end{document}

%% file: figure_2-rooms.tex
\begin{figure*}[!t]
	\centering
	%
	%
	\begin{tikzpicture}
	\begin{customlegend}[legend columns=8,legend style={align=left,draw=none,column sep=2ex,font=\footnotesize},legend entries={ 1-T2VT, 1-MGVT, 3-T2VT, 3-MGVT}]
	%
	%
	%
	\addlegendimage{green1!60!black,ultra thick,dashed}   
	\addlegendimage{orange!80!white,dotted,ultra thick}
	\addlegendimage{blue1,ultra thick}
	\addlegendimage{violet!90!white,dash dot,ultra thick}
	
	\end{customlegend}
	\end{tikzpicture}
\subfigure[2-rooms polynomial dynamic.]
{
	\begin{tikzpicture}
	\begin{axis}[
	width=0.32\textwidth,
	height=4cm,
	xmin=50,
	xmax=2950,
	xtick={500,1000,...,3000},
	ymin=0,
	ymax=0.9,
	ytick={0,0.2,...,0.8},
	%
	%
	xlabel=Iterations,
	ylabel=Average Return,
	mark options={scale=0.2},
	cycle list name = custom,
	scaled x ticks=base 10:-3
	]
	\addplot
	table [x=i,y=mean-1-T2VT,col sep=comma] 
	{csv/2-rooms/polynomial-lrev.csv};
	\addplot[name path=top,draw=none,forget plot]
	table [name path=top,x=i,y expr=\thisrow{mean-1-T2VT}+\thisrow{std-1-T2VT},col sep=comma] 
	{csv/2-rooms/polynomial-lrev.csv};
	\addplot[name path=bot,draw=none,forget plot]
	table [name path=top,x=i,y expr=\thisrow{mean-1-T2VT}-\thisrow{std-1-T2VT},col sep=comma] 
	{csv/2-rooms/polynomial-lrev.csv};
	\addplot [forget plot, draw=none,opacity=0.4,pattern=north east lines,fill=green1!60!black]
	fill between[of=top and bot];
	\addplot
	table [x=i,y=mean-1-MGVT,col sep=comma] 
	{csv/2-rooms/polynomial-lrev.csv};
	\addplot[name path=top,draw=none,forget plot]
	table [name path=top,x=i,y expr=\thisrow{mean-1-MGVT}+\thisrow{std-1-MGVT},col sep=comma] 
	{csv/2-rooms/polynomial-lrev.csv};
	\addplot[name path=bot,draw=none,forget plot]
	table [name path=top,x=i,y expr=\thisrow{mean-1-MGVT}-\thisrow{std-1-MGVT},col sep=comma] 
	{csv/2-rooms/polynomial-lrev.csv};
	\addplot [forget plot, draw=none,opacity=0.4,pattern=north east lines,fill=orange]
	fill between[of=top and bot];
	\addplot
	table [x=i,y=mean-3-T2VT,col sep=comma] 
	{csv/2-rooms/polynomial-lrev.csv};
	\addplot[name path=top,draw=none,forget plot]
	table [name path=top,x=i,y expr=\thisrow{mean-3-T2VT}+\thisrow{std-3-T2VT},col sep=comma] 
	{csv/2-rooms/polynomial-lrev.csv};
	\addplot[name path=bot,draw=none,forget plot]
	table [name path=top,x=i,y expr=\thisrow{mean-3-T2VT}-\thisrow{std-3-T2VT},col sep=comma] 
	{csv/2-rooms/polynomial-lrev.csv};
	\addplot [forget plot, draw=none,opacity=0.15,pattern=north east lines,fill=blue1]
	fill between[of=top and bot];
	\addplot
	table [x=i,y=mean-3-MGVT,col sep=comma] 
	{csv/2-rooms/polynomial-lrev.csv};
	\addplot[name path=top,draw=none,forget plot]
	table [name path=top,x=i,y expr=\thisrow{mean-3-MGVT}+\thisrow{std-3-MGVT},col sep=comma] 
	{csv/2-rooms/polynomial-lrev.csv};
	\addplot[name path=bot,draw=none,forget plot]
	table [name path=top,x=i,y expr=\thisrow{mean-3-MGVT}-\thisrow{std-3-MGVT},col sep=comma] 
	{csv/2-rooms/polynomial-lrev.csv};
	\addplot [forget plot, draw=none,opacity=0.15,pattern=north east lines,fill=violet]
	fill between[of=top and bot];

	\end{axis}
	\end{tikzpicture}
	\label{fig:2-rooms-polynomial}
}
\quad
\subfigure[2-rooms linear dynamic.]
{
	\begin{tikzpicture}
	\begin{axis}[
	width=0.32\textwidth,
	height=4cm,
	xmin=50,
	xmax=2950,
	xtick={500,1000,...,3000},
	ymin=0,
	ymax=0.9,
	ytick={0,0.2,...,0.8},
	%
	%
	xlabel=Iterations,
	ylabel=Average Return,
	mark options={scale=0.2},
	cycle list name = custom,
	scaled x ticks=base 10:-3
	]
	\addplot
	table [x=i,y=mean-1-T2VT,col sep=comma] 
	{csv/2-rooms/linear-lrev.csv};
	\addplot[name path=top,draw=none,forget plot]
	table [name path=top,x=i,y expr=\thisrow{mean-1-T2VT}+\thisrow{std-1-T2VT},col sep=comma] 
	{csv/2-rooms/linear-lrev.csv};
	\addplot[name path=bot,draw=none,forget plot]
	table [name path=top,x=i,y expr=\thisrow{mean-1-T2VT}-\thisrow{std-1-T2VT},col sep=comma] 
	{csv/2-rooms/linear-lrev.csv};
	\addplot [forget plot, draw=none,opacity=0.4,pattern=north east lines,fill=green1!60!black]
	fill between[of=top and bot];
	\addplot
	table [x=i,y=mean-1-MGVT,col sep=comma] 
	{csv/2-rooms/linear-lrev.csv};
	\addplot[name path=top,draw=none,forget plot]
	table [name path=top,x=i,y expr=\thisrow{mean-1-MGVT}+\thisrow{std-1-MGVT},col sep=comma] 
	{csv/2-rooms/linear-lrev.csv};
	\addplot[name path=bot,draw=none,forget plot]
	table [name path=top,x=i,y expr=\thisrow{mean-1-MGVT}-\thisrow{std-1-MGVT},col sep=comma] 
	{csv/2-rooms/linear-lrev.csv};
	\addplot [forget plot, draw=none,opacity=0.4,pattern=north east lines,fill=orange]
	fill between[of=top and bot];
	\addplot
	table [x=i,y=mean-3-T2VT,col sep=comma] 
	{csv/2-rooms/linear-lrev.csv};
	\addplot[name path=top,draw=none,forget plot]
	table [name path=top,x=i,y expr=\thisrow{mean-3-T2VT}+\thisrow{std-3-T2VT},col sep=comma] 
	{csv/2-rooms/linear-lrev.csv};
	\addplot[name path=bot,draw=none,forget plot]
	table [name path=top,x=i,y expr=\thisrow{mean-3-T2VT}-\thisrow{std-3-T2VT},col sep=comma] 
	{csv/2-rooms/linear-lrev.csv};
	\addplot [forget plot, draw=none,opacity=0.15,pattern=north east lines,fill=blue1]
	fill between[of=top and bot];
	\addplot
	table [x=i,y=mean-3-MGVT,col sep=comma] 
	{csv/2-rooms/linear-lrev.csv};
	\addplot[name path=top,draw=none,forget plot]
	table [name path=top,x=i,y expr=\thisrow{mean-3-MGVT}+\thisrow{std-3-MGVT},col sep=comma] 
	{csv/2-rooms/linear-lrev.csv};
	\addplot[name path=bot,draw=none,forget plot]
	table [name path=top,x=i,y expr=\thisrow{mean-3-MGVT}-\thisrow{std-3-MGVT},col sep=comma] 
	{csv/2-rooms/linear-lrev.csv};
	\addplot [forget plot, draw=none,opacity=0.15,pattern=north east lines,fill=violet]
	fill between[of=top and bot];

	\end{axis}
	\end{tikzpicture}
	\label{fig:2-rooms-linear}
}
\quad
\subfigure[2-rooms $\sin$ dynamic.]
{
	\begin{tikzpicture}
	\begin{axis}[
	width=0.32\textwidth,
	height=4cm,
	xmin=50,
	xmax=2950,
	xtick={500,1000,...,3000},
	ymin=0,
	ymax=0.9,
	ytick={0,0.2,...,0.8},
	%
	%
	xlabel=Iterations,
	ylabel=Average Return,
	mark options={scale=0.2},
	cycle list name = custom,
	scaled x ticks=base 10:-3
	]
	\addplot
	table [x=i,y=mean-1-T2VT,col sep=comma] 
	{csv/2-rooms/sin-rep-lamb-03333-lrev.csv};
	\addplot[name path=top,draw=none,forget plot]
	table [name path=top,x=i,y expr=\thisrow{mean-1-T2VT}+\thisrow{std-1-T2VT},col sep=comma] 
	{csv/2-rooms/sin-rep-lamb-03333-lrev.csv};
	\addplot[name path=bot,draw=none,forget plot]
	table [name path=top,x=i,y expr=\thisrow{mean-1-T2VT}-\thisrow{std-1-T2VT},col sep=comma] 
	{csv/2-rooms/sin-rep-lamb-03333-lrev.csv};
	\addplot [forget plot, draw=none,opacity=0.4,pattern=north east lines,fill=green1!60!black]
	fill between[of=top and bot];
	\addplot
	table [x=i,y=mean-1-MGVT,col sep=comma] 
	{csv/2-rooms/sin-rep-lamb-03333-lrev.csv};
	\addplot[name path=top,draw=none,forget plot]
	table [name path=top,x=i,y expr=\thisrow{mean-1-MGVT}+\thisrow{std-1-MGVT},col sep=comma] 
	{csv/2-rooms/sin-rep-lamb-03333-lrev.csv};
	\addplot[name path=bot,draw=none,forget plot]
	table [name path=top,x=i,y expr=\thisrow{mean-1-MGVT}-\thisrow{std-1-MGVT},col sep=comma] 
	{csv/2-rooms/sin-rep-lamb-03333-lrev.csv};
	\addplot [forget plot, draw=none,opacity=0.4,pattern=north east lines,fill=orange]
	fill between[of=top and bot];
	\addplot
	table [x=i,y=mean-3-T2VT,col sep=comma] 
	{csv/2-rooms/sin-rep-lamb-03333-lrev.csv};
	\addplot[name path=top,draw=none,forget plot]
	table [name path=top,x=i,y expr=\thisrow{mean-3-T2VT}+\thisrow{std-3-T2VT},col sep=comma] 
	{csv/2-rooms/sin-rep-lamb-03333-lrev.csv};
	\addplot[name path=bot,draw=none,forget plot]
	table [name path=top,x=i,y expr=\thisrow{mean-3-T2VT}-\thisrow{std-3-T2VT},col sep=comma] 
	{csv/2-rooms/sin-rep-lamb-03333-lrev.csv};
	\addplot [forget plot, draw=none,opacity=0.15,pattern=north east lines,fill=blue1]
	fill between[of=top and bot];
	\addplot
	table [x=i,y=mean-3-MGVT,col sep=comma] 
	{csv/2-rooms/sin-rep-lamb-03333-lrev.csv};
	\addplot[name path=top,draw=none,forget plot]
	table [name path=top,x=i,y expr=\thisrow{mean-3-MGVT}+\thisrow{std-3-MGVT},col sep=comma] 
	{csv/2-rooms/sin-rep-lamb-03333-lrev.csv};
	\addplot[name path=bot,draw=none,forget plot]
	table [name path=top,x=i,y expr=\thisrow{mean-3-MGVT}-\thisrow{std-3-MGVT},col sep=comma] 
	{csv/2-rooms/sin-rep-lamb-03333-lrev.csv};
	\addplot [forget plot, draw=none,opacity=0.15,pattern=north east lines,fill=violet]
	fill between[of=top and bot];

	\end{axis}
	\end{tikzpicture}
	\label{fig:2-rooms-sin-rep-lamb-03333}
}

	\caption{Average return achived by the algorithms with $95\%$ confidence intervals computed using $50$ independent runs.}
	\label{fig:2-rooms}
\end{figure*}

%% file: figure_3-rooms.tex
\begin{figure*}[!t]
	\centering
	%
	%
	\begin{tikzpicture}
	\begin{customlegend}[legend columns=8,legend style={align=left,draw=none,column sep=2ex,font=\footnotesize},legend entries={ 1-T2VT, 1-MGVT, 3-T2VT, 3-MGVT}]
	%
	%
	%
	\addlegendimage{green1!60!black,ultra thick,dashed}   
	\addlegendimage{orange!80!white,dotted,ultra thick}
	\addlegendimage{blue1,ultra thick}
	\addlegendimage{violet!90!white,dash dot,ultra thick}
	
	\end{customlegend}
	\end{tikzpicture}
\subfigure[3-rooms polynomial dynamic.]
{
	\begin{tikzpicture}
	\begin{axis}[
	width=0.32\textwidth,
	height=4cm,
	xmin=50,
	xmax=14950,
	xtick={3000,6000,...,12000},
	ymin=0,
	ymax=0.7,
	ytick={0,0.2,...,0.6},
	%
	%
	xlabel=Iterations,
	ylabel=Average Return,
	mark options={scale=0.2},
	cycle list name = custom,
	scaled x ticks=base 10:-3
	]
	\addplot
	table [x=i,y=mean-1-T2VT,col sep=comma] 
	{csv/3-rooms/polynomial-lrev.csv};
	\addplot[name path=top,draw=none,forget plot]
	table [name path=top,x=i,y expr=\thisrow{mean-1-T2VT}+\thisrow{std-1-T2VT},col sep=comma] 
	{csv/3-rooms/polynomial-lrev.csv};
	\addplot[name path=bot,draw=none,forget plot]
	table [name path=top,x=i,y expr=\thisrow{mean-1-T2VT}-\thisrow{std-1-T2VT},col sep=comma] 
	{csv/3-rooms/polynomial-lrev.csv};
	\addplot [forget plot, draw=none,opacity=0.4,pattern=north east lines,fill=green1!60!black]
	fill between[of=top and bot];
	\addplot
	table [x=i,y=mean-1-MGVT,col sep=comma] 
	{csv/3-rooms/polynomial-lrev.csv};
	\addplot[name path=top,draw=none,forget plot]
	table [name path=top,x=i,y expr=\thisrow{mean-1-MGVT}+\thisrow{std-1-MGVT},col sep=comma] 
	{csv/3-rooms/polynomial-lrev.csv};
	\addplot[name path=bot,draw=none,forget plot]
	table [name path=top,x=i,y expr=\thisrow{mean-1-MGVT}-\thisrow{std-1-MGVT},col sep=comma] 
	{csv/3-rooms/polynomial-lrev.csv};
	\addplot [forget plot, draw=none,opacity=0.4,pattern=north east lines,fill=orange]
	fill between[of=top and bot];
	\addplot
	table [x=i,y=mean-3-T2VT,col sep=comma] 
	{csv/3-rooms/polynomial-lrev.csv};
	\addplot[name path=top,draw=none,forget plot]
	table [name path=top,x=i,y expr=\thisrow{mean-3-T2VT}+\thisrow{std-3-T2VT},col sep=comma] 
	{csv/3-rooms/polynomial-lrev.csv};
	\addplot[name path=bot,draw=none,forget plot]
	table [name path=top,x=i,y expr=\thisrow{mean-3-T2VT}-\thisrow{std-3-T2VT},col sep=comma] 
	{csv/3-rooms/polynomial-lrev.csv};
	\addplot [forget plot, draw=none,opacity=0.15,pattern=north east lines,fill=blue1]
	fill between[of=top and bot];
	\addplot
	table [x=i,y=mean-3-MGVT,col sep=comma] 
	{csv/3-rooms/polynomial-lrev.csv};
	\addplot[name path=top,draw=none,forget plot]
	table [name path=top,x=i,y expr=\thisrow{mean-3-MGVT}+\thisrow{std-3-MGVT},col sep=comma] 
	{csv/3-rooms/polynomial-lrev.csv};
	\addplot[name path=bot,draw=none,forget plot]
	table [name path=top,x=i,y expr=\thisrow{mean-3-MGVT}-\thisrow{std-3-MGVT},col sep=comma] 
	{csv/3-rooms/polynomial-lrev.csv};
	\addplot [forget plot, draw=none,opacity=0.15,pattern=north east lines,fill=violet]
	fill between[of=top and bot];

	\end{axis}
	\end{tikzpicture}
	\label{fig:3-rooms-polynomial}
}
\quad
\subfigure[3-rooms linear dynamic.]
{
	\begin{tikzpicture}
	\begin{axis}[
	width=0.33\textwidth,
	height=4cm,
	xmin=50,
	xmax=14950,
	xtick={3000,6000,...,12000},
	ymin=0,
	ymax=0.75,
	ytick={0.1,0.3,...,0.7},
	%
	%
	xlabel=Iterations,
	ylabel=Average Return,
	mark options={scale=0.2},
	cycle list name = custom,
	scaled x ticks=base 10:-3
	]
	\addplot
	table [x=i,y=mean-1-T2VT,col sep=comma] 
	{csv/3-rooms/linear-lrev.csv};
	\addplot[name path=top,draw=none,forget plot]
	table [name path=top,x=i,y expr=\thisrow{mean-1-T2VT}+\thisrow{std-1-T2VT},col sep=comma] 
	{csv/3-rooms/linear-lrev.csv};
	\addplot[name path=bot,draw=none,forget plot]
	table [name path=top,x=i,y expr=\thisrow{mean-1-T2VT}-\thisrow{std-1-T2VT},col sep=comma] 
	{csv/3-rooms/linear-lrev.csv};
	\addplot [forget plot, draw=none,opacity=0.4,pattern=north east lines,fill=green1!60!black]
	fill between[of=top and bot];
	\addplot
	table [x=i,y=mean-1-MGVT,col sep=comma] 
	{csv/3-rooms/linear-lrev.csv};
	\addplot[name path=top,draw=none,forget plot]
	table [name path=top,x=i,y expr=\thisrow{mean-1-MGVT}+\thisrow{std-1-MGVT},col sep=comma] 
	{csv/3-rooms/linear-lrev.csv};
	\addplot[name path=bot,draw=none,forget plot]
	table [name path=top,x=i,y expr=\thisrow{mean-1-MGVT}-\thisrow{std-1-MGVT},col sep=comma] 
	{csv/3-rooms/linear-lrev.csv};
	\addplot [forget plot, draw=none,opacity=0.4,pattern=north east lines,fill=orange]
	fill between[of=top and bot];
	\addplot
	table [x=i,y=mean-3-T2VT,col sep=comma] 
	{csv/3-rooms/linear-lrev.csv};
	\addplot[name path=top,draw=none,forget plot]
	table [name path=top,x=i,y expr=\thisrow{mean-3-T2VT}+\thisrow{std-3-T2VT},col sep=comma] 
	{csv/3-rooms/linear-lrev.csv};
	\addplot[name path=bot,draw=none,forget plot]
	table [name path=top,x=i,y expr=\thisrow{mean-3-T2VT}-\thisrow{std-3-T2VT},col sep=comma] 
	{csv/3-rooms/linear-lrev.csv};
	\addplot [forget plot, draw=none,opacity=0.15,pattern=north east lines,fill=blue1]
	fill between[of=top and bot];
	\addplot
	table [x=i,y=mean-3-MGVT,col sep=comma] 
	{csv/3-rooms/linear-lrev.csv};
	\addplot[name path=top,draw=none,forget plot]
	table [name path=top,x=i,y expr=\thisrow{mean-3-MGVT}+\thisrow{std-3-MGVT},col sep=comma] 
	{csv/3-rooms/linear-lrev.csv};
	\addplot[name path=bot,draw=none,forget plot]
	table [name path=top,x=i,y expr=\thisrow{mean-3-MGVT}-\thisrow{std-3-MGVT},col sep=comma] 
	{csv/3-rooms/linear-lrev.csv};
	\addplot [forget plot, draw=none,opacity=0.15,pattern=north east lines,fill=violet]
	fill between[of=top and bot];	

	\end{axis}
	\end{tikzpicture}
	\label{fig:3-rooms-linear}
}
\quad
\subfigure[3-rooms $\sin$ dynamic.]
{
	\begin{tikzpicture}
	\begin{axis}[
	width=0.32\textwidth,
	height=4cm,
	xmin=50,
	xmax=14950,
	xtick={3000,6000,...,12000},
	ymin=0,
	ymax=0.9,
	ytick={0,0.2,...,0.8},
	%
	%
	xlabel=Iterations,
	ylabel=Average Return,
	mark options={scale=0.2},
	cycle list name = custom,
	scaled x ticks=base 10:-3
	]
	\addplot
	table [x=i,y=mean-1-T2VT,col sep=comma] 
	{csv/3-rooms/sin-rep-lamb-03333-lrev.csv};
	\addplot[name path=top,draw=none,forget plot]
	table [name path=top,x=i,y expr=\thisrow{mean-1-T2VT}+\thisrow{std-1-T2VT},col sep=comma] 
	{csv/3-rooms/sin-rep-lamb-03333-lrev.csv};
	\addplot[name path=bot,draw=none,forget plot]
	table [name path=top,x=i,y expr=\thisrow{mean-1-T2VT}-\thisrow{std-1-T2VT},col sep=comma] 
	{csv/3-rooms/sin-rep-lamb-03333-lrev.csv};
	\addplot [forget plot, draw=none,opacity=0.4,pattern=north east lines,fill=green1!60!black]
	fill between[of=top and bot];
	\addplot
	table [x=i,y=mean-1-MGVT,col sep=comma] 
	{csv/3-rooms/sin-rep-lamb-03333-lrev.csv};
	\addplot[name path=top,draw=none,forget plot]
	table [name path=top,x=i,y expr=\thisrow{mean-1-MGVT}+\thisrow{std-1-MGVT},col sep=comma] 
	{csv/3-rooms/sin-rep-lamb-03333-lrev.csv};
	\addplot[name path=bot,draw=none,forget plot]
	table [name path=top,x=i,y expr=\thisrow{mean-1-MGVT}-\thisrow{std-1-MGVT},col sep=comma] 
	{csv/3-rooms/sin-rep-lamb-03333-lrev.csv};
	\addplot [forget plot, draw=none,opacity=0.4,pattern=north east lines,fill=orange]
	fill between[of=top and bot];
	\addplot
	table [x=i,y=mean-3-T2VT,col sep=comma] 
	{csv/3-rooms/sin-rep-lamb-03333-lrev.csv};
	\addplot[name path=top,draw=none,forget plot]
	table [name path=top,x=i,y expr=\thisrow{mean-3-T2VT}+\thisrow{std-3-T2VT},col sep=comma] 
	{csv/3-rooms/sin-rep-lamb-03333-lrev.csv};
	\addplot[name path=bot,draw=none,forget plot]
	table [name path=top,x=i,y expr=\thisrow{mean-3-T2VT}-\thisrow{std-3-T2VT},col sep=comma] 
	{csv/3-rooms/sin-rep-lamb-03333-lrev.csv};
	\addplot [forget plot, draw=none,opacity=0.15,pattern=north east lines,fill=blue1]
	fill between[of=top and bot];
	\addplot
	table [x=i,y=mean-3-MGVT,col sep=comma] 
	{csv/3-rooms/sin-rep-lamb-03333-lrev.csv};
	\addplot[name path=top,draw=none,forget plot]
	table [name path=top,x=i,y expr=\thisrow{mean-3-MGVT}+\thisrow{std-3-MGVT},col sep=comma] 
	{csv/3-rooms/sin-rep-lamb-03333-lrev.csv};
	\addplot[name path=bot,draw=none,forget plot]
	table [name path=top,x=i,y expr=\thisrow{mean-3-MGVT}-\thisrow{std-3-MGVT},col sep=comma] 
	{csv/3-rooms/sin-rep-lamb-03333-lrev.csv};
	\addplot [forget plot, draw=none,opacity=0.15,pattern=north east lines,fill=violet]
	fill between[of=top and bot];

	\end{axis}
	\end{tikzpicture}
	\label{fig:3-rooms-sin-rep-lamb-03333}
}

	\caption{Average return achived by the algorithms with $95\%$ confidence intervals computed using $50$ independent runs.}
	\label{fig:3-rooms}
\end{figure*}

%% file: figure_mountain-car.tex
\begin{figure*}[!t]
	\centering
	%
	%
	\begin{tikzpicture}
	\begin{customlegend}[legend columns=8,legend style={align=left,draw=none,column sep=2ex,font=\footnotesize},legend entries={ 1-T2VT, 1-MGVT, 3-T2VT, 3-MGVT}]
	%
	%
	%
	\addlegendimage{green1!60!black,ultra thick,dashed}   
	\addlegendimage{orange!80!white,dotted,ultra thick}
	\addlegendimage{blue1,ultra thick}
	\addlegendimage{violet!90!white,dash dot,ultra thick}
	
	\end{customlegend}
	\end{tikzpicture}
\subfigure[Mountain Car polynomial dynamic.]
{
	\begin{tikzpicture}
	\begin{axis}[
	width=0.32\textwidth,
	height=4cm,
	xmin=200,
	xmax=74800,
	minor x tick num=5,
	ymin=-75,
	ymax=-45,
	ytick={-75,-70,...,-45},
	%
	%
	xlabel=Iterations,
	ylabel=Average Return,
	mark options={scale=0.2},
	cycle list name = custom,
	scaled x ticks=base 10:-3
	]
	\addplot
	table [x=i,y=mean-1-T2VT,col sep=comma] 
	{csv/mountain-car/polynomial-lrev.csv};
	\addplot[name path=top,draw=none,forget plot]
	table [name path=top,x=i,y expr=\thisrow{mean-1-T2VT}+\thisrow{std-1-T2VT},col sep=comma] 
	{csv/mountain-car/polynomial-lrev.csv};
	\addplot[name path=bot,draw=none,forget plot]
	table [name path=top,x=i,y expr=\thisrow{mean-1-T2VT}-\thisrow{std-1-T2VT},col sep=comma] 
	{csv/mountain-car/polynomial-lrev.csv};
	\addplot [forget plot, draw=none,opacity=0.4,pattern=north east lines,fill=green1!60!black]
	fill between[of=top and bot];
	\addplot
	table [x=i,y=mean-1-MGVT,col sep=comma] 
	{csv/mountain-car/polynomial-lrev.csv};
	\addplot[name path=top,draw=none,forget plot]
	table [name path=top,x=i,y expr=\thisrow{mean-1-MGVT}+\thisrow{std-1-MGVT},col sep=comma] 
	{csv/mountain-car/polynomial-lrev.csv};
	\addplot[name path=bot,draw=none,forget plot]
	table [name path=top,x=i,y expr=\thisrow{mean-1-MGVT}-\thisrow{std-1-MGVT},col sep=comma] 
	{csv/mountain-car/polynomial-lrev.csv};
	\addplot [forget plot, draw=none,opacity=0.4,pattern=north east lines,fill=orange]
	fill between[of=top and bot];
	\addplot
	table [x=i,y=mean-3-T2VT,col sep=comma] 
	{csv/mountain-car/polynomial-lrev.csv};
	\addplot[name path=top,draw=none,forget plot]
	table [name path=top,x=i,y expr=\thisrow{mean-3-T2VT}+\thisrow{std-3-T2VT},col sep=comma] 
	{csv/mountain-car/polynomial-lrev.csv};
	\addplot[name path=bot,draw=none,forget plot]
	table [name path=top,x=i,y expr=\thisrow{mean-3-T2VT}-\thisrow{std-3-T2VT},col sep=comma] 
	{csv/mountain-car/polynomial-lrev.csv};
	\addplot [forget plot, draw=none,opacity=0.15,pattern=north east lines,fill=blue1]
	fill between[of=top and bot];
	\addplot
	table [x=i,y=mean-3-MGVT,col sep=comma] 
	{csv/mountain-car/polynomial-lrev.csv};
	\addplot[name path=top,draw=none,forget plot]
	table [name path=top,x=i,y expr=\thisrow{mean-3-MGVT}+\thisrow{std-3-MGVT},col sep=comma] 
	{csv/mountain-car/polynomial-lrev.csv};
	\addplot[name path=bot,draw=none,forget plot]
	table [name path=top,x=i,y expr=\thisrow{mean-3-MGVT}-\thisrow{std-3-MGVT},col sep=comma] 
	{csv/mountain-car/polynomial-lrev.csv};
	\addplot [forget plot, draw=none,opacity=0.15,pattern=north east lines,fill=violet]
	fill between[of=top and bot];

	\end{axis}
	\end{tikzpicture}
	\label{fig:mountain-car-polynomial}
}
\quad
\subfigure[Mountain Car linear dynamic.]
{
	\begin{tikzpicture}
	\begin{axis}[
	width=0.32\textwidth,
	height=4cm,
	xmin=200,
	xmax=74800,
	minor x tick num=5,
	ymin=-80,
	ymax=-40,
	ytick={-80,-70,...,-40},
	%
	%
	xlabel=Iterations,
	ylabel=Average Return,
	mark options={scale=0.2},
	cycle list name = custom,
	scaled x ticks=base 10:-4
	]
	\addplot
	table [x=i,y=mean-1-T2VT,col sep=comma] 
	{csv/mountain-car/linear-lrev.csv};
	\addplot[name path=top,draw=none,forget plot]
	table [name path=top,x=i,y expr=\thisrow{mean-1-T2VT}+\thisrow{std-1-T2VT},col sep=comma] 
	{csv/mountain-car/linear-lrev.csv};
	\addplot[name path=bot,draw=none,forget plot]
	table [name path=top,x=i,y expr=\thisrow{mean-1-T2VT}-\thisrow{std-1-T2VT},col sep=comma] 
	{csv/mountain-car/linear-lrev.csv};
	\addplot [forget plot, draw=none,opacity=0.4,pattern=north east lines,fill=green1!60!black]
	fill between[of=top and bot];
	\addplot
	table [x=i,y=mean-1-MGVT,col sep=comma] 
	{csv/mountain-car/linear-lrev.csv};
	\addplot[name path=top,draw=none,forget plot]
	table [name path=top,x=i,y expr=\thisrow{mean-1-MGVT}+\thisrow{std-1-MGVT},col sep=comma] 
	{csv/mountain-car/linear-lrev.csv};
	\addplot[name path=bot,draw=none,forget plot]
	table [name path=top,x=i,y expr=\thisrow{mean-1-MGVT}-\thisrow{std-1-MGVT},col sep=comma] 
	{csv/mountain-car/linear-lrev.csv};
	\addplot [forget plot, draw=none,opacity=0.4,pattern=north east lines,fill=orange]
	fill between[of=top and bot];
	\addplot
	table [x=i,y=mean-3-T2VT,col sep=comma] 
	{csv/mountain-car/linear-lrev.csv};
	\addplot[name path=top,draw=none,forget plot]
	table [name path=top,x=i,y expr=\thisrow{mean-3-T2VT}+\thisrow{std-3-T2VT},col sep=comma] 
	{csv/mountain-car/linear-lrev.csv};
	\addplot[name path=bot,draw=none,forget plot]
	table [name path=top,x=i,y expr=\thisrow{mean-3-T2VT}-\thisrow{std-3-T2VT},col sep=comma] 
	{csv/mountain-car/linear-lrev.csv};
	\addplot [forget plot, draw=none,opacity=0.15,pattern=north east lines,fill=blue1]
	fill between[of=top and bot];
	\addplot
	table [x=i,y=mean-3-MGVT,col sep=comma] 
	{csv/mountain-car/linear-lrev.csv};
	\addplot[name path=top,draw=none,forget plot]
	table [name path=top,x=i,y expr=\thisrow{mean-3-MGVT}+\thisrow{std-3-MGVT},col sep=comma] 
	{csv/mountain-car/linear-lrev.csv};
	\addplot[name path=bot,draw=none,forget plot]
	table [name path=top,x=i,y expr=\thisrow{mean-3-MGVT}-\thisrow{std-3-MGVT},col sep=comma] 
	{csv/mountain-car/linear-lrev.csv};
	\addplot [forget plot, draw=none,opacity=0.15,pattern=north east lines,fill=violet]
	fill between[of=top and bot];

	\end{axis}
	\end{tikzpicture}
	\label{fig:mountain-car-linear}
}	
\quad
\subfigure[Mountain Car $\sin$ dynamic.]
{
	\begin{tikzpicture}
	\begin{axis}[
	width=0.32\textwidth,
	height=4cm,
	xmin=200,
	xmax=74800,
	minor x tick num=5,
	ymin=-85,
	ymax=-60,
	ytick={-85,-80,...,-60},
	%
	%
	xlabel=Iterations,
	ylabel=Average Return,
	mark options={scale=0.2},
	cycle list name = custom,
	scaled x ticks=base 10:-3
	]
	\addplot
	table [x=i,y=mean-1-T2VT,col sep=comma] 
	{csv/mountain-car/sin-rep-lamb-03333-lrev.csv};
	\addplot[name path=top,draw=none,forget plot]
	table [name path=top,x=i,y expr=\thisrow{mean-1-T2VT}+\thisrow{std-1-T2VT},col sep=comma] 
	{csv/mountain-car/sin-rep-lamb-03333-lrev.csv};
	\addplot[name path=bot,draw=none,forget plot]
	table [name path=top,x=i,y expr=\thisrow{mean-1-T2VT}-\thisrow{std-1-T2VT},col sep=comma] 
	{csv/mountain-car/sin-rep-lamb-03333-lrev.csv};
	\addplot [forget plot, draw=none,opacity=0.4,pattern=north east lines,fill=green1!60!black]
	fill between[of=top and bot];
	\addplot
	table [x=i,y=mean-1-MGVT,col sep=comma] 
	{csv/mountain-car/sin-rep-lamb-03333-lrev.csv};
	\addplot[name path=top,draw=none,forget plot]
	table [name path=top,x=i,y expr=\thisrow{mean-1-MGVT}+\thisrow{std-1-MGVT},col sep=comma] 
	{csv/mountain-car/sin-rep-lamb-03333-lrev.csv};
	\addplot[name path=bot,draw=none,forget plot]
	table [name path=top,x=i,y expr=\thisrow{mean-1-MGVT}-\thisrow{std-1-MGVT},col sep=comma] 
	{csv/mountain-car/sin-rep-lamb-03333-lrev.csv};
	\addplot [forget plot, draw=none,opacity=0.4,pattern=north east lines,fill=orange]
	fill between[of=top and bot];
	\addplot
	table [x=i,y=mean-3-T2VT,col sep=comma] 
	{csv/mountain-car/sin-rep-lamb-03333-lrev.csv};
	\addplot[name path=top,draw=none,forget plot]
	table [name path=top,x=i,y expr=\thisrow{mean-3-T2VT}+\thisrow{std-3-T2VT},col sep=comma] 
	{csv/mountain-car/sin-rep-lamb-03333-lrev.csv};
	\addplot[name path=bot,draw=none,forget plot]
	table [name path=top,x=i,y expr=\thisrow{mean-3-T2VT}-\thisrow{std-3-T2VT},col sep=comma] 
	{csv/mountain-car/sin-rep-lamb-03333-lrev.csv};
	\addplot [forget plot, draw=none,opacity=0.15,pattern=north east lines,fill=blue1]
	fill between[of=top and bot];
	\addplot
	table [x=i,y=mean-3-MGVT,col sep=comma] 
	{csv/mountain-car/sin-rep-lamb-03333-lrev.csv};
	\addplot[name path=top,draw=none,forget plot]
	table [name path=top,x=i,y expr=\thisrow{mean-3-MGVT}+\thisrow{std-3-MGVT},col sep=comma] 
	{csv/mountain-car/sin-rep-lamb-03333-lrev.csv};
	\addplot[name path=bot,draw=none,forget plot]
	table [name path=top,x=i,y expr=\thisrow{mean-3-MGVT}-\thisrow{std-3-MGVT},col sep=comma] 
	{csv/mountain-car/sin-rep-lamb-03333-lrev.csv};
	\addplot [forget plot, draw=none,opacity=0.15,pattern=north east lines,fill=violet]
	fill between[of=top and bot];

	\end{axis}
	\end{tikzpicture}
	\label{fig:mountain-car-sin-rep-lamb-03333}
}

	\caption{Average return achived by the algorithms with $95\%$ confidence intervals computed using $50$ independent runs.}
	\label{fig:mountain-car}
\end{figure*}

%% file: figure_temporal-dynamics.tex
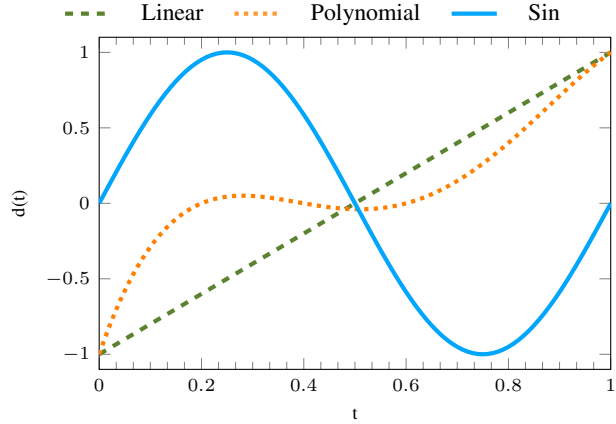
\begin{figure*}[!b]
	\centering
	%
	%
	\begin{tikzpicture}
	\begin{customlegend}[legend columns=8,legend style={align=left,draw=none,column sep=2ex,font=\footnotesize},legend entries={Linear, Polynomial, Sin}]
	%
	%
	%
	\addlegendimage{green1!60!black,ultra thick,dashed}   
	\addlegendimage{orange!80!white,dotted,ultra thick}
	\addlegendimage{blue1,ultra thick}
	\addlegendimage{violet!90!white,dash dot,ultra thick}
	
	\end{customlegend}
	\end{tikzpicture}

	\begin{tikzpicture}
	\begin{axis}[
	width=0.60\textwidth,
	height=6cm,
	xmin=0,
	xmax=1,
	minor x tick num=5,
	ymin=-1.1,
	ymax=1.1,
	ytick={-1,-0.5,...,1},
	%
	%
	xlabel=t,
	ylabel=d(t),
	mark options={scale=0.2},
	cycle list name = custom,
	]
	\addplot
	table [x=x,y=linear,col sep=comma] 
	{csv/dynamics/dynamics.csv};
	\addplot
	table [x=x,y=polynomial,col sep=comma] 
	{csv/dynamics/dynamics.csv};
	\addplot
	table [x=x,y=sin,col sep=comma] 
	{csv/dynamics/dynamics.csv};

	\end{axis}
	\end{tikzpicture}

	\caption{Temporal dynamics.}
	\label{fig:dynamics}
\end{figure*}

%% file: figure_lambda-sensitivity-2-rooms-1c.tex
\begin{figure*}[!b]
	\centering
	%
	%
	\begin{tikzpicture}
	\begin{customlegend}[legend columns=6,legend style={align=left,draw=none,column sep=2ex,font=\footnotesize},legend entries={ $\lambda=0.1$, $\lambda=0.2$, $\lambda=0.3$, $\lambda=0.4$, $\lambda=0.5$, $\lambda=0.6$, $\lambda=0.7$, $\lambda=0.8$, $\lambda=0.9$, $\lambda=1$, likelihood}]
	%
	%
	%
	\addlegendimage{green1!60!black,ultra thick,dashed}   
	\addlegendimage{orange!80!white,dotted,ultra thick}
	\addlegendimage{blue1,ultra thick}
	\addlegendimage{violet!90!white,dash dot,ultra thick}
	\addlegendimage{mikadoYellow, densely dashed,ultra thick}	
	\addlegendimage{lapisLazuli,densely dotted,ultra thick}
	\addlegendimage{harvardCrimson,densely dash dot, ultra thick}
	\addlegendimage{harlequin,ultra thick}
	\addlegendimage{neonFuchsia, densely dotted,ultra thick}
	\addlegendimage{bronze, densely dash dot,ultra thick}
	\addlegendimage{black,ultra thick}

	\end{customlegend}
	\end{tikzpicture}
	
\subfigure[2-rooms polynomial dynamic.]
{
	\begin{tikzpicture}
	\begin{axis}[
	width=0.45\textwidth,
	height=4.5cm,
	xmin=50,
	xmax=2950,
	xtick={500,1000,...,3000},
	ymin=0,
	ymax=0.9,
	ytick={0,0.2,...,0.8},
	%
	%
	xlabel=Iterations,
	ylabel=Average Return,
	mark options={scale=0.2},
	cycle list name = custom,
	scaled x ticks=base 10:-3
	]
	\addplot
	table [x=i,y=mean-lambda-0.1,col sep=comma] 
	{csv/lambda/2-rooms-1c/polynomial.csv};
	\addplot[name path=top,draw=none,forget plot]
	table [name path=top,x=i,y expr=\thisrow{mean-lambda-0.1}+\thisrow{std-lambda-0.1},col sep=comma] 
	{csv/lambda/2-rooms-1c/polynomial.csv};
	\addplot[name path=bot,draw=none,forget plot]
	table [name path=top,x=i,y expr=\thisrow{mean-lambda-0.1}-\thisrow{std-lambda-0.1},col sep=comma] 
	{csv/lambda/2-rooms-1c/polynomial.csv};
	\addplot [forget plot, draw=none,opacity=0.4,pattern=north east lines,fill=green1!60!black]
	fill between[of=top and bot];
	\addplot
	table [x=i,y=mean-lambda-0.2,col sep=comma] 
	{csv/lambda/2-rooms-1c/polynomial.csv};
	\addplot[name path=top,draw=none,forget plot]
	table [name path=top,x=i,y expr=\thisrow{mean-lambda-0.2}+\thisrow{std-lambda-0.2},col sep=comma] 
	{csv/lambda/2-rooms-1c/polynomial.csv};
	\addplot[name path=bot,draw=none,forget plot]
	table [name path=top,x=i,y expr=\thisrow{mean-lambda-0.2}-\thisrow{std-lambda-0.2},col sep=comma] 
	{csv/lambda/2-rooms-1c/polynomial.csv};
	\addplot [forget plot, draw=none,opacity=0.4,pattern=north east lines,fill=orange]
	fill between[of=top and bot];
	\addplot
	table [x=i,y=mean-lambda-0.3,col sep=comma] 
	{csv/lambda/2-rooms-1c/polynomial.csv};
	\addplot[name path=top,draw=none,forget plot]
	table [name path=top,x=i,y expr=\thisrow{mean-lambda-0.3}+\thisrow{std-lambda-0.3},col sep=comma] 
	{csv/lambda/2-rooms-1c/polynomial.csv};
	\addplot[name path=bot,draw=none,forget plot]
	table [name path=top,x=i,y expr=\thisrow{mean-lambda-0.3}-\thisrow{std-lambda-0.3},col sep=comma] 
	{csv/lambda/2-rooms-1c/polynomial.csv};
	\addplot [forget plot, draw=none,opacity=0.15,pattern=north east lines,fill=blue1]
	fill between[of=top and bot];
	\addplot
	table [x=i,y=mean-lambda-0.4,col sep=comma] 
	{csv/lambda/2-rooms-1c/polynomial.csv};
	\addplot[name path=top,draw=none,forget plot]
	table [name path=top,x=i,y expr=\thisrow{mean-lambda-0.4}+\thisrow{std-lambda-0.4},col sep=comma] 
	{csv/lambda/2-rooms-1c/polynomial.csv};
	\addplot[name path=bot,draw=none,forget plot]
	table [name path=top,x=i,y expr=\thisrow{mean-lambda-0.4}-\thisrow{std-lambda-0.4},col sep=comma] 
	{csv/lambda/2-rooms-1c/polynomial.csv};
	\addplot [forget plot, draw=none,opacity=0.15,pattern=north east lines,fill=violet]
	fill between[of=top and bot];
	\addplot
	table [x=i,y=mean-lambda-0.5,col sep=comma] 
	{csv/lambda/2-rooms-1c/polynomial.csv};
	\addplot[name path=top,draw=none,forget plot]
	table [name path=top,x=i,y expr=\thisrow{mean-lambda-0.5}+\thisrow{std-lambda-0.5},col sep=comma] 
	{csv/lambda/2-rooms-1c/polynomial.csv};
	\addplot[name path=bot,draw=none,forget plot]
	table [name path=top,x=i,y expr=\thisrow{mean-lambda-0.5}-\thisrow{std-lambda-0.5},col sep=comma] 
	{csv/lambda/2-rooms-1c/polynomial.csv};
	\addplot [forget plot, draw=none,opacity=0.15,pattern=north east lines,fill=mikadoYellow]
	fill between[of=top and bot];
	\addplot
	table [x=i,y=mean-lambda-0.6,col sep=comma] 
	{csv/lambda/2-rooms-1c/polynomial.csv};
	\addplot[name path=top,draw=none,forget plot]
	table [name path=top,x=i,y expr=\thisrow{mean-lambda-0.6}+\thisrow{std-lambda-0.6},col sep=comma] 
	{csv/lambda/2-rooms-1c/polynomial.csv};
	\addplot[name path=bot,draw=none,forget plot]
	table [name path=top,x=i,y expr=\thisrow{mean-lambda-0.6}-\thisrow{std-lambda-0.6},col sep=comma] 
	{csv/lambda/2-rooms-1c/polynomial.csv};
	\addplot [forget plot, draw=none,opacity=0.15,pattern=north east lines,fill=lapisLazuli]
	fill between[of=top and bot];
	\addplot
	table [x=i,y=mean-lambda-0.7,col sep=comma] 
	{csv/lambda/2-rooms-1c/polynomial.csv};
	\addplot[name path=top,draw=none,forget plot]
	table [name path=top,x=i,y expr=\thisrow{mean-lambda-0.7}+\thisrow{std-lambda-0.7},col sep=comma] 
	{csv/lambda/2-rooms-1c/polynomial.csv};
	\addplot[name path=bot,draw=none,forget plot]
	table [name path=top,x=i,y expr=\thisrow{mean-lambda-0.7}-\thisrow{std-lambda-0.7},col sep=comma] 
	{csv/lambda/2-rooms-1c/polynomial.csv};
	\addplot [forget plot, draw=none,opacity=0.15,pattern=north east lines,fill=harvardCrimson]
	fill between[of=top and bot];
	\addplot
	table [x=i,y=mean-lambda-0.8,col sep=comma] 
	{csv/lambda/2-rooms-1c/polynomial.csv};
	\addplot[name path=top,draw=none,forget plot]
	table [name path=top,x=i,y expr=\thisrow{mean-lambda-0.8}+\thisrow{std-lambda-0.8},col sep=comma] 
	{csv/lambda/2-rooms-1c/polynomial.csv};
	\addplot[name path=bot,draw=none,forget plot]
	table [name path=top,x=i,y expr=\thisrow{mean-lambda-0.8}-\thisrow{std-lambda-0.8},col sep=comma] 
	{csv/lambda/2-rooms-1c/polynomial.csv};
	\addplot [forget plot, draw=none,opacity=0.15,pattern=north east lines,fill=harlequin]
	fill between[of=top and bot];
	\addplot
	table [x=i,y=mean-lambda-0.9,col sep=comma] 
	{csv/lambda/2-rooms-1c/polynomial.csv};
	\addplot[name path=top,draw=none,forget plot]
	table [name path=top,x=i,y expr=\thisrow{mean-lambda-0.9}+\thisrow{std-lambda-0.9},col sep=comma] 
	{csv/lambda/2-rooms-1c/polynomial.csv};
	\addplot[name path=bot,draw=none,forget plot]
	table [name path=top,x=i,y expr=\thisrow{mean-lambda-0.9}-\thisrow{std-lambda-0.9},col sep=comma] 
	{csv/lambda/2-rooms-1c/polynomial.csv};
	\addplot [forget plot, draw=none,opacity=0.15,pattern=north east lines,fill=neonFuchsia]
	fill between[of=top and bot];
	\addplot
	table [x=i,y=mean-lambda-1.0,col sep=comma] 
	{csv/lambda/2-rooms-1c/polynomial.csv};
	\addplot[name path=top,draw=none,forget plot]
	table [name path=top,x=i,y expr=\thisrow{mean-lambda-1.0}+\thisrow{std-lambda-1.0},col sep=comma] 
	{csv/lambda/2-rooms-1c/polynomial.csv};
	\addplot[name path=bot,draw=none,forget plot]
	table [name path=top,x=i,y expr=\thisrow{mean-lambda-1.0}-\thisrow{std-lambda-1.0},col sep=comma] 
	{csv/lambda/2-rooms-1c/polynomial.csv};
	\addplot [forget plot, draw=none,opacity=0.15,pattern=north east lines,fill=bronze]
	fill between[of=top and bot];
	\addplot
	table [x=i,y=mean-likelihood,col sep=comma] 
	{csv/lambda/2-rooms-1c/polynomial.csv};
	\addplot[name path=top,draw=none,forget plot]
	table [name path=top,x=i,y expr=\thisrow{mean-likelihood}+\thisrow{std-likelihood},col sep=comma] 
	{csv/lambda/2-rooms-1c/polynomial.csv};
	\addplot[name path=bot,draw=none,forget plot]
	table [name path=top,x=i,y expr=\thisrow{mean-likelihood}-\thisrow{std-likelihood},col sep=comma] 
	{csv/lambda/2-rooms-1c/polynomial.csv};
	\addplot [forget plot, draw=none,opacity=0.15,pattern=north east lines,fill=black]
	fill between[of=top and bot];

	\end{axis}
	\end{tikzpicture}
	\label{fig:2-rooms-lambda-polynomial-1c}
}
\quad
\subfigure[2-rooms linear dynamic.]
{
	\begin{tikzpicture}
	\begin{axis}[
	width=0.45\textwidth,
	height=4.5cm,
	xmin=50,
	xmax=2950,
	xtick={500,1000,...,3000},
	ymin=0,
	ymax=0.9,
	ytick={0,0.2,...,0.8},
	%
	%
	xlabel=Iterations,
	ylabel=Average Return,
	mark options={scale=0.2},
	cycle list name = custom,
	scaled x ticks=base 10:-3
	]
	\addplot
	table [x=i,y=mean-lambda-0.1,col sep=comma] 
	{csv/lambda/2-rooms-1c/linear.csv};
	\addplot[name path=top,draw=none,forget plot]
	table [name path=top,x=i,y expr=\thisrow{mean-lambda-0.1}+\thisrow{std-lambda-0.1},col sep=comma] 
	{csv/lambda/2-rooms-1c/linear.csv};
	\addplot[name path=bot,draw=none,forget plot]
	table [name path=top,x=i,y expr=\thisrow{mean-lambda-0.1}-\thisrow{std-lambda-0.1},col sep=comma] 
	{csv/lambda/2-rooms-1c/linear.csv};
	\addplot [forget plot, draw=none,opacity=0.4,pattern=north east lines,fill=green1!60!black]
	fill between[of=top and bot];
	\addplot
	table [x=i,y=mean-lambda-0.2,col sep=comma] 
	{csv/lambda/2-rooms-1c/linear.csv};
	\addplot[name path=top,draw=none,forget plot]
	table [name path=top,x=i,y expr=\thisrow{mean-lambda-0.2}+\thisrow{std-lambda-0.2},col sep=comma] 
	{csv/lambda/2-rooms-1c/linear.csv};
	\addplot[name path=bot,draw=none,forget plot]
	table [name path=top,x=i,y expr=\thisrow{mean-lambda-0.2}-\thisrow{std-lambda-0.2},col sep=comma] 
	{csv/lambda/2-rooms-1c/linear.csv};
	\addplot [forget plot, draw=none,opacity=0.4,pattern=north east lines,fill=orange]
	fill between[of=top and bot];
	\addplot
	table [x=i,y=mean-lambda-0.3,col sep=comma] 
	{csv/lambda/2-rooms-1c/linear.csv};
	\addplot[name path=top,draw=none,forget plot]
	table [name path=top,x=i,y expr=\thisrow{mean-lambda-0.3}+\thisrow{std-lambda-0.3},col sep=comma] 
	{csv/lambda/2-rooms-1c/linear.csv};
	\addplot[name path=bot,draw=none,forget plot]
	table [name path=top,x=i,y expr=\thisrow{mean-lambda-0.3}-\thisrow{std-lambda-0.3},col sep=comma] 
	{csv/lambda/2-rooms-1c/linear.csv};
	\addplot [forget plot, draw=none,opacity=0.15,pattern=north east lines,fill=blue1]
	fill between[of=top and bot];
	\addplot
	table [x=i,y=mean-lambda-0.4,col sep=comma] 
	{csv/lambda/2-rooms-1c/linear.csv};
	\addplot[name path=top,draw=none,forget plot]
	table [name path=top,x=i,y expr=\thisrow{mean-lambda-0.4}+\thisrow{std-lambda-0.4},col sep=comma] 
	{csv/lambda/2-rooms-1c/linear.csv};
	\addplot[name path=bot,draw=none,forget plot]
	table [name path=top,x=i,y expr=\thisrow{mean-lambda-0.4}-\thisrow{std-lambda-0.4},col sep=comma] 
	{csv/lambda/2-rooms-1c/linear.csv};
	\addplot [forget plot, draw=none,opacity=0.15,pattern=north east lines,fill=violet]
	fill between[of=top and bot];
	\addplot
	table [x=i,y=mean-lambda-0.5,col sep=comma] 
	{csv/lambda/2-rooms-1c/linear.csv};
	\addplot[name path=top,draw=none,forget plot]
	table [name path=top,x=i,y expr=\thisrow{mean-lambda-0.5}+\thisrow{std-lambda-0.5},col sep=comma] 
	{csv/lambda/2-rooms-1c/linear.csv};
	\addplot[name path=bot,draw=none,forget plot]
	table [name path=top,x=i,y expr=\thisrow{mean-lambda-0.5}-\thisrow{std-lambda-0.5},col sep=comma] 
	{csv/lambda/2-rooms-1c/linear.csv};
	\addplot [forget plot, draw=none,opacity=0.15,pattern=north east lines,fill=mikadoYellow]
	fill between[of=top and bot];
	\addplot
	table [x=i,y=mean-lambda-0.6,col sep=comma] 
	{csv/lambda/2-rooms-1c/linear.csv};
	\addplot[name path=top,draw=none,forget plot]
	table [name path=top,x=i,y expr=\thisrow{mean-lambda-0.6}+\thisrow{std-lambda-0.6},col sep=comma] 
	{csv/lambda/2-rooms-1c/linear.csv};
	\addplot[name path=bot,draw=none,forget plot]
	table [name path=top,x=i,y expr=\thisrow{mean-lambda-0.6}-\thisrow{std-lambda-0.6},col sep=comma] 
	{csv/lambda/2-rooms-1c/linear.csv};
	\addplot [forget plot, draw=none,opacity=0.15,pattern=north east lines,fill=lapisLazuli]
	fill between[of=top and bot];
	\addplot
	table [x=i,y=mean-lambda-0.7,col sep=comma] 
	{csv/lambda/2-rooms-1c/linear.csv};
	\addplot[name path=top,draw=none,forget plot]
	table [name path=top,x=i,y expr=\thisrow{mean-lambda-0.7}+\thisrow{std-lambda-0.7},col sep=comma] 
	{csv/lambda/2-rooms-1c/linear.csv};
	\addplot[name path=bot,draw=none,forget plot]
	table [name path=top,x=i,y expr=\thisrow{mean-lambda-0.7}-\thisrow{std-lambda-0.7},col sep=comma] 
	{csv/lambda/2-rooms-1c/linear.csv};
	\addplot [forget plot, draw=none,opacity=0.15,pattern=north east lines,fill=harvardCrimson]
	fill between[of=top and bot];
	\addplot
	table [x=i,y=mean-lambda-0.8,col sep=comma] 
	{csv/lambda/2-rooms-1c/linear.csv};
	\addplot[name path=top,draw=none,forget plot]
	table [name path=top,x=i,y expr=\thisrow{mean-lambda-0.8}+\thisrow{std-lambda-0.8},col sep=comma] 
	{csv/lambda/2-rooms-1c/linear.csv};
	\addplot[name path=bot,draw=none,forget plot]
	table [name path=top,x=i,y expr=\thisrow{mean-lambda-0.8}-\thisrow{std-lambda-0.8},col sep=comma] 
	{csv/lambda/2-rooms-1c/linear.csv};
	\addplot [forget plot, draw=none,opacity=0.15,pattern=north east lines,fill=harlequin]
	fill between[of=top and bot];
	\addplot
	table [x=i,y=mean-lambda-0.9,col sep=comma] 
	{csv/lambda/2-rooms-1c/linear.csv};
	\addplot[name path=top,draw=none,forget plot]
	table [name path=top,x=i,y expr=\thisrow{mean-lambda-0.9}+\thisrow{std-lambda-0.9},col sep=comma] 
	{csv/lambda/2-rooms-1c/linear.csv};
	\addplot[name path=bot,draw=none,forget plot]
	table [name path=top,x=i,y expr=\thisrow{mean-lambda-0.9}-\thisrow{std-lambda-0.9},col sep=comma] 
	{csv/lambda/2-rooms-1c/linear.csv};
	\addplot [forget plot, draw=none,opacity=0.15,pattern=north east lines,fill=neonFuchsia]
	fill between[of=top and bot];
	\addplot
	table [x=i,y=mean-lambda-1.0,col sep=comma] 
	{csv/lambda/2-rooms-1c/linear.csv};
	\addplot[name path=top,draw=none,forget plot]
	table [name path=top,x=i,y expr=\thisrow{mean-lambda-1.0}+\thisrow{std-lambda-1.0},col sep=comma] 
	{csv/lambda/2-rooms-1c/linear.csv};
	\addplot[name path=bot,draw=none,forget plot]
	table [name path=top,x=i,y expr=\thisrow{mean-lambda-1.0}-\thisrow{std-lambda-1.0},col sep=comma] 
	{csv/lambda/2-rooms-1c/linear.csv};
	\addplot [forget plot, draw=none,opacity=0.15,pattern=north east lines,fill=bronze]
	fill between[of=top and bot];
	\addplot
	table [x=i,y=mean-likelihood,col sep=comma] 
	{csv/lambda/2-rooms-1c/linear.csv};
	\addplot[name path=top,draw=none,forget plot]
	table [name path=top,x=i,y expr=\thisrow{mean-likelihood}+\thisrow{std-likelihood},col sep=comma] 
	{csv/lambda/2-rooms-1c/linear.csv};
	\addplot[name path=bot,draw=none,forget plot]
	table [name path=top,x=i,y expr=\thisrow{mean-likelihood}-\thisrow{std-likelihood},col sep=comma] 
	{csv/lambda/2-rooms-1c/linear.csv};
	\addplot [forget plot, draw=none,opacity=0.15,pattern=north east lines,fill=black]
	fill between[of=top and bot];

	\end{axis}
	\end{tikzpicture}
	\label{fig:2-rooms-lambda-linear-1c}
}

\subfigure[2-rooms $\sin$ dynamic.]
{
	\begin{tikzpicture}
	\begin{axis}[
	width=0.45\textwidth,
	height=4.5cm,
	xmin=50,
	xmax=2950,
	xtick={500,1000,...,3000},
	ymin=0,
	ymax=0.9,
	ytick={0,0.2,...,0.8},
	%
	%
	xlabel=Iterations,
	ylabel=Average Return,
	mark options={scale=0.2},
	cycle list name = custom,
	scaled x ticks=base 10:-3
	]
	\addplot
	table [x=i,y=mean-lambda-0.1,col sep=comma] 
	{csv/lambda/2-rooms-1c/sin.csv};
	\addplot[name path=top,draw=none,forget plot]
	table [name path=top,x=i,y expr=\thisrow{mean-lambda-0.1}+\thisrow{std-lambda-0.1},col sep=comma] 
	{csv/lambda/2-rooms-1c/sin.csv};
	\addplot[name path=bot,draw=none,forget plot]
	table [name path=top,x=i,y expr=\thisrow{mean-lambda-0.1}-\thisrow{std-lambda-0.1},col sep=comma] 
	{csv/lambda/2-rooms-1c/sin.csv};
	\addplot [forget plot, draw=none,opacity=0.4,pattern=north east lines,fill=green1!60!black]
	fill between[of=top and bot];
	\addplot
	table [x=i,y=mean-lambda-0.2,col sep=comma] 
	{csv/lambda/2-rooms-1c/sin.csv};
	\addplot[name path=top,draw=none,forget plot]
	table [name path=top,x=i,y expr=\thisrow{mean-lambda-0.2}+\thisrow{std-lambda-0.2},col sep=comma] 
	{csv/lambda/2-rooms-1c/sin.csv};
	\addplot[name path=bot,draw=none,forget plot]
	table [name path=top,x=i,y expr=\thisrow{mean-lambda-0.2}-\thisrow{std-lambda-0.2},col sep=comma] 
	{csv/lambda/2-rooms-1c/sin.csv};
	\addplot [forget plot, draw=none,opacity=0.4,pattern=north east lines,fill=orange]
	fill between[of=top and bot];
	\addplot
	table [x=i,y=mean-lambda-0.3,col sep=comma] 
	{csv/lambda/2-rooms-1c/sin.csv};
	\addplot[name path=top,draw=none,forget plot]
	table [name path=top,x=i,y expr=\thisrow{mean-lambda-0.3}+\thisrow{std-lambda-0.3},col sep=comma] 
	{csv/lambda/2-rooms-1c/sin.csv};
	\addplot[name path=bot,draw=none,forget plot]
	table [name path=top,x=i,y expr=\thisrow{mean-lambda-0.3}-\thisrow{std-lambda-0.3},col sep=comma] 
	{csv/lambda/2-rooms-1c/sin.csv};
	\addplot [forget plot, draw=none,opacity=0.15,pattern=north east lines,fill=blue1]
	fill between[of=top and bot];
	\addplot
	table [x=i,y=mean-lambda-0.4,col sep=comma] 
	{csv/lambda/2-rooms-1c/sin.csv};
	\addplot[name path=top,draw=none,forget plot]
	table [name path=top,x=i,y expr=\thisrow{mean-lambda-0.4}+\thisrow{std-lambda-0.4},col sep=comma] 
	{csv/lambda/2-rooms-1c/sin.csv};
	\addplot[name path=bot,draw=none,forget plot]
	table [name path=top,x=i,y expr=\thisrow{mean-lambda-0.4}-\thisrow{std-lambda-0.4},col sep=comma] 
	{csv/lambda/2-rooms-1c/sin.csv};
	\addplot [forget plot, draw=none,opacity=0.15,pattern=north east lines,fill=violet]
	fill between[of=top and bot];
	\addplot
	table [x=i,y=mean-lambda-0.5,col sep=comma] 
	{csv/lambda/2-rooms-1c/sin.csv};
	\addplot[name path=top,draw=none,forget plot]
	table [name path=top,x=i,y expr=\thisrow{mean-lambda-0.5}+\thisrow{std-lambda-0.5},col sep=comma] 
	{csv/lambda/2-rooms-1c/sin.csv};
	\addplot[name path=bot,draw=none,forget plot]
	table [name path=top,x=i,y expr=\thisrow{mean-lambda-0.5}-\thisrow{std-lambda-0.5},col sep=comma] 
	{csv/lambda/2-rooms-1c/sin.csv};
	\addplot [forget plot, draw=none,opacity=0.15,pattern=north east lines,fill=mikadoYellow]
	fill between[of=top and bot];
	\addplot
	table [x=i,y=mean-lambda-0.6,col sep=comma] 
	{csv/lambda/2-rooms-1c/sin.csv};
	\addplot[name path=top,draw=none,forget plot]
	table [name path=top,x=i,y expr=\thisrow{mean-lambda-0.6}+\thisrow{std-lambda-0.6},col sep=comma] 
	{csv/lambda/2-rooms-1c/sin.csv};
	\addplot[name path=bot,draw=none,forget plot]
	table [name path=top,x=i,y expr=\thisrow{mean-lambda-0.6}-\thisrow{std-lambda-0.6},col sep=comma] 
	{csv/lambda/2-rooms-1c/sin.csv};
	\addplot [forget plot, draw=none,opacity=0.15,pattern=north east lines,fill=lapisLazuli]
	fill between[of=top and bot];
	\addplot
	table [x=i,y=mean-lambda-0.7,col sep=comma] 
	{csv/lambda/2-rooms-1c/sin.csv};
	\addplot[name path=top,draw=none,forget plot]
	table [name path=top,x=i,y expr=\thisrow{mean-lambda-0.7}+\thisrow{std-lambda-0.7},col sep=comma] 
	{csv/lambda/2-rooms-1c/sin.csv};
	\addplot[name path=bot,draw=none,forget plot]
	table [name path=top,x=i,y expr=\thisrow{mean-lambda-0.7}-\thisrow{std-lambda-0.7},col sep=comma] 
	{csv/lambda/2-rooms-1c/sin.csv};
	\addplot [forget plot, draw=none,opacity=0.15,pattern=north east lines,fill=harvardCrimson]
	fill between[of=top and bot];
	\addplot
	table [x=i,y=mean-lambda-0.8,col sep=comma] 
	{csv/lambda/2-rooms-1c/sin.csv};
	\addplot[name path=top,draw=none,forget plot]
	table [name path=top,x=i,y expr=\thisrow{mean-lambda-0.8}+\thisrow{std-lambda-0.8},col sep=comma] 
	{csv/lambda/2-rooms-1c/sin.csv};
	\addplot[name path=bot,draw=none,forget plot]
	table [name path=top,x=i,y expr=\thisrow{mean-lambda-0.8}-\thisrow{std-lambda-0.8},col sep=comma] 
	{csv/lambda/2-rooms-1c/sin.csv};
	\addplot [forget plot, draw=none,opacity=0.15,pattern=north east lines,fill=harlequin]
	fill between[of=top and bot];
	\addplot
	table [x=i,y=mean-lambda-0.9,col sep=comma] 
	{csv/lambda/2-rooms-1c/sin.csv};
	\addplot[name path=top,draw=none,forget plot]
	table [name path=top,x=i,y expr=\thisrow{mean-lambda-0.9}+\thisrow{std-lambda-0.9},col sep=comma] 
	{csv/lambda/2-rooms-1c/sin.csv};
	\addplot[name path=bot,draw=none,forget plot]
	table [name path=top,x=i,y expr=\thisrow{mean-lambda-0.9}-\thisrow{std-lambda-0.9},col sep=comma] 
	{csv/lambda/2-rooms-1c/sin.csv};
	\addplot [forget plot, draw=none,opacity=0.15,pattern=north east lines,fill=neonFuchsia]
	fill between[of=top and bot];
	\addplot
	table [x=i,y=mean-lambda-1.0,col sep=comma] 
	{csv/lambda/2-rooms-1c/sin.csv};
	\addplot[name path=top,draw=none,forget plot]
	table [name path=top,x=i,y expr=\thisrow{mean-lambda-1.0}+\thisrow{std-lambda-1.0},col sep=comma] 
	{csv/lambda/2-rooms-1c/sin.csv};
	\addplot[name path=bot,draw=none,forget plot]
	table [name path=top,x=i,y expr=\thisrow{mean-lambda-1.0}-\thisrow{std-lambda-1.0},col sep=comma] 
	{csv/lambda/2-rooms-1c/sin.csv};
	\addplot [forget plot, draw=none,opacity=0.15,pattern=north east lines,fill=bronze]
	fill between[of=top and bot];
	\addplot
	table [x=i,y=mean-likelihood,col sep=comma] 
	{csv/lambda/2-rooms-1c/sin.csv};
	\addplot[name path=top,draw=none,forget plot]
	table [name path=top,x=i,y expr=\thisrow{mean-likelihood}+\thisrow{std-likelihood},col sep=comma] 
	{csv/lambda/2-rooms-1c/sin.csv};
	\addplot[name path=bot,draw=none,forget plot]
	table [name path=top,x=i,y expr=\thisrow{mean-likelihood}-\thisrow{std-likelihood},col sep=comma] 
	{csv/lambda/2-rooms-1c/sin.csv};
	\addplot [forget plot, draw=none,opacity=0.15,pattern=north east lines,fill=black]
	fill between[of=top and bot];

	\end{axis}
	\end{tikzpicture}
	\label{fig:2-rooms-lambda-sin-1c}
}

	\caption{Average return achieved by $1$-T2VT \wrt different choices of $\lambda$ with $95\%$ confidence intervals computed using $50$ independent runs.}
	\label{fig:2-rooms-lambda-1c}
\end{figure*}

%% file: figure_lambda-sensitivity-2-rooms-3c.tex
\begin{figure*}[!b]
	\centering
	%
	%
	\begin{tikzpicture}
	\begin{customlegend}[legend columns=6,legend style={align=left,draw=none,column sep=2ex,font=\footnotesize},legend entries={ $\lambda=0.1$, $\lambda=0.2$, $\lambda=0.3$, $\lambda=0.4$, $\lambda=0.5$, $\lambda=0.6$, $\lambda=0.7$, $\lambda=0.8$, $\lambda=0.9$, $\lambda=1$, likelihood}]
	%
	%
	%
	\addlegendimage{green1!60!black,ultra thick,dashed}   
	\addlegendimage{orange!80!white,dotted,ultra thick}
	\addlegendimage{blue1,ultra thick}
	\addlegendimage{violet!90!white,dash dot,ultra thick}
	\addlegendimage{mikadoYellow, densely dashed,ultra thick}	
	\addlegendimage{lapisLazuli,densely dotted,ultra thick}
	\addlegendimage{harvardCrimson,densely dash dot, ultra thick}
	\addlegendimage{harlequin,ultra thick}
	\addlegendimage{neonFuchsia, densely dotted,ultra thick}
	\addlegendimage{bronze, densely dash dot,ultra thick}
	\addlegendimage{black,ultra thick}

	\end{customlegend}
	\end{tikzpicture}
	
\subfigure[2-rooms polynomial dynamic.]
{
	\begin{tikzpicture}
	\begin{axis}[
	width=0.45\textwidth,
	height=4.5cm,
	xmin=50,
	xmax=2950,
	xtick={500,1000,...,3000},
	ymin=0,
	ymax=0.9,
	ytick={0,0.2,...,0.8},
	%
	%
	xlabel=Iterations,
	ylabel=Average Return,
	mark options={scale=0.2},
	cycle list name = custom,
	scaled x ticks=base 10:-3
	]
	\addplot
	table [x=i,y=mean-lambda-0.1,col sep=comma] 
	{csv/lambda/2-rooms-3c/polynomial.csv};
	\addplot[name path=top,draw=none,forget plot]
	table [name path=top,x=i,y expr=\thisrow{mean-lambda-0.1}+\thisrow{std-lambda-0.1},col sep=comma] 
	{csv/lambda/2-rooms-3c/polynomial.csv};
	\addplot[name path=bot,draw=none,forget plot]
	table [name path=top,x=i,y expr=\thisrow{mean-lambda-0.1}-\thisrow{std-lambda-0.1},col sep=comma] 
	{csv/lambda/2-rooms-3c/polynomial.csv};
	\addplot [forget plot, draw=none,opacity=0.4,pattern=north east lines,fill=green1!60!black]
	fill between[of=top and bot];
	\addplot
	table [x=i,y=mean-lambda-0.2,col sep=comma] 
	{csv/lambda/2-rooms-3c/polynomial.csv};
	\addplot[name path=top,draw=none,forget plot]
	table [name path=top,x=i,y expr=\thisrow{mean-lambda-0.2}+\thisrow{std-lambda-0.2},col sep=comma] 
	{csv/lambda/2-rooms-3c/polynomial.csv};
	\addplot[name path=bot,draw=none,forget plot]
	table [name path=top,x=i,y expr=\thisrow{mean-lambda-0.2}-\thisrow{std-lambda-0.2},col sep=comma] 
	{csv/lambda/2-rooms-3c/polynomial.csv};
	\addplot [forget plot, draw=none,opacity=0.4,pattern=north east lines,fill=orange]
	fill between[of=top and bot];
	\addplot
	table [x=i,y=mean-lambda-0.3,col sep=comma] 
	{csv/lambda/2-rooms-3c/polynomial.csv};
	\addplot[name path=top,draw=none,forget plot]
	table [name path=top,x=i,y expr=\thisrow{mean-lambda-0.3}+\thisrow{std-lambda-0.3},col sep=comma] 
	{csv/lambda/2-rooms-3c/polynomial.csv};
	\addplot[name path=bot,draw=none,forget plot]
	table [name path=top,x=i,y expr=\thisrow{mean-lambda-0.3}-\thisrow{std-lambda-0.3},col sep=comma] 
	{csv/lambda/2-rooms-3c/polynomial.csv};
	\addplot [forget plot, draw=none,opacity=0.15,pattern=north east lines,fill=blue1]
	fill between[of=top and bot];
	\addplot
	table [x=i,y=mean-lambda-0.4,col sep=comma] 
	{csv/lambda/2-rooms-3c/polynomial.csv};
	\addplot[name path=top,draw=none,forget plot]
	table [name path=top,x=i,y expr=\thisrow{mean-lambda-0.4}+\thisrow{std-lambda-0.4},col sep=comma] 
	{csv/lambda/2-rooms-3c/polynomial.csv};
	\addplot[name path=bot,draw=none,forget plot]
	table [name path=top,x=i,y expr=\thisrow{mean-lambda-0.4}-\thisrow{std-lambda-0.4},col sep=comma] 
	{csv/lambda/2-rooms-3c/polynomial.csv};
	\addplot [forget plot, draw=none,opacity=0.15,pattern=north east lines,fill=violet]
	fill between[of=top and bot];
	\addplot
	table [x=i,y=mean-lambda-0.5,col sep=comma] 
	{csv/lambda/2-rooms-3c/polynomial.csv};
	\addplot[name path=top,draw=none,forget plot]
	table [name path=top,x=i,y expr=\thisrow{mean-lambda-0.5}+\thisrow{std-lambda-0.5},col sep=comma] 
	{csv/lambda/2-rooms-3c/polynomial.csv};
	\addplot[name path=bot,draw=none,forget plot]
	table [name path=top,x=i,y expr=\thisrow{mean-lambda-0.5}-\thisrow{std-lambda-0.5},col sep=comma] 
	{csv/lambda/2-rooms-3c/polynomial.csv};
	\addplot [forget plot, draw=none,opacity=0.15,pattern=north east lines,fill=mikadoYellow]
	fill between[of=top and bot];
	\addplot
	table [x=i,y=mean-lambda-0.6,col sep=comma] 
	{csv/lambda/2-rooms-3c/polynomial.csv};
	\addplot[name path=top,draw=none,forget plot]
	table [name path=top,x=i,y expr=\thisrow{mean-lambda-0.6}+\thisrow{std-lambda-0.6},col sep=comma] 
	{csv/lambda/2-rooms-3c/polynomial.csv};
	\addplot[name path=bot,draw=none,forget plot]
	table [name path=top,x=i,y expr=\thisrow{mean-lambda-0.6}-\thisrow{std-lambda-0.6},col sep=comma] 
	{csv/lambda/2-rooms-3c/polynomial.csv};
	\addplot [forget plot, draw=none,opacity=0.15,pattern=north east lines,fill=lapisLazuli]
	fill between[of=top and bot];
	\addplot
	table [x=i,y=mean-lambda-0.7,col sep=comma] 
	{csv/lambda/2-rooms-3c/polynomial.csv};
	\addplot[name path=top,draw=none,forget plot]
	table [name path=top,x=i,y expr=\thisrow{mean-lambda-0.7}+\thisrow{std-lambda-0.7},col sep=comma] 
	{csv/lambda/2-rooms-3c/polynomial.csv};
	\addplot[name path=bot,draw=none,forget plot]
	table [name path=top,x=i,y expr=\thisrow{mean-lambda-0.7}-\thisrow{std-lambda-0.7},col sep=comma] 
	{csv/lambda/2-rooms-3c/polynomial.csv};
	\addplot [forget plot, draw=none,opacity=0.15,pattern=north east lines,fill=harvardCrimson]
	fill between[of=top and bot];
	\addplot
	table [x=i,y=mean-lambda-0.8,col sep=comma] 
	{csv/lambda/2-rooms-3c/polynomial.csv};
	\addplot[name path=top,draw=none,forget plot]
	table [name path=top,x=i,y expr=\thisrow{mean-lambda-0.8}+\thisrow{std-lambda-0.8},col sep=comma] 
	{csv/lambda/2-rooms-3c/polynomial.csv};
	\addplot[name path=bot,draw=none,forget plot]
	table [name path=top,x=i,y expr=\thisrow{mean-lambda-0.8}-\thisrow{std-lambda-0.8},col sep=comma] 
	{csv/lambda/2-rooms-3c/polynomial.csv};
	\addplot [forget plot, draw=none,opacity=0.15,pattern=north east lines,fill=harlequin]
	fill between[of=top and bot];
	\addplot
	table [x=i,y=mean-lambda-0.9,col sep=comma] 
	{csv/lambda/2-rooms-3c/polynomial.csv};
	\addplot[name path=top,draw=none,forget plot]
	table [name path=top,x=i,y expr=\thisrow{mean-lambda-0.9}+\thisrow{std-lambda-0.9},col sep=comma] 
	{csv/lambda/2-rooms-3c/polynomial.csv};
	\addplot[name path=bot,draw=none,forget plot]
	table [name path=top,x=i,y expr=\thisrow{mean-lambda-0.9}-\thisrow{std-lambda-0.9},col sep=comma] 
	{csv/lambda/2-rooms-3c/polynomial.csv};
	\addplot [forget plot, draw=none,opacity=0.15,pattern=north east lines,fill=neonFuchsia]
	fill between[of=top and bot];
	\addplot
	table [x=i,y=mean-lambda-1.0,col sep=comma] 
	{csv/lambda/2-rooms-3c/polynomial.csv};
	\addplot[name path=top,draw=none,forget plot]
	table [name path=top,x=i,y expr=\thisrow{mean-lambda-1.0}+\thisrow{std-lambda-1.0},col sep=comma] 
	{csv/lambda/2-rooms-3c/polynomial.csv};
	\addplot[name path=bot,draw=none,forget plot]
	table [name path=top,x=i,y expr=\thisrow{mean-lambda-1.0}-\thisrow{std-lambda-1.0},col sep=comma] 
	{csv/lambda/2-rooms-3c/polynomial.csv};
	\addplot [forget plot, draw=none,opacity=0.15,pattern=north east lines,fill=bronze]
	fill between[of=top and bot];
	\addplot
	table [x=i,y=mean-likelihood,col sep=comma] 
	{csv/lambda/2-rooms-3c/polynomial.csv};
	\addplot[name path=top,draw=none,forget plot]
	table [name path=top,x=i,y expr=\thisrow{mean-likelihood}+\thisrow{std-likelihood},col sep=comma] 
	{csv/lambda/2-rooms-3c/polynomial.csv};
	\addplot[name path=bot,draw=none,forget plot]
	table [name path=top,x=i,y expr=\thisrow{mean-likelihood}-\thisrow{std-likelihood},col sep=comma] 
	{csv/lambda/2-rooms-3c/polynomial.csv};
	\addplot [forget plot, draw=none,opacity=0.15,pattern=north east lines,fill=black]
	fill between[of=top and bot];

	\end{axis}
	\end{tikzpicture}
	\label{fig:2-rooms-lambda-polynomial-3c}
}
\quad
\subfigure[2-rooms linear dynamic.]
{
	\begin{tikzpicture}
	\begin{axis}[
	width=0.45\textwidth,
	height=4.5cm,
	xmin=50,
	xmax=2950,
	xtick={500,1000,...,3000},
	ymin=0,
	ymax=0.9,
	ytick={0,0.2,...,0.8},
	%
	%
	xlabel=Iterations,
	ylabel=Average Return,
	mark options={scale=0.2},
	cycle list name = custom,
	scaled x ticks=base 10:-3
	]
	\addplot
	table [x=i,y=mean-lambda-0.1,col sep=comma] 
	{csv/lambda/2-rooms-3c/linear.csv};
	\addplot[name path=top,draw=none,forget plot]
	table [name path=top,x=i,y expr=\thisrow{mean-lambda-0.1}+\thisrow{std-lambda-0.1},col sep=comma] 
	{csv/lambda/2-rooms-3c/linear.csv};
	\addplot[name path=bot,draw=none,forget plot]
	table [name path=top,x=i,y expr=\thisrow{mean-lambda-0.1}-\thisrow{std-lambda-0.1},col sep=comma] 
	{csv/lambda/2-rooms-3c/linear.csv};
	\addplot [forget plot, draw=none,opacity=0.4,pattern=north east lines,fill=green1!60!black]
	fill between[of=top and bot];
	\addplot
	table [x=i,y=mean-lambda-0.2,col sep=comma] 
	{csv/lambda/2-rooms-3c/linear.csv};
	\addplot[name path=top,draw=none,forget plot]
	table [name path=top,x=i,y expr=\thisrow{mean-lambda-0.2}+\thisrow{std-lambda-0.2},col sep=comma] 
	{csv/lambda/2-rooms-3c/linear.csv};
	\addplot[name path=bot,draw=none,forget plot]
	table [name path=top,x=i,y expr=\thisrow{mean-lambda-0.2}-\thisrow{std-lambda-0.2},col sep=comma] 
	{csv/lambda/2-rooms-3c/linear.csv};
	\addplot [forget plot, draw=none,opacity=0.4,pattern=north east lines,fill=orange]
	fill between[of=top and bot];
	\addplot
	table [x=i,y=mean-lambda-0.3,col sep=comma] 
	{csv/lambda/2-rooms-3c/linear.csv};
	\addplot[name path=top,draw=none,forget plot]
	table [name path=top,x=i,y expr=\thisrow{mean-lambda-0.3}+\thisrow{std-lambda-0.3},col sep=comma] 
	{csv/lambda/2-rooms-3c/linear.csv};
	\addplot[name path=bot,draw=none,forget plot]
	table [name path=top,x=i,y expr=\thisrow{mean-lambda-0.3}-\thisrow{std-lambda-0.3},col sep=comma] 
	{csv/lambda/2-rooms-3c/linear.csv};
	\addplot [forget plot, draw=none,opacity=0.15,pattern=north east lines,fill=blue1]
	fill between[of=top and bot];
	\addplot
	table [x=i,y=mean-lambda-0.4,col sep=comma] 
	{csv/lambda/2-rooms-3c/linear.csv};
	\addplot[name path=top,draw=none,forget plot]
	table [name path=top,x=i,y expr=\thisrow{mean-lambda-0.4}+\thisrow{std-lambda-0.4},col sep=comma] 
	{csv/lambda/2-rooms-3c/linear.csv};
	\addplot[name path=bot,draw=none,forget plot]
	table [name path=top,x=i,y expr=\thisrow{mean-lambda-0.4}-\thisrow{std-lambda-0.4},col sep=comma] 
	{csv/lambda/2-rooms-3c/linear.csv};
	\addplot [forget plot, draw=none,opacity=0.15,pattern=north east lines,fill=violet]
	fill between[of=top and bot];
	\addplot
	table [x=i,y=mean-lambda-0.5,col sep=comma] 
	{csv/lambda/2-rooms-3c/linear.csv};
	\addplot[name path=top,draw=none,forget plot]
	table [name path=top,x=i,y expr=\thisrow{mean-lambda-0.5}+\thisrow{std-lambda-0.5},col sep=comma] 
	{csv/lambda/2-rooms-3c/linear.csv};
	\addplot[name path=bot,draw=none,forget plot]
	table [name path=top,x=i,y expr=\thisrow{mean-lambda-0.5}-\thisrow{std-lambda-0.5},col sep=comma] 
	{csv/lambda/2-rooms-3c/linear.csv};
	\addplot [forget plot, draw=none,opacity=0.15,pattern=north east lines,fill=mikadoYellow]
	fill between[of=top and bot];
	\addplot
	table [x=i,y=mean-lambda-0.6,col sep=comma] 
	{csv/lambda/2-rooms-3c/linear.csv};
	\addplot[name path=top,draw=none,forget plot]
	table [name path=top,x=i,y expr=\thisrow{mean-lambda-0.6}+\thisrow{std-lambda-0.6},col sep=comma] 
	{csv/lambda/2-rooms-3c/linear.csv};
	\addplot[name path=bot,draw=none,forget plot]
	table [name path=top,x=i,y expr=\thisrow{mean-lambda-0.6}-\thisrow{std-lambda-0.6},col sep=comma] 
	{csv/lambda/2-rooms-3c/linear.csv};
	\addplot [forget plot, draw=none,opacity=0.15,pattern=north east lines,fill=lapisLazuli]
	fill between[of=top and bot];
	\addplot
	table [x=i,y=mean-lambda-0.7,col sep=comma] 
	{csv/lambda/2-rooms-3c/linear.csv};
	\addplot[name path=top,draw=none,forget plot]
	table [name path=top,x=i,y expr=\thisrow{mean-lambda-0.7}+\thisrow{std-lambda-0.7},col sep=comma] 
	{csv/lambda/2-rooms-3c/linear.csv};
	\addplot[name path=bot,draw=none,forget plot]
	table [name path=top,x=i,y expr=\thisrow{mean-lambda-0.7}-\thisrow{std-lambda-0.7},col sep=comma] 
	{csv/lambda/2-rooms-3c/linear.csv};
	\addplot [forget plot, draw=none,opacity=0.15,pattern=north east lines,fill=harvardCrimson]
	fill between[of=top and bot];
	\addplot
	table [x=i,y=mean-lambda-0.8,col sep=comma] 
	{csv/lambda/2-rooms-3c/linear.csv};
	\addplot[name path=top,draw=none,forget plot]
	table [name path=top,x=i,y expr=\thisrow{mean-lambda-0.8}+\thisrow{std-lambda-0.8},col sep=comma] 
	{csv/lambda/2-rooms-3c/linear.csv};
	\addplot[name path=bot,draw=none,forget plot]
	table [name path=top,x=i,y expr=\thisrow{mean-lambda-0.8}-\thisrow{std-lambda-0.8},col sep=comma] 
	{csv/lambda/2-rooms-3c/linear.csv};
	\addplot [forget plot, draw=none,opacity=0.15,pattern=north east lines,fill=harlequin]
	fill between[of=top and bot];
	\addplot
	table [x=i,y=mean-lambda-0.9,col sep=comma] 
	{csv/lambda/2-rooms-3c/linear.csv};
	\addplot[name path=top,draw=none,forget plot]
	table [name path=top,x=i,y expr=\thisrow{mean-lambda-0.9}+\thisrow{std-lambda-0.9},col sep=comma] 
	{csv/lambda/2-rooms-3c/linear.csv};
	\addplot[name path=bot,draw=none,forget plot]
	table [name path=top,x=i,y expr=\thisrow{mean-lambda-0.9}-\thisrow{std-lambda-0.9},col sep=comma] 
	{csv/lambda/2-rooms-3c/linear.csv};
	\addplot [forget plot, draw=none,opacity=0.15,pattern=north east lines,fill=neonFuchsia]
	fill between[of=top and bot];
	\addplot
	table [x=i,y=mean-lambda-1.0,col sep=comma] 
	{csv/lambda/2-rooms-3c/linear.csv};
	\addplot[name path=top,draw=none,forget plot]
	table [name path=top,x=i,y expr=\thisrow{mean-lambda-1.0}+\thisrow{std-lambda-1.0},col sep=comma] 
	{csv/lambda/2-rooms-3c/linear.csv};
	\addplot[name path=bot,draw=none,forget plot]
	table [name path=top,x=i,y expr=\thisrow{mean-lambda-1.0}-\thisrow{std-lambda-1.0},col sep=comma] 
	{csv/lambda/2-rooms-3c/linear.csv};
	\addplot [forget plot, draw=none,opacity=0.15,pattern=north east lines,fill=bronze]
	fill between[of=top and bot];
	\addplot
	table [x=i,y=mean-likelihood,col sep=comma] 
	{csv/lambda/2-rooms-3c/linear.csv};
	\addplot[name path=top,draw=none,forget plot]
	table [name path=top,x=i,y expr=\thisrow{mean-likelihood}+\thisrow{std-likelihood},col sep=comma] 
	{csv/lambda/2-rooms-3c/linear.csv};
	\addplot[name path=bot,draw=none,forget plot]
	table [name path=top,x=i,y expr=\thisrow{mean-likelihood}-\thisrow{std-likelihood},col sep=comma] 
	{csv/lambda/2-rooms-3c/linear.csv};
	\addplot [forget plot, draw=none,opacity=0.15,pattern=north east lines,fill=black]
	fill between[of=top and bot];

	\end{axis}
	\end{tikzpicture}
	\label{fig:2-rooms-lambda-linear-3c}
}

\subfigure[2-rooms $\sin$ dynamic.]
{
	\begin{tikzpicture}
	\begin{axis}[
	width=0.45\textwidth,
	height=4.5cm,
	xmin=50,
	xmax=2950,
	xtick={500,1000,...,3000},
	ymin=0,
	ymax=0.9,
	ytick={0,0.2,...,0.8},
	%
	%
	xlabel=Iterations,
	ylabel=Average Return,
	mark options={scale=0.2},
	cycle list name = custom,
	scaled x ticks=base 10:-3
	]
	\addplot
	table [x=i,y=mean-lambda-0.1,col sep=comma] 
	{csv/lambda/2-rooms-3c/sin.csv};
	\addplot[name path=top,draw=none,forget plot]
	table [name path=top,x=i,y expr=\thisrow{mean-lambda-0.1}+\thisrow{std-lambda-0.1},col sep=comma] 
	{csv/lambda/2-rooms-3c/sin.csv};
	\addplot[name path=bot,draw=none,forget plot]
	table [name path=top,x=i,y expr=\thisrow{mean-lambda-0.1}-\thisrow{std-lambda-0.1},col sep=comma] 
	{csv/lambda/2-rooms-3c/sin.csv};
	\addplot [forget plot, draw=none,opacity=0.4,pattern=north east lines,fill=green1!60!black]
	fill between[of=top and bot];
	\addplot
	table [x=i,y=mean-lambda-0.2,col sep=comma] 
	{csv/lambda/2-rooms-3c/sin.csv};
	\addplot[name path=top,draw=none,forget plot]
	table [name path=top,x=i,y expr=\thisrow{mean-lambda-0.2}+\thisrow{std-lambda-0.2},col sep=comma] 
	{csv/lambda/2-rooms-3c/sin.csv};
	\addplot[name path=bot,draw=none,forget plot]
	table [name path=top,x=i,y expr=\thisrow{mean-lambda-0.2}-\thisrow{std-lambda-0.2},col sep=comma] 
	{csv/lambda/2-rooms-3c/sin.csv};
	\addplot [forget plot, draw=none,opacity=0.4,pattern=north east lines,fill=orange]
	fill between[of=top and bot];
	\addplot
	table [x=i,y=mean-lambda-0.3,col sep=comma] 
	{csv/lambda/2-rooms-3c/sin.csv};
	\addplot[name path=top,draw=none,forget plot]
	table [name path=top,x=i,y expr=\thisrow{mean-lambda-0.3}+\thisrow{std-lambda-0.3},col sep=comma] 
	{csv/lambda/2-rooms-3c/sin.csv};
	\addplot[name path=bot,draw=none,forget plot]
	table [name path=top,x=i,y expr=\thisrow{mean-lambda-0.3}-\thisrow{std-lambda-0.3},col sep=comma] 
	{csv/lambda/2-rooms-3c/sin.csv};
	\addplot [forget plot, draw=none,opacity=0.15,pattern=north east lines,fill=blue1]
	fill between[of=top and bot];
	\addplot
	table [x=i,y=mean-lambda-0.4,col sep=comma] 
	{csv/lambda/2-rooms-3c/sin.csv};
	\addplot[name path=top,draw=none,forget plot]
	table [name path=top,x=i,y expr=\thisrow{mean-lambda-0.4}+\thisrow{std-lambda-0.4},col sep=comma] 
	{csv/lambda/2-rooms-3c/sin.csv};
	\addplot[name path=bot,draw=none,forget plot]
	table [name path=top,x=i,y expr=\thisrow{mean-lambda-0.4}-\thisrow{std-lambda-0.4},col sep=comma] 
	{csv/lambda/2-rooms-3c/sin.csv};
	\addplot [forget plot, draw=none,opacity=0.15,pattern=north east lines,fill=violet]
	fill between[of=top and bot];
	\addplot
	table [x=i,y=mean-lambda-0.5,col sep=comma] 
	{csv/lambda/2-rooms-3c/sin.csv};
	\addplot[name path=top,draw=none,forget plot]
	table [name path=top,x=i,y expr=\thisrow{mean-lambda-0.5}+\thisrow{std-lambda-0.5},col sep=comma] 
	{csv/lambda/2-rooms-3c/sin.csv};
	\addplot[name path=bot,draw=none,forget plot]
	table [name path=top,x=i,y expr=\thisrow{mean-lambda-0.5}-\thisrow{std-lambda-0.5},col sep=comma] 
	{csv/lambda/2-rooms-3c/sin.csv};
	\addplot [forget plot, draw=none,opacity=0.15,pattern=north east lines,fill=mikadoYellow]
	fill between[of=top and bot];
	\addplot
	table [x=i,y=mean-lambda-0.6,col sep=comma] 
	{csv/lambda/2-rooms-3c/sin.csv};
	\addplot[name path=top,draw=none,forget plot]
	table [name path=top,x=i,y expr=\thisrow{mean-lambda-0.6}+\thisrow{std-lambda-0.6},col sep=comma] 
	{csv/lambda/2-rooms-3c/sin.csv};
	\addplot[name path=bot,draw=none,forget plot]
	table [name path=top,x=i,y expr=\thisrow{mean-lambda-0.6}-\thisrow{std-lambda-0.6},col sep=comma] 
	{csv/lambda/2-rooms-3c/sin.csv};
	\addplot [forget plot, draw=none,opacity=0.15,pattern=north east lines,fill=lapisLazuli]
	fill between[of=top and bot];
	\addplot
	table [x=i,y=mean-lambda-0.7,col sep=comma] 
	{csv/lambda/2-rooms-3c/sin.csv};
	\addplot[name path=top,draw=none,forget plot]
	table [name path=top,x=i,y expr=\thisrow{mean-lambda-0.7}+\thisrow{std-lambda-0.7},col sep=comma] 
	{csv/lambda/2-rooms-3c/sin.csv};
	\addplot[name path=bot,draw=none,forget plot]
	table [name path=top,x=i,y expr=\thisrow{mean-lambda-0.7}-\thisrow{std-lambda-0.7},col sep=comma] 
	{csv/lambda/2-rooms-3c/sin.csv};
	\addplot [forget plot, draw=none,opacity=0.15,pattern=north east lines,fill=harvardCrimson]
	fill between[of=top and bot];
	\addplot
	table [x=i,y=mean-lambda-0.8,col sep=comma] 
	{csv/lambda/2-rooms-3c/sin.csv};
	\addplot[name path=top,draw=none,forget plot]
	table [name path=top,x=i,y expr=\thisrow{mean-lambda-0.8}+\thisrow{std-lambda-0.8},col sep=comma] 
	{csv/lambda/2-rooms-3c/sin.csv};
	\addplot[name path=bot,draw=none,forget plot]
	table [name path=top,x=i,y expr=\thisrow{mean-lambda-0.8}-\thisrow{std-lambda-0.8},col sep=comma] 
	{csv/lambda/2-rooms-3c/sin.csv};
	\addplot [forget plot, draw=none,opacity=0.15,pattern=north east lines,fill=harlequin]
	fill between[of=top and bot];
	\addplot
	table [x=i,y=mean-lambda-0.9,col sep=comma] 
	{csv/lambda/2-rooms-3c/sin.csv};
	\addplot[name path=top,draw=none,forget plot]
	table [name path=top,x=i,y expr=\thisrow{mean-lambda-0.9}+\thisrow{std-lambda-0.9},col sep=comma] 
	{csv/lambda/2-rooms-3c/sin.csv};
	\addplot[name path=bot,draw=none,forget plot]
	table [name path=top,x=i,y expr=\thisrow{mean-lambda-0.9}-\thisrow{std-lambda-0.9},col sep=comma] 
	{csv/lambda/2-rooms-3c/sin.csv};
	\addplot [forget plot, draw=none,opacity=0.15,pattern=north east lines,fill=neonFuchsia]
	fill between[of=top and bot];
	\addplot
	table [x=i,y=mean-lambda-1.0,col sep=comma] 
	{csv/lambda/2-rooms-3c/sin.csv};
	\addplot[name path=top,draw=none,forget plot]
	table [name path=top,x=i,y expr=\thisrow{mean-lambda-1.0}+\thisrow{std-lambda-1.0},col sep=comma] 
	{csv/lambda/2-rooms-3c/sin.csv};
	\addplot[name path=bot,draw=none,forget plot]
	table [name path=top,x=i,y expr=\thisrow{mean-lambda-1.0}-\thisrow{std-lambda-1.0},col sep=comma] 
	{csv/lambda/2-rooms-3c/sin.csv};
	\addplot [forget plot, draw=none,opacity=0.15,pattern=north east lines,fill=bronze]
	fill between[of=top and bot];
	\addplot
	table [x=i,y=mean-likelihood,col sep=comma] 
	{csv/lambda/2-rooms-3c/sin.csv};
	\addplot[name path=top,draw=none,forget plot]
	table [name path=top,x=i,y expr=\thisrow{mean-likelihood}+\thisrow{std-likelihood},col sep=comma] 
	{csv/lambda/2-rooms-3c/sin.csv};
	\addplot[name path=bot,draw=none,forget plot]
	table [name path=top,x=i,y expr=\thisrow{mean-likelihood}-\thisrow{std-likelihood},col sep=comma] 
	{csv/lambda/2-rooms-3c/sin.csv};
	\addplot [forget plot, draw=none,opacity=0.15,pattern=north east lines,fill=black]
	fill between[of=top and bot];

	\end{axis}
	\end{tikzpicture}
	\label{fig:2-rooms-lambda-sin-3c}
}

	\caption{Average return achieved by $3$-T2VT \wrt different choices of $\lambda$ with $95\%$ confidence intervals computed using $50$ independent runs.}
	\label{fig:2-rooms-lambda-3c}
\end{figure*}

%% file: figure_lambda-sensitivity-3-rooms-1c.tex
\begin{figure*}[!b]
	\centering
	%
	%
	\begin{tikzpicture}
	\begin{customlegend}[legend columns=6,legend style={align=left,draw=none,column sep=2ex,font=\footnotesize},legend entries={ $\lambda=0.1$, $\lambda=0.2$, $\lambda=0.3$, $\lambda=0.4$, $\lambda=0.5$, $\lambda=0.6$, $\lambda=0.7$, $\lambda=0.8$, $\lambda=0.9$, $\lambda=1$, likelihood}]
	%
	%
	%
	\addlegendimage{green1!60!black,ultra thick,dashed}   
	\addlegendimage{orange!80!white,dotted,ultra thick}
	\addlegendimage{blue1,ultra thick}
	\addlegendimage{violet!90!white,dash dot,ultra thick}
	\addlegendimage{mikadoYellow, densely dashed,ultra thick}	
	\addlegendimage{lapisLazuli,densely dotted,ultra thick}
	\addlegendimage{harvardCrimson,densely dash dot, ultra thick}
	\addlegendimage{harlequin,ultra thick}
	\addlegendimage{neonFuchsia, densely dotted,ultra thick}
	\addlegendimage{bronze, densely dash dot,ultra thick}
	\addlegendimage{black,ultra thick}

	\end{customlegend}
	\end{tikzpicture}
	
\subfigure[3-rooms polynomial dynamic.]
{
	\begin{tikzpicture}
	\begin{axis}[
	width=0.45\textwidth,
	height=4.5cm,
	xmin=50,
	xmax=14950,
	xtick={3000,6000,...,12000},
	ymin=0,
	ymax=0.75,
	ytick={0.1,0.3,...,0.7},
	%
	%
	each nth point=4,
	xlabel=Iterations,
	ylabel=Average Return,
	mark options={scale=0.2},
	cycle list name = custom,
	scaled x ticks=base 10:-3
	]
	\addplot
	table [x=i,y=mean-lambda-0.1,col sep=comma] 
	{csv/lambda/3-rooms-1c/polynomial.csv};
	\addplot[name path=top,draw=none,forget plot]
	table [name path=top,x=i,y expr=\thisrow{mean-lambda-0.1}+\thisrow{std-lambda-0.1},col sep=comma] 
	{csv/lambda/3-rooms-1c/polynomial.csv};
	\addplot[name path=bot,draw=none,forget plot]
	table [name path=top,x=i,y expr=\thisrow{mean-lambda-0.1}-\thisrow{std-lambda-0.1},col sep=comma] 
	{csv/lambda/3-rooms-1c/polynomial.csv};
	\addplot [forget plot, draw=none,opacity=0.4,pattern=north east lines,fill=green1!60!black]
	fill between[of=top and bot];
	\addplot
	table [x=i,y=mean-lambda-0.2,col sep=comma] 
	{csv/lambda/3-rooms-1c/polynomial.csv};
	\addplot[name path=top,draw=none,forget plot]
	table [name path=top,x=i,y expr=\thisrow{mean-lambda-0.2}+\thisrow{std-lambda-0.2},col sep=comma] 
	{csv/lambda/3-rooms-1c/polynomial.csv};
	\addplot[name path=bot,draw=none,forget plot]
	table [name path=top,x=i,y expr=\thisrow{mean-lambda-0.2}-\thisrow{std-lambda-0.2},col sep=comma] 
	{csv/lambda/3-rooms-1c/polynomial.csv};
	\addplot [forget plot, draw=none,opacity=0.4,pattern=north east lines,fill=orange]
	fill between[of=top and bot];
	\addplot
	table [x=i,y=mean-lambda-0.3,col sep=comma] 
	{csv/lambda/3-rooms-1c/polynomial.csv};
	\addplot[name path=top,draw=none,forget plot]
	table [name path=top,x=i,y expr=\thisrow{mean-lambda-0.3}+\thisrow{std-lambda-0.3},col sep=comma] 
	{csv/lambda/3-rooms-1c/polynomial.csv};
	\addplot[name path=bot,draw=none,forget plot]
	table [name path=top,x=i,y expr=\thisrow{mean-lambda-0.3}-\thisrow{std-lambda-0.3},col sep=comma] 
	{csv/lambda/3-rooms-1c/polynomial.csv};
	\addplot [forget plot, draw=none,opacity=0.15,pattern=north east lines,fill=blue1]
	fill between[of=top and bot];
	\addplot
	table [x=i,y=mean-lambda-0.4,col sep=comma] 
	{csv/lambda/3-rooms-1c/polynomial.csv};
	\addplot[name path=top,draw=none,forget plot]
	table [name path=top,x=i,y expr=\thisrow{mean-lambda-0.4}+\thisrow{std-lambda-0.4},col sep=comma] 
	{csv/lambda/3-rooms-1c/polynomial.csv};
	\addplot[name path=bot,draw=none,forget plot]
	table [name path=top,x=i,y expr=\thisrow{mean-lambda-0.4}-\thisrow{std-lambda-0.4},col sep=comma] 
	{csv/lambda/3-rooms-1c/polynomial.csv};
	\addplot [forget plot, draw=none,opacity=0.15,pattern=north east lines,fill=violet]
	fill between[of=top and bot];
	\addplot
	table [x=i,y=mean-lambda-0.5,col sep=comma] 
	{csv/lambda/3-rooms-1c/polynomial.csv};
	\addplot[name path=top,draw=none,forget plot]
	table [name path=top,x=i,y expr=\thisrow{mean-lambda-0.5}+\thisrow{std-lambda-0.5},col sep=comma] 
	{csv/lambda/3-rooms-1c/polynomial.csv};
	\addplot[name path=bot,draw=none,forget plot]
	table [name path=top,x=i,y expr=\thisrow{mean-lambda-0.5}-\thisrow{std-lambda-0.5},col sep=comma] 
	{csv/lambda/3-rooms-1c/polynomial.csv};
	\addplot [forget plot, draw=none,opacity=0.15,pattern=north east lines,fill=mikadoYellow]
	fill between[of=top and bot];
	\addplot
	table [x=i,y=mean-lambda-0.6,col sep=comma] 
	{csv/lambda/3-rooms-1c/polynomial.csv};
	\addplot[name path=top,draw=none,forget plot]
	table [name path=top,x=i,y expr=\thisrow{mean-lambda-0.6}+\thisrow{std-lambda-0.6},col sep=comma] 
	{csv/lambda/3-rooms-1c/polynomial.csv};
	\addplot[name path=bot,draw=none,forget plot]
	table [name path=top,x=i,y expr=\thisrow{mean-lambda-0.6}-\thisrow{std-lambda-0.6},col sep=comma] 
	{csv/lambda/3-rooms-1c/polynomial.csv};
	\addplot [forget plot, draw=none,opacity=0.15,pattern=north east lines,fill=lapisLazuli]
	fill between[of=top and bot];
	\addplot
	table [x=i,y=mean-lambda-0.7,col sep=comma] 
	{csv/lambda/3-rooms-1c/polynomial.csv};
	\addplot[name path=top,draw=none,forget plot]
	table [name path=top,x=i,y expr=\thisrow{mean-lambda-0.7}+\thisrow{std-lambda-0.7},col sep=comma] 
	{csv/lambda/3-rooms-1c/polynomial.csv};
	\addplot[name path=bot,draw=none,forget plot]
	table [name path=top,x=i,y expr=\thisrow{mean-lambda-0.7}-\thisrow{std-lambda-0.7},col sep=comma] 
	{csv/lambda/3-rooms-1c/polynomial.csv};
	\addplot [forget plot, draw=none,opacity=0.15,pattern=north east lines,fill=harvardCrimson]
	fill between[of=top and bot];
	\addplot
	table [x=i,y=mean-lambda-0.8,col sep=comma] 
	{csv/lambda/3-rooms-1c/polynomial.csv};
	\addplot[name path=top,draw=none,forget plot]
	table [name path=top,x=i,y expr=\thisrow{mean-lambda-0.8}+\thisrow{std-lambda-0.8},col sep=comma] 
	{csv/lambda/3-rooms-1c/polynomial.csv};
	\addplot[name path=bot,draw=none,forget plot]
	table [name path=top,x=i,y expr=\thisrow{mean-lambda-0.8}-\thisrow{std-lambda-0.8},col sep=comma] 
	{csv/lambda/3-rooms-1c/polynomial.csv};
	\addplot [forget plot, draw=none,opacity=0.15,pattern=north east lines,fill=harlequin]
	fill between[of=top and bot];
	\addplot
	table [x=i,y=mean-lambda-0.9,col sep=comma] 
	{csv/lambda/3-rooms-1c/polynomial.csv};
	\addplot[name path=top,draw=none,forget plot]
	table [name path=top,x=i,y expr=\thisrow{mean-lambda-0.9}+\thisrow{std-lambda-0.9},col sep=comma] 
	{csv/lambda/3-rooms-1c/polynomial.csv};
	\addplot[name path=bot,draw=none,forget plot]
	table [name path=top,x=i,y expr=\thisrow{mean-lambda-0.9}-\thisrow{std-lambda-0.9},col sep=comma] 
	{csv/lambda/3-rooms-1c/polynomial.csv};
	\addplot [forget plot, draw=none,opacity=0.15,pattern=north east lines,fill=neonFuchsia]
	fill between[of=top and bot];
	\addplot
	table [x=i,y=mean-lambda-1.0,col sep=comma] 
	{csv/lambda/3-rooms-1c/polynomial.csv};
	\addplot[name path=top,draw=none,forget plot]
	table [name path=top,x=i,y expr=\thisrow{mean-lambda-1.0}+\thisrow{std-lambda-1.0},col sep=comma] 
	{csv/lambda/3-rooms-1c/polynomial.csv};
	\addplot[name path=bot,draw=none,forget plot]
	table [name path=top,x=i,y expr=\thisrow{mean-lambda-1.0}-\thisrow{std-lambda-1.0},col sep=comma] 
	{csv/lambda/3-rooms-1c/polynomial.csv};
	\addplot [forget plot, draw=none,opacity=0.15,pattern=north east lines,fill=bronze]
	fill between[of=top and bot];
	\addplot
	table [x=i,y=mean-likelihood,col sep=comma] 
	{csv/lambda/3-rooms-1c/polynomial.csv};
	\addplot[name path=top,draw=none,forget plot]
	table [name path=top,x=i,y expr=\thisrow{mean-likelihood}+\thisrow{std-likelihood},col sep=comma] 
	{csv/lambda/3-rooms-1c/polynomial.csv};
	\addplot[name path=bot,draw=none,forget plot]
	table [name path=top,x=i,y expr=\thisrow{mean-likelihood}-\thisrow{std-likelihood},col sep=comma] 
	{csv/lambda/3-rooms-1c/polynomial.csv};
	\addplot [forget plot, draw=none,opacity=0.15,pattern=north east lines,fill=black]
	fill between[of=top and bot];

	\end{axis}
	\end{tikzpicture}
	\label{fig:3-rooms-lambda-polynomial-1c}
}
\quad
\subfigure[3-rooms linear dynamic.]
{
	\begin{tikzpicture}
	\begin{axis}[
	width=0.45\textwidth,
	height=4.5cm,
	xmin=50,
	xmax=14950,
	xtick={3000,6000,...,12000},
	ymin=0,
	ymax=0.75,
	ytick={0.1,0.3,...,0.7},
	%
	%
	each nth point=4,
	xlabel=Iterations,
	ylabel=Average Return,
	mark options={scale=0.2},
	cycle list name = custom,
	scaled x ticks=base 10:-3
	]
	\addplot
	table [x=i,y=mean-lambda-0.1,col sep=comma] 
	{csv/lambda/3-rooms-1c/linear.csv};
	\addplot[name path=top,draw=none,forget plot]
	table [name path=top,x=i,y expr=\thisrow{mean-lambda-0.1}+\thisrow{std-lambda-0.1},col sep=comma] 
	{csv/lambda/3-rooms-1c/linear.csv};
	\addplot[name path=bot,draw=none,forget plot]
	table [name path=top,x=i,y expr=\thisrow{mean-lambda-0.1}-\thisrow{std-lambda-0.1},col sep=comma] 
	{csv/lambda/3-rooms-1c/linear.csv};
	\addplot [forget plot, draw=none,opacity=0.4,pattern=north east lines,fill=green1!60!black]
	fill between[of=top and bot];
	\addplot
	table [x=i,y=mean-lambda-0.2,col sep=comma] 
	{csv/lambda/3-rooms-1c/linear.csv};
	\addplot[name path=top,draw=none,forget plot]
	table [name path=top,x=i,y expr=\thisrow{mean-lambda-0.2}+\thisrow{std-lambda-0.2},col sep=comma] 
	{csv/lambda/3-rooms-1c/linear.csv};
	\addplot[name path=bot,draw=none,forget plot]
	table [name path=top,x=i,y expr=\thisrow{mean-lambda-0.2}-\thisrow{std-lambda-0.2},col sep=comma] 
	{csv/lambda/3-rooms-1c/linear.csv};
	\addplot [forget plot, draw=none,opacity=0.4,pattern=north east lines,fill=orange]
	fill between[of=top and bot];
	\addplot
	table [x=i,y=mean-lambda-0.3,col sep=comma] 
	{csv/lambda/3-rooms-1c/linear.csv};
	\addplot[name path=top,draw=none,forget plot]
	table [name path=top,x=i,y expr=\thisrow{mean-lambda-0.3}+\thisrow{std-lambda-0.3},col sep=comma] 
	{csv/lambda/3-rooms-1c/linear.csv};
	\addplot[name path=bot,draw=none,forget plot]
	table [name path=top,x=i,y expr=\thisrow{mean-lambda-0.3}-\thisrow{std-lambda-0.3},col sep=comma] 
	{csv/lambda/3-rooms-1c/linear.csv};
	\addplot [forget plot, draw=none,opacity=0.15,pattern=north east lines,fill=blue1]
	fill between[of=top and bot];
	\addplot
	table [x=i,y=mean-lambda-0.4,col sep=comma] 
	{csv/lambda/3-rooms-1c/linear.csv};
	\addplot[name path=top,draw=none,forget plot]
	table [name path=top,x=i,y expr=\thisrow{mean-lambda-0.4}+\thisrow{std-lambda-0.4},col sep=comma] 
	{csv/lambda/3-rooms-1c/linear.csv};
	\addplot[name path=bot,draw=none,forget plot]
	table [name path=top,x=i,y expr=\thisrow{mean-lambda-0.4}-\thisrow{std-lambda-0.4},col sep=comma] 
	{csv/lambda/3-rooms-1c/linear.csv};
	\addplot [forget plot, draw=none,opacity=0.15,pattern=north east lines,fill=violet]
	fill between[of=top and bot];
	\addplot
	table [x=i,y=mean-lambda-0.5,col sep=comma] 
	{csv/lambda/3-rooms-1c/linear.csv};
	\addplot[name path=top,draw=none,forget plot]
	table [name path=top,x=i,y expr=\thisrow{mean-lambda-0.5}+\thisrow{std-lambda-0.5},col sep=comma] 
	{csv/lambda/3-rooms-1c/linear.csv};
	\addplot[name path=bot,draw=none,forget plot]
	table [name path=top,x=i,y expr=\thisrow{mean-lambda-0.5}-\thisrow{std-lambda-0.5},col sep=comma] 
	{csv/lambda/3-rooms-1c/linear.csv};
	\addplot [forget plot, draw=none,opacity=0.15,pattern=north east lines,fill=mikadoYellow]
	fill between[of=top and bot];
	\addplot
	table [x=i,y=mean-lambda-0.6,col sep=comma] 
	{csv/lambda/3-rooms-1c/linear.csv};
	\addplot[name path=top,draw=none,forget plot]
	table [name path=top,x=i,y expr=\thisrow{mean-lambda-0.6}+\thisrow{std-lambda-0.6},col sep=comma] 
	{csv/lambda/3-rooms-1c/linear.csv};
	\addplot[name path=bot,draw=none,forget plot]
	table [name path=top,x=i,y expr=\thisrow{mean-lambda-0.6}-\thisrow{std-lambda-0.6},col sep=comma] 
	{csv/lambda/3-rooms-1c/linear.csv};
	\addplot [forget plot, draw=none,opacity=0.15,pattern=north east lines,fill=lapisLazuli]
	fill between[of=top and bot];
	\addplot
	table [x=i,y=mean-lambda-0.7,col sep=comma] 
	{csv/lambda/3-rooms-1c/linear.csv};
	\addplot[name path=top,draw=none,forget plot]
	table [name path=top,x=i,y expr=\thisrow{mean-lambda-0.7}+\thisrow{std-lambda-0.7},col sep=comma] 
	{csv/lambda/3-rooms-1c/linear.csv};
	\addplot[name path=bot,draw=none,forget plot]
	table [name path=top,x=i,y expr=\thisrow{mean-lambda-0.7}-\thisrow{std-lambda-0.7},col sep=comma] 
	{csv/lambda/3-rooms-1c/linear.csv};
	\addplot [forget plot, draw=none,opacity=0.15,pattern=north east lines,fill=harvardCrimson]
	fill between[of=top and bot];
	\addplot
	table [x=i,y=mean-lambda-0.8,col sep=comma] 
	{csv/lambda/3-rooms-1c/linear.csv};
	\addplot[name path=top,draw=none,forget plot]
	table [name path=top,x=i,y expr=\thisrow{mean-lambda-0.8}+\thisrow{std-lambda-0.8},col sep=comma] 
	{csv/lambda/3-rooms-1c/linear.csv};
	\addplot[name path=bot,draw=none,forget plot]
	table [name path=top,x=i,y expr=\thisrow{mean-lambda-0.8}-\thisrow{std-lambda-0.8},col sep=comma] 
	{csv/lambda/3-rooms-1c/linear.csv};
	\addplot [forget plot, draw=none,opacity=0.15,pattern=north east lines,fill=harlequin]
	fill between[of=top and bot];
	\addplot
	table [x=i,y=mean-lambda-0.9,col sep=comma] 
	{csv/lambda/3-rooms-1c/linear.csv};
	\addplot[name path=top,draw=none,forget plot]
	table [name path=top,x=i,y expr=\thisrow{mean-lambda-0.9}+\thisrow{std-lambda-0.9},col sep=comma] 
	{csv/lambda/3-rooms-1c/linear.csv};
	\addplot[name path=bot,draw=none,forget plot]
	table [name path=top,x=i,y expr=\thisrow{mean-lambda-0.9}-\thisrow{std-lambda-0.9},col sep=comma] 
	{csv/lambda/3-rooms-1c/linear.csv};
	\addplot [forget plot, draw=none,opacity=0.15,pattern=north east lines,fill=neonFuchsia]
	fill between[of=top and bot];
	\addplot
	table [x=i,y=mean-lambda-1.0,col sep=comma] 
	{csv/lambda/3-rooms-1c/linear.csv};
	\addplot[name path=top,draw=none,forget plot]
	table [name path=top,x=i,y expr=\thisrow{mean-lambda-1.0}+\thisrow{std-lambda-1.0},col sep=comma] 
	{csv/lambda/3-rooms-1c/linear.csv};
	\addplot[name path=bot,draw=none,forget plot]
	table [name path=top,x=i,y expr=\thisrow{mean-lambda-1.0}-\thisrow{std-lambda-1.0},col sep=comma] 
	{csv/lambda/3-rooms-1c/linear.csv};
	\addplot [forget plot, draw=none,opacity=0.15,pattern=north east lines,fill=bronze]
	fill between[of=top and bot];
	\addplot
	table [x=i,y=mean-likelihood,col sep=comma] 
	{csv/lambda/3-rooms-1c/linear.csv};
	\addplot[name path=top,draw=none,forget plot]
	table [name path=top,x=i,y expr=\thisrow{mean-likelihood}+\thisrow{std-likelihood},col sep=comma] 
	{csv/lambda/3-rooms-1c/linear.csv};
	\addplot[name path=bot,draw=none,forget plot]
	table [name path=top,x=i,y expr=\thisrow{mean-likelihood}-\thisrow{std-likelihood},col sep=comma] 
	{csv/lambda/3-rooms-1c/linear.csv};
	\addplot [forget plot, draw=none,opacity=0.15,pattern=north east lines,fill=black]
	fill between[of=top and bot];

	\end{axis}
	\end{tikzpicture}
	\label{fig:3-rooms-lambda-linear-1c}
}

\subfigure[3-rooms $\sin$ dynamic.]
{
	\begin{tikzpicture}
	\begin{axis}[
	width=0.45\textwidth,
	height=4.5cm,
	xmin=50,
	xmax=14950,
	xtick={3000,6000,...,12000},
	ymin=0,
	ymax=0.9,
	ytick={0,0.2,...,0.8},
	%
	%
	each nth point=4,
	xlabel=Iterations,
	ylabel=Average Return,
	mark options={scale=0.2},
	cycle list name = custom,
	scaled x ticks=base 10:-3
	]
	\addplot
	table [x=i,y=mean-lambda-0.1,col sep=comma] 
	{csv/lambda/3-rooms-1c/sin.csv};
	\addplot[name path=top,draw=none,forget plot]
	table [name path=top,x=i,y expr=\thisrow{mean-lambda-0.1}+\thisrow{std-lambda-0.1},col sep=comma] 
	{csv/lambda/3-rooms-1c/sin.csv};
	\addplot[name path=bot,draw=none,forget plot]
	table [name path=top,x=i,y expr=\thisrow{mean-lambda-0.1}-\thisrow{std-lambda-0.1},col sep=comma] 
	{csv/lambda/3-rooms-1c/sin.csv};
	\addplot [forget plot, draw=none,opacity=0.4,pattern=north east lines,fill=green1!60!black]
	fill between[of=top and bot];
	\addplot
	table [x=i,y=mean-lambda-0.2,col sep=comma] 
	{csv/lambda/3-rooms-1c/sin.csv};
	\addplot[name path=top,draw=none,forget plot]
	table [name path=top,x=i,y expr=\thisrow{mean-lambda-0.2}+\thisrow{std-lambda-0.2},col sep=comma] 
	{csv/lambda/3-rooms-1c/sin.csv};
	\addplot[name path=bot,draw=none,forget plot]
	table [name path=top,x=i,y expr=\thisrow{mean-lambda-0.2}-\thisrow{std-lambda-0.2},col sep=comma] 
	{csv/lambda/3-rooms-1c/sin.csv};
	\addplot [forget plot, draw=none,opacity=0.4,pattern=north east lines,fill=orange]
	fill between[of=top and bot];
	\addplot
	table [x=i,y=mean-lambda-0.3,col sep=comma] 
	{csv/lambda/3-rooms-1c/sin.csv};
	\addplot[name path=top,draw=none,forget plot]
	table [name path=top,x=i,y expr=\thisrow{mean-lambda-0.3}+\thisrow{std-lambda-0.3},col sep=comma] 
	{csv/lambda/3-rooms-1c/sin.csv};
	\addplot[name path=bot,draw=none,forget plot]
	table [name path=top,x=i,y expr=\thisrow{mean-lambda-0.3}-\thisrow{std-lambda-0.3},col sep=comma] 
	{csv/lambda/3-rooms-1c/sin.csv};
	\addplot [forget plot, draw=none,opacity=0.15,pattern=north east lines,fill=blue1]
	fill between[of=top and bot];
	\addplot
	table [x=i,y=mean-lambda-0.4,col sep=comma] 
	{csv/lambda/3-rooms-1c/sin.csv};
	\addplot[name path=top,draw=none,forget plot]
	table [name path=top,x=i,y expr=\thisrow{mean-lambda-0.4}+\thisrow{std-lambda-0.4},col sep=comma] 
	{csv/lambda/3-rooms-1c/sin.csv};
	\addplot[name path=bot,draw=none,forget plot]
	table [name path=top,x=i,y expr=\thisrow{mean-lambda-0.4}-\thisrow{std-lambda-0.4},col sep=comma] 
	{csv/lambda/3-rooms-1c/sin.csv};
	\addplot [forget plot, draw=none,opacity=0.15,pattern=north east lines,fill=violet]
	fill between[of=top and bot];
	\addplot
	table [x=i,y=mean-lambda-0.5,col sep=comma] 
	{csv/lambda/3-rooms-1c/sin.csv};
	\addplot[name path=top,draw=none,forget plot]
	table [name path=top,x=i,y expr=\thisrow{mean-lambda-0.5}+\thisrow{std-lambda-0.5},col sep=comma] 
	{csv/lambda/3-rooms-1c/sin.csv};
	\addplot[name path=bot,draw=none,forget plot]
	table [name path=top,x=i,y expr=\thisrow{mean-lambda-0.5}-\thisrow{std-lambda-0.5},col sep=comma] 
	{csv/lambda/3-rooms-1c/sin.csv};
	\addplot [forget plot, draw=none,opacity=0.15,pattern=north east lines,fill=mikadoYellow]
	fill between[of=top and bot];
	\addplot
	table [x=i,y=mean-lambda-0.6,col sep=comma] 
	{csv/lambda/3-rooms-1c/sin.csv};
	\addplot[name path=top,draw=none,forget plot]
	table [name path=top,x=i,y expr=\thisrow{mean-lambda-0.6}+\thisrow{std-lambda-0.6},col sep=comma] 
	{csv/lambda/3-rooms-1c/sin.csv};
	\addplot[name path=bot,draw=none,forget plot]
	table [name path=top,x=i,y expr=\thisrow{mean-lambda-0.6}-\thisrow{std-lambda-0.6},col sep=comma] 
	{csv/lambda/3-rooms-1c/sin.csv};
	\addplot [forget plot, draw=none,opacity=0.15,pattern=north east lines,fill=lapisLazuli]
	fill between[of=top and bot];
	\addplot
	table [x=i,y=mean-lambda-0.7,col sep=comma] 
	{csv/lambda/3-rooms-1c/sin.csv};
	\addplot[name path=top,draw=none,forget plot]
	table [name path=top,x=i,y expr=\thisrow{mean-lambda-0.7}+\thisrow{std-lambda-0.7},col sep=comma] 
	{csv/lambda/3-rooms-1c/sin.csv};
	\addplot[name path=bot,draw=none,forget plot]
	table [name path=top,x=i,y expr=\thisrow{mean-lambda-0.7}-\thisrow{std-lambda-0.7},col sep=comma] 
	{csv/lambda/3-rooms-1c/sin.csv};
	\addplot [forget plot, draw=none,opacity=0.15,pattern=north east lines,fill=harvardCrimson]
	fill between[of=top and bot];
	\addplot
	table [x=i,y=mean-lambda-0.8,col sep=comma] 
	{csv/lambda/3-rooms-1c/sin.csv};
	\addplot[name path=top,draw=none,forget plot]
	table [name path=top,x=i,y expr=\thisrow{mean-lambda-0.8}+\thisrow{std-lambda-0.8},col sep=comma] 
	{csv/lambda/3-rooms-1c/sin.csv};
	\addplot[name path=bot,draw=none,forget plot]
	table [name path=top,x=i,y expr=\thisrow{mean-lambda-0.8}-\thisrow{std-lambda-0.8},col sep=comma] 
	{csv/lambda/3-rooms-1c/sin.csv};
	\addplot [forget plot, draw=none,opacity=0.15,pattern=north east lines,fill=harlequin]
	fill between[of=top and bot];
	\addplot
	table [x=i,y=mean-lambda-0.9,col sep=comma] 
	{csv/lambda/3-rooms-1c/sin.csv};
	\addplot[name path=top,draw=none,forget plot]
	table [name path=top,x=i,y expr=\thisrow{mean-lambda-0.9}+\thisrow{std-lambda-0.9},col sep=comma] 
	{csv/lambda/3-rooms-1c/sin.csv};
	\addplot[name path=bot,draw=none,forget plot]
	table [name path=top,x=i,y expr=\thisrow{mean-lambda-0.9}-\thisrow{std-lambda-0.9},col sep=comma] 
	{csv/lambda/3-rooms-1c/sin.csv};
	\addplot [forget plot, draw=none,opacity=0.15,pattern=north east lines,fill=neonFuchsia]
	fill between[of=top and bot];
	\addplot
	table [x=i,y=mean-lambda-1.0,col sep=comma] 
	{csv/lambda/3-rooms-1c/sin.csv};
	\addplot[name path=top,draw=none,forget plot]
	table [name path=top,x=i,y expr=\thisrow{mean-lambda-1.0}+\thisrow{std-lambda-1.0},col sep=comma] 
	{csv/lambda/3-rooms-1c/sin.csv};
	\addplot[name path=bot,draw=none,forget plot]
	table [name path=top,x=i,y expr=\thisrow{mean-lambda-1.0}-\thisrow{std-lambda-1.0},col sep=comma] 
	{csv/lambda/3-rooms-1c/sin.csv};
	\addplot [forget plot, draw=none,opacity=0.15,pattern=north east lines,fill=bronze]
	fill between[of=top and bot];
	\addplot
	table [x=i,y=mean-likelihood,col sep=comma] 
	{csv/lambda/3-rooms-1c/sin.csv};
	\addplot[name path=top,draw=none,forget plot]
	table [name path=top,x=i,y expr=\thisrow{mean-likelihood}+\thisrow{std-likelihood},col sep=comma] 
	{csv/lambda/3-rooms-1c/sin.csv};
	\addplot[name path=bot,draw=none,forget plot]
	table [name path=top,x=i,y expr=\thisrow{mean-likelihood}-\thisrow{std-likelihood},col sep=comma] 
	{csv/lambda/3-rooms-1c/sin.csv};
	\addplot [forget plot, draw=none,opacity=0.15,pattern=north east lines,fill=black]
	fill between[of=top and bot];

	\end{axis}
	\end{tikzpicture}
	\label{fig:3-rooms-lambda-sin-1c}
}

	\caption{Average return achieved by $1$-T2VT \wrt different choices of $\lambda$ with $95\%$ confidence intervals computed using $50$ independent runs.}
	\label{fig:3-rooms-lambda-1c}
\end{figure*}

%% file: figure_lambda-sensitivity-3-rooms-3c.tex
\begin{figure*}[!b]
	\centering
	%
	%
	\begin{tikzpicture}
	\begin{customlegend}[legend columns=6,legend style={align=left,draw=none,column sep=2ex,font=\footnotesize},legend entries={ $\lambda=0.1$, $\lambda=0.2$, $\lambda=0.3$, $\lambda=0.4$, $\lambda=0.5$, $\lambda=0.6$, $\lambda=0.7$, $\lambda=0.8$, $\lambda=0.9$, $\lambda=1$, likelihood}]
	%
	%
	%
	\addlegendimage{green1!60!black,ultra thick,dashed}   
	\addlegendimage{orange!80!white,dotted,ultra thick}
	\addlegendimage{blue1,ultra thick}
	\addlegendimage{violet!90!white,dash dot,ultra thick}
	\addlegendimage{mikadoYellow, densely dashed,ultra thick}	
	\addlegendimage{lapisLazuli,densely dotted,ultra thick}
	\addlegendimage{harvardCrimson,densely dash dot, ultra thick}
	\addlegendimage{harlequin,ultra thick}
	\addlegendimage{neonFuchsia, densely dotted,ultra thick}
	\addlegendimage{bronze, densely dash dot,ultra thick}
	\addlegendimage{black,ultra thick}

	\end{customlegend}
	\end{tikzpicture}
	
\subfigure[3-rooms polynomial dynamic.]
{
	\begin{tikzpicture}
	\begin{axis}[
	width=0.45\textwidth,
	height=4.5cm,
	xmin=50,
	xmax=14950,
	xtick={3000,6000,...,12000},
	ymin=0,
	ymax=0.75,
	ytick={0.1,0.3,...,0.7},
	%
	%
	each nth point=4,
	xlabel=Iterations,
	ylabel=Average Return,
	mark options={scale=0.2},
	cycle list name = custom,
	scaled x ticks=base 10:-3
	]
	\addplot
	table [x=i,y=mean-lambda-0.1,col sep=comma] 
	{csv/lambda/3-rooms-3c/polynomial.csv};
	\addplot[name path=top,draw=none,forget plot]
	table [name path=top,x=i,y expr=\thisrow{mean-lambda-0.1}+\thisrow{std-lambda-0.1},col sep=comma] 
	{csv/lambda/3-rooms-3c/polynomial.csv};
	\addplot[name path=bot,draw=none,forget plot]
	table [name path=top,x=i,y expr=\thisrow{mean-lambda-0.1}-\thisrow{std-lambda-0.1},col sep=comma] 
	{csv/lambda/3-rooms-3c/polynomial.csv};
	\addplot [forget plot, draw=none,opacity=0.4,pattern=north east lines,fill=green1!60!black]
	fill between[of=top and bot];
	\addplot
	table [x=i,y=mean-lambda-0.2,col sep=comma] 
	{csv/lambda/3-rooms-3c/polynomial.csv};
	\addplot[name path=top,draw=none,forget plot]
	table [name path=top,x=i,y expr=\thisrow{mean-lambda-0.2}+\thisrow{std-lambda-0.2},col sep=comma] 
	{csv/lambda/3-rooms-3c/polynomial.csv};
	\addplot[name path=bot,draw=none,forget plot]
	table [name path=top,x=i,y expr=\thisrow{mean-lambda-0.2}-\thisrow{std-lambda-0.2},col sep=comma] 
	{csv/lambda/3-rooms-3c/polynomial.csv};
	\addplot [forget plot, draw=none,opacity=0.4,pattern=north east lines,fill=orange]
	fill between[of=top and bot];
	\addplot
	table [x=i,y=mean-lambda-0.3,col sep=comma] 
	{csv/lambda/3-rooms-3c/polynomial.csv};
	\addplot[name path=top,draw=none,forget plot]
	table [name path=top,x=i,y expr=\thisrow{mean-lambda-0.3}+\thisrow{std-lambda-0.3},col sep=comma] 
	{csv/lambda/3-rooms-3c/polynomial.csv};
	\addplot[name path=bot,draw=none,forget plot]
	table [name path=top,x=i,y expr=\thisrow{mean-lambda-0.3}-\thisrow{std-lambda-0.3},col sep=comma] 
	{csv/lambda/3-rooms-3c/polynomial.csv};
	\addplot [forget plot, draw=none,opacity=0.15,pattern=north east lines,fill=blue1]
	fill between[of=top and bot];
	\addplot
	table [x=i,y=mean-lambda-0.4,col sep=comma] 
	{csv/lambda/3-rooms-3c/polynomial.csv};
	\addplot[name path=top,draw=none,forget plot]
	table [name path=top,x=i,y expr=\thisrow{mean-lambda-0.4}+\thisrow{std-lambda-0.4},col sep=comma] 
	{csv/lambda/3-rooms-3c/polynomial.csv};
	\addplot[name path=bot,draw=none,forget plot]
	table [name path=top,x=i,y expr=\thisrow{mean-lambda-0.4}-\thisrow{std-lambda-0.4},col sep=comma] 
	{csv/lambda/3-rooms-3c/polynomial.csv};
	\addplot [forget plot, draw=none,opacity=0.15,pattern=north east lines,fill=violet]
	fill between[of=top and bot];
	\addplot
	table [x=i,y=mean-lambda-0.5,col sep=comma] 
	{csv/lambda/3-rooms-3c/polynomial.csv};
	\addplot[name path=top,draw=none,forget plot]
	table [name path=top,x=i,y expr=\thisrow{mean-lambda-0.5}+\thisrow{std-lambda-0.5},col sep=comma] 
	{csv/lambda/3-rooms-3c/polynomial.csv};
	\addplot[name path=bot,draw=none,forget plot]
	table [name path=top,x=i,y expr=\thisrow{mean-lambda-0.5}-\thisrow{std-lambda-0.5},col sep=comma] 
	{csv/lambda/3-rooms-3c/polynomial.csv};
	\addplot [forget plot, draw=none,opacity=0.15,pattern=north east lines,fill=mikadoYellow]
	fill between[of=top and bot];
	\addplot
	table [x=i,y=mean-lambda-0.6,col sep=comma] 
	{csv/lambda/3-rooms-3c/polynomial.csv};
	\addplot[name path=top,draw=none,forget plot]
	table [name path=top,x=i,y expr=\thisrow{mean-lambda-0.6}+\thisrow{std-lambda-0.6},col sep=comma] 
	{csv/lambda/3-rooms-3c/polynomial.csv};
	\addplot[name path=bot,draw=none,forget plot]
	table [name path=top,x=i,y expr=\thisrow{mean-lambda-0.6}-\thisrow{std-lambda-0.6},col sep=comma] 
	{csv/lambda/3-rooms-3c/polynomial.csv};
	\addplot [forget plot, draw=none,opacity=0.15,pattern=north east lines,fill=lapisLazuli]
	fill between[of=top and bot];
	\addplot
	table [x=i,y=mean-lambda-0.7,col sep=comma] 
	{csv/lambda/3-rooms-3c/polynomial.csv};
	\addplot[name path=top,draw=none,forget plot]
	table [name path=top,x=i,y expr=\thisrow{mean-lambda-0.7}+\thisrow{std-lambda-0.7},col sep=comma] 
	{csv/lambda/3-rooms-3c/polynomial.csv};
	\addplot[name path=bot,draw=none,forget plot]
	table [name path=top,x=i,y expr=\thisrow{mean-lambda-0.7}-\thisrow{std-lambda-0.7},col sep=comma] 
	{csv/lambda/3-rooms-3c/polynomial.csv};
	\addplot [forget plot, draw=none,opacity=0.15,pattern=north east lines,fill=harvardCrimson]
	fill between[of=top and bot];
	\addplot
	table [x=i,y=mean-lambda-0.8,col sep=comma] 
	{csv/lambda/3-rooms-3c/polynomial.csv};
	\addplot[name path=top,draw=none,forget plot]
	table [name path=top,x=i,y expr=\thisrow{mean-lambda-0.8}+\thisrow{std-lambda-0.8},col sep=comma] 
	{csv/lambda/3-rooms-3c/polynomial.csv};
	\addplot[name path=bot,draw=none,forget plot]
	table [name path=top,x=i,y expr=\thisrow{mean-lambda-0.8}-\thisrow{std-lambda-0.8},col sep=comma] 
	{csv/lambda/3-rooms-3c/polynomial.csv};
	\addplot [forget plot, draw=none,opacity=0.15,pattern=north east lines,fill=harlequin]
	fill between[of=top and bot];
	\addplot
	table [x=i,y=mean-lambda-0.9,col sep=comma] 
	{csv/lambda/3-rooms-3c/polynomial.csv};
	\addplot[name path=top,draw=none,forget plot]
	table [name path=top,x=i,y expr=\thisrow{mean-lambda-0.9}+\thisrow{std-lambda-0.9},col sep=comma] 
	{csv/lambda/3-rooms-3c/polynomial.csv};
	\addplot[name path=bot,draw=none,forget plot]
	table [name path=top,x=i,y expr=\thisrow{mean-lambda-0.9}-\thisrow{std-lambda-0.9},col sep=comma] 
	{csv/lambda/3-rooms-3c/polynomial.csv};
	\addplot [forget plot, draw=none,opacity=0.15,pattern=north east lines,fill=neonFuchsia]
	fill between[of=top and bot];
	\addplot
	table [x=i,y=mean-lambda-1.0,col sep=comma] 
	{csv/lambda/3-rooms-3c/polynomial.csv};
	\addplot[name path=top,draw=none,forget plot]
	table [name path=top,x=i,y expr=\thisrow{mean-lambda-1.0}+\thisrow{std-lambda-1.0},col sep=comma] 
	{csv/lambda/3-rooms-3c/polynomial.csv};
	\addplot[name path=bot,draw=none,forget plot]
	table [name path=top,x=i,y expr=\thisrow{mean-lambda-1.0}-\thisrow{std-lambda-1.0},col sep=comma] 
	{csv/lambda/3-rooms-3c/polynomial.csv};
	\addplot [forget plot, draw=none,opacity=0.15,pattern=north east lines,fill=bronze]
	fill between[of=top and bot];
	\addplot
	table [x=i,y=mean-likelihood,col sep=comma] 
	{csv/lambda/3-rooms-3c/polynomial.csv};
	\addplot[name path=top,draw=none,forget plot]
	table [name path=top,x=i,y expr=\thisrow{mean-likelihood}+\thisrow{std-likelihood},col sep=comma] 
	{csv/lambda/3-rooms-3c/polynomial.csv};
	\addplot[name path=bot,draw=none,forget plot]
	table [name path=top,x=i,y expr=\thisrow{mean-likelihood}-\thisrow{std-likelihood},col sep=comma] 
	{csv/lambda/3-rooms-3c/polynomial.csv};
	\addplot [forget plot, draw=none,opacity=0.15,pattern=north east lines,fill=black]
	fill between[of=top and bot];

	\end{axis}
	\end{tikzpicture}
	\label{fig:3-rooms-lambda-polynomial-3c}
}
\quad
\subfigure[3-rooms linear dynamic.]
{
	\begin{tikzpicture}
	\begin{axis}[
	width=0.45\textwidth,
	height=4.5cm,
	xmin=50,
	xmax=14950,
	xtick={3000,6000,...,12000},
	ymin=0,
	ymax=0.75,
	ytick={0.1,0.3,...,0.7},
	%
	%
	each nth point=4,
	xlabel=Iterations,
	ylabel=Average Return,
	mark options={scale=0.2},
	cycle list name = custom,
	scaled x ticks=base 10:-3
	]
	\addplot
	table [x=i,y=mean-lambda-0.1,col sep=comma] 
	{csv/lambda/3-rooms-3c/linear.csv};
	\addplot[name path=top,draw=none,forget plot]
	table [name path=top,x=i,y expr=\thisrow{mean-lambda-0.1}+\thisrow{std-lambda-0.1},col sep=comma] 
	{csv/lambda/3-rooms-3c/linear.csv};
	\addplot[name path=bot,draw=none,forget plot]
	table [name path=top,x=i,y expr=\thisrow{mean-lambda-0.1}-\thisrow{std-lambda-0.1},col sep=comma] 
	{csv/lambda/3-rooms-3c/linear.csv};
	\addplot [forget plot, draw=none,opacity=0.4,pattern=north east lines,fill=green1!60!black]
	fill between[of=top and bot];
	\addplot
	table [x=i,y=mean-lambda-0.2,col sep=comma] 
	{csv/lambda/3-rooms-3c/linear.csv};
	\addplot[name path=top,draw=none,forget plot]
	table [name path=top,x=i,y expr=\thisrow{mean-lambda-0.2}+\thisrow{std-lambda-0.2},col sep=comma] 
	{csv/lambda/3-rooms-3c/linear.csv};
	\addplot[name path=bot,draw=none,forget plot]
	table [name path=top,x=i,y expr=\thisrow{mean-lambda-0.2}-\thisrow{std-lambda-0.2},col sep=comma] 
	{csv/lambda/3-rooms-3c/linear.csv};
	\addplot [forget plot, draw=none,opacity=0.4,pattern=north east lines,fill=orange]
	fill between[of=top and bot];
	\addplot
	table [x=i,y=mean-lambda-0.3,col sep=comma] 
	{csv/lambda/3-rooms-3c/linear.csv};
	\addplot[name path=top,draw=none,forget plot]
	table [name path=top,x=i,y expr=\thisrow{mean-lambda-0.3}+\thisrow{std-lambda-0.3},col sep=comma] 
	{csv/lambda/3-rooms-3c/linear.csv};
	\addplot[name path=bot,draw=none,forget plot]
	table [name path=top,x=i,y expr=\thisrow{mean-lambda-0.3}-\thisrow{std-lambda-0.3},col sep=comma] 
	{csv/lambda/3-rooms-3c/linear.csv};
	\addplot [forget plot, draw=none,opacity=0.15,pattern=north east lines,fill=blue1]
	fill between[of=top and bot];
	\addplot
	table [x=i,y=mean-lambda-0.4,col sep=comma] 
	{csv/lambda/3-rooms-3c/linear.csv};
	\addplot[name path=top,draw=none,forget plot]
	table [name path=top,x=i,y expr=\thisrow{mean-lambda-0.4}+\thisrow{std-lambda-0.4},col sep=comma] 
	{csv/lambda/3-rooms-3c/linear.csv};
	\addplot[name path=bot,draw=none,forget plot]
	table [name path=top,x=i,y expr=\thisrow{mean-lambda-0.4}-\thisrow{std-lambda-0.4},col sep=comma] 
	{csv/lambda/3-rooms-3c/linear.csv};
	\addplot [forget plot, draw=none,opacity=0.15,pattern=north east lines,fill=violet]
	fill between[of=top and bot];
	\addplot
	table [x=i,y=mean-lambda-0.5,col sep=comma] 
	{csv/lambda/3-rooms-3c/linear.csv};
	\addplot[name path=top,draw=none,forget plot]
	table [name path=top,x=i,y expr=\thisrow{mean-lambda-0.5}+\thisrow{std-lambda-0.5},col sep=comma] 
	{csv/lambda/3-rooms-3c/linear.csv};
	\addplot[name path=bot,draw=none,forget plot]
	table [name path=top,x=i,y expr=\thisrow{mean-lambda-0.5}-\thisrow{std-lambda-0.5},col sep=comma] 
	{csv/lambda/3-rooms-3c/linear.csv};
	\addplot [forget plot, draw=none,opacity=0.15,pattern=north east lines,fill=mikadoYellow]
	fill between[of=top and bot];
	\addplot
	table [x=i,y=mean-lambda-0.6,col sep=comma] 
	{csv/lambda/3-rooms-3c/linear.csv};
	\addplot[name path=top,draw=none,forget plot]
	table [name path=top,x=i,y expr=\thisrow{mean-lambda-0.6}+\thisrow{std-lambda-0.6},col sep=comma] 
	{csv/lambda/3-rooms-3c/linear.csv};
	\addplot[name path=bot,draw=none,forget plot]
	table [name path=top,x=i,y expr=\thisrow{mean-lambda-0.6}-\thisrow{std-lambda-0.6},col sep=comma] 
	{csv/lambda/3-rooms-3c/linear.csv};
	\addplot [forget plot, draw=none,opacity=0.15,pattern=north east lines,fill=lapisLazuli]
	fill between[of=top and bot];
	\addplot
	table [x=i,y=mean-lambda-0.7,col sep=comma] 
	{csv/lambda/3-rooms-3c/linear.csv};
	\addplot[name path=top,draw=none,forget plot]
	table [name path=top,x=i,y expr=\thisrow{mean-lambda-0.7}+\thisrow{std-lambda-0.7},col sep=comma] 
	{csv/lambda/3-rooms-3c/linear.csv};
	\addplot[name path=bot,draw=none,forget plot]
	table [name path=top,x=i,y expr=\thisrow{mean-lambda-0.7}-\thisrow{std-lambda-0.7},col sep=comma] 
	{csv/lambda/3-rooms-3c/linear.csv};
	\addplot [forget plot, draw=none,opacity=0.15,pattern=north east lines,fill=harvardCrimson]
	fill between[of=top and bot];
	\addplot
	table [x=i,y=mean-lambda-0.8,col sep=comma] 
	{csv/lambda/3-rooms-3c/linear.csv};
	\addplot[name path=top,draw=none,forget plot]
	table [name path=top,x=i,y expr=\thisrow{mean-lambda-0.8}+\thisrow{std-lambda-0.8},col sep=comma] 
	{csv/lambda/3-rooms-3c/linear.csv};
	\addplot[name path=bot,draw=none,forget plot]
	table [name path=top,x=i,y expr=\thisrow{mean-lambda-0.8}-\thisrow{std-lambda-0.8},col sep=comma] 
	{csv/lambda/3-rooms-3c/linear.csv};
	\addplot [forget plot, draw=none,opacity=0.15,pattern=north east lines,fill=harlequin]
	fill between[of=top and bot];
	\addplot
	table [x=i,y=mean-lambda-0.9,col sep=comma] 
	{csv/lambda/3-rooms-3c/linear.csv};
	\addplot[name path=top,draw=none,forget plot]
	table [name path=top,x=i,y expr=\thisrow{mean-lambda-0.9}+\thisrow{std-lambda-0.9},col sep=comma] 
	{csv/lambda/3-rooms-3c/linear.csv};
	\addplot[name path=bot,draw=none,forget plot]
	table [name path=top,x=i,y expr=\thisrow{mean-lambda-0.9}-\thisrow{std-lambda-0.9},col sep=comma] 
	{csv/lambda/3-rooms-3c/linear.csv};
	\addplot [forget plot, draw=none,opacity=0.15,pattern=north east lines,fill=neonFuchsia]
	fill between[of=top and bot];
	\addplot
	table [x=i,y=mean-lambda-1.0,col sep=comma] 
	{csv/lambda/3-rooms-3c/linear.csv};
	\addplot[name path=top,draw=none,forget plot]
	table [name path=top,x=i,y expr=\thisrow{mean-lambda-1.0}+\thisrow{std-lambda-1.0},col sep=comma] 
	{csv/lambda/3-rooms-3c/linear.csv};
	\addplot[name path=bot,draw=none,forget plot]
	table [name path=top,x=i,y expr=\thisrow{mean-lambda-1.0}-\thisrow{std-lambda-1.0},col sep=comma] 
	{csv/lambda/3-rooms-3c/linear.csv};
	\addplot [forget plot, draw=none,opacity=0.15,pattern=north east lines,fill=bronze]
	fill between[of=top and bot];
	\addplot
	table [x=i,y=mean-likelihood,col sep=comma] 
	{csv/lambda/3-rooms-3c/linear.csv};
	\addplot[name path=top,draw=none,forget plot]
	table [name path=top,x=i,y expr=\thisrow{mean-likelihood}+\thisrow{std-likelihood},col sep=comma] 
	{csv/lambda/3-rooms-3c/linear.csv};
	\addplot[name path=bot,draw=none,forget plot]
	table [name path=top,x=i,y expr=\thisrow{mean-likelihood}-\thisrow{std-likelihood},col sep=comma] 
	{csv/lambda/3-rooms-3c/linear.csv};
	\addplot [forget plot, draw=none,opacity=0.15,pattern=north east lines,fill=black]
	fill between[of=top and bot];

	\end{axis}
	\end{tikzpicture}
	\label{fig:3-rooms-lambda-linear-3c}
}

\subfigure[3-rooms $\sin$ dynamic.]
{
	\begin{tikzpicture}
	\begin{axis}[
	width=0.45\textwidth,
	height=4.5cm,
	xmin=50,
	xmax=14950,
	xtick={3000,6000,...,12000},
	ymin=0,
	ymax=0.9,
	ytick={0,0.2,...,0.8},
	%
	%
	each nth point=4,
	xlabel=Iterations,
	ylabel=Average Return,
	mark options={scale=0.2},
	cycle list name = custom,
	scaled x ticks=base 10:-3
	]
	\addplot
	table [x=i,y=mean-lambda-0.1,col sep=comma] 
	{csv/lambda/3-rooms-3c/sin.csv};
	\addplot[name path=top,draw=none,forget plot]
	table [name path=top,x=i,y expr=\thisrow{mean-lambda-0.1}+\thisrow{std-lambda-0.1},col sep=comma] 
	{csv/lambda/3-rooms-3c/sin.csv};
	\addplot[name path=bot,draw=none,forget plot]
	table [name path=top,x=i,y expr=\thisrow{mean-lambda-0.1}-\thisrow{std-lambda-0.1},col sep=comma] 
	{csv/lambda/3-rooms-3c/sin.csv};
	\addplot [forget plot, draw=none,opacity=0.4,pattern=north east lines,fill=green1!60!black]
	fill between[of=top and bot];
	\addplot
	table [x=i,y=mean-lambda-0.2,col sep=comma] 
	{csv/lambda/3-rooms-3c/sin.csv};
	\addplot[name path=top,draw=none,forget plot]
	table [name path=top,x=i,y expr=\thisrow{mean-lambda-0.2}+\thisrow{std-lambda-0.2},col sep=comma] 
	{csv/lambda/3-rooms-3c/sin.csv};
	\addplot[name path=bot,draw=none,forget plot]
	table [name path=top,x=i,y expr=\thisrow{mean-lambda-0.2}-\thisrow{std-lambda-0.2},col sep=comma] 
	{csv/lambda/3-rooms-3c/sin.csv};
	\addplot [forget plot, draw=none,opacity=0.4,pattern=north east lines,fill=orange]
	fill between[of=top and bot];
	\addplot
	table [x=i,y=mean-lambda-0.3,col sep=comma] 
	{csv/lambda/3-rooms-3c/sin.csv};
	\addplot[name path=top,draw=none,forget plot]
	table [name path=top,x=i,y expr=\thisrow{mean-lambda-0.3}+\thisrow{std-lambda-0.3},col sep=comma] 
	{csv/lambda/3-rooms-3c/sin.csv};
	\addplot[name path=bot,draw=none,forget plot]
	table [name path=top,x=i,y expr=\thisrow{mean-lambda-0.3}-\thisrow{std-lambda-0.3},col sep=comma] 
	{csv/lambda/3-rooms-3c/sin.csv};
	\addplot [forget plot, draw=none,opacity=0.15,pattern=north east lines,fill=blue1]
	fill between[of=top and bot];
	\addplot
	table [x=i,y=mean-lambda-0.4,col sep=comma] 
	{csv/lambda/3-rooms-3c/sin.csv};
	\addplot[name path=top,draw=none,forget plot]
	table [name path=top,x=i,y expr=\thisrow{mean-lambda-0.4}+\thisrow{std-lambda-0.4},col sep=comma] 
	{csv/lambda/3-rooms-3c/sin.csv};
	\addplot[name path=bot,draw=none,forget plot]
	table [name path=top,x=i,y expr=\thisrow{mean-lambda-0.4}-\thisrow{std-lambda-0.4},col sep=comma] 
	{csv/lambda/3-rooms-3c/sin.csv};
	\addplot [forget plot, draw=none,opacity=0.15,pattern=north east lines,fill=violet]
	fill between[of=top and bot];
	\addplot
	table [x=i,y=mean-lambda-0.5,col sep=comma] 
	{csv/lambda/3-rooms-3c/sin.csv};
	\addplot[name path=top,draw=none,forget plot]
	table [name path=top,x=i,y expr=\thisrow{mean-lambda-0.5}+\thisrow{std-lambda-0.5},col sep=comma] 
	{csv/lambda/3-rooms-3c/sin.csv};
	\addplot[name path=bot,draw=none,forget plot]
	table [name path=top,x=i,y expr=\thisrow{mean-lambda-0.5}-\thisrow{std-lambda-0.5},col sep=comma] 
	{csv/lambda/3-rooms-3c/sin.csv};
	\addplot [forget plot, draw=none,opacity=0.15,pattern=north east lines,fill=mikadoYellow]
	fill between[of=top and bot];
	\addplot
	table [x=i,y=mean-lambda-0.6,col sep=comma] 
	{csv/lambda/3-rooms-3c/sin.csv};
	\addplot[name path=top,draw=none,forget plot]
	table [name path=top,x=i,y expr=\thisrow{mean-lambda-0.6}+\thisrow{std-lambda-0.6},col sep=comma] 
	{csv/lambda/3-rooms-3c/sin.csv};
	\addplot[name path=bot,draw=none,forget plot]
	table [name path=top,x=i,y expr=\thisrow{mean-lambda-0.6}-\thisrow{std-lambda-0.6},col sep=comma] 
	{csv/lambda/3-rooms-3c/sin.csv};
	\addplot [forget plot, draw=none,opacity=0.15,pattern=north east lines,fill=lapisLazuli]
	fill between[of=top and bot];
	\addplot
	table [x=i,y=mean-lambda-0.7,col sep=comma] 
	{csv/lambda/3-rooms-3c/sin.csv};
	\addplot[name path=top,draw=none,forget plot]
	table [name path=top,x=i,y expr=\thisrow{mean-lambda-0.7}+\thisrow{std-lambda-0.7},col sep=comma] 
	{csv/lambda/3-rooms-3c/sin.csv};
	\addplot[name path=bot,draw=none,forget plot]
	table [name path=top,x=i,y expr=\thisrow{mean-lambda-0.7}-\thisrow{std-lambda-0.7},col sep=comma] 
	{csv/lambda/3-rooms-3c/sin.csv};
	\addplot [forget plot, draw=none,opacity=0.15,pattern=north east lines,fill=harvardCrimson]
	fill between[of=top and bot];
	\addplot
	table [x=i,y=mean-lambda-0.8,col sep=comma] 
	{csv/lambda/3-rooms-3c/sin.csv};
	\addplot[name path=top,draw=none,forget plot]
	table [name path=top,x=i,y expr=\thisrow{mean-lambda-0.8}+\thisrow{std-lambda-0.8},col sep=comma] 
	{csv/lambda/3-rooms-3c/sin.csv};
	\addplot[name path=bot,draw=none,forget plot]
	table [name path=top,x=i,y expr=\thisrow{mean-lambda-0.8}-\thisrow{std-lambda-0.8},col sep=comma] 
	{csv/lambda/3-rooms-3c/sin.csv};
	\addplot [forget plot, draw=none,opacity=0.15,pattern=north east lines,fill=harlequin]
	fill between[of=top and bot];
	\addplot
	table [x=i,y=mean-lambda-0.9,col sep=comma] 
	{csv/lambda/3-rooms-3c/sin.csv};
	\addplot[name path=top,draw=none,forget plot]
	table [name path=top,x=i,y expr=\thisrow{mean-lambda-0.9}+\thisrow{std-lambda-0.9},col sep=comma] 
	{csv/lambda/3-rooms-3c/sin.csv};
	\addplot[name path=bot,draw=none,forget plot]
	table [name path=top,x=i,y expr=\thisrow{mean-lambda-0.9}-\thisrow{std-lambda-0.9},col sep=comma] 
	{csv/lambda/3-rooms-3c/sin.csv};
	\addplot [forget plot, draw=none,opacity=0.15,pattern=north east lines,fill=neonFuchsia]
	fill between[of=top and bot];
	\addplot
	table [x=i,y=mean-lambda-1.0,col sep=comma] 
	{csv/lambda/3-rooms-3c/sin.csv};
	\addplot[name path=top,draw=none,forget plot]
	table [name path=top,x=i,y expr=\thisrow{mean-lambda-1.0}+\thisrow{std-lambda-1.0},col sep=comma] 
	{csv/lambda/3-rooms-3c/sin.csv};
	\addplot[name path=bot,draw=none,forget plot]
	table [name path=top,x=i,y expr=\thisrow{mean-lambda-1.0}-\thisrow{std-lambda-1.0},col sep=comma] 
	{csv/lambda/3-rooms-3c/sin.csv};
	\addplot [forget plot, draw=none,opacity=0.15,pattern=north east lines,fill=bronze]
	fill between[of=top and bot];
	\addplot
	table [x=i,y=mean-likelihood,col sep=comma] 
	{csv/lambda/3-rooms-3c/sin.csv};
	\addplot[name path=top,draw=none,forget plot]
	table [name path=top,x=i,y expr=\thisrow{mean-likelihood}+\thisrow{std-likelihood},col sep=comma] 
	{csv/lambda/3-rooms-3c/sin.csv};
	\addplot[name path=bot,draw=none,forget plot]
	table [name path=top,x=i,y expr=\thisrow{mean-likelihood}-\thisrow{std-likelihood},col sep=comma] 
	{csv/lambda/3-rooms-3c/sin.csv};
	\addplot [forget plot, draw=none,opacity=0.15,pattern=north east lines,fill=black]
	fill between[of=top and bot];

	\end{axis}
	\end{tikzpicture}
	\label{fig:3-rooms-lambda-sin-3c}
}

	\caption{Average return achieved by $3$-T2VT \wrt different choices of $\lambda$ with $95\%$ confidence intervals computed using $50$ independent runs.}
	\label{fig:3-rooms-lambda-3c}
\end{figure*}